\newif\ifarXiv         
\newif\ifjournal        
\journalfalse   \arXivtrue          
\documentclass[11pt]{article}
\usepackage[small,compact]{titlesec}
\usepackage{amsmath,amssymb,amsfonts,mathabx,setspace}
\usepackage{amsbsy,amsmath}
\usepackage{mathrsfs}
\usepackage{cases}
\usepackage{graphicx,caption,epsfig}
\usepackage{epsfig,wrapfig}
\usepackage[dvipsnames,svgnames]{xcolor}
\usepackage{epstopdf}
\usepackage{authblk}
\usepackage{graphicx,epsfig,wrapfig,caption,epsfig,subcaption,sidecap}
\usepackage{url,color,verbatim,algorithmicx}
\usepackage[noend]{algpseudocode}
\usepackage[ruled,boxed]{algorithm} 

\usepackage{bbm}
\usepackage{enumitem} 
\usepackage{booktabs}  
\usepackage[sort,compress]{cite}
\usepackage[scaled]{helvet}
\usepackage[T1]{fontenc}

\usepackage{tikz}
\usetikzlibrary{positioning}
\usetikzlibrary{arrows,shapes}
\usetikzlibrary{decorations.markings}
\usetikzlibrary{arrows.meta}
\usetikzlibrary{shapes.multipart}
\usetikzlibrary{shapes.geometric}

\usepackage[bookmarks=true, bookmarksnumbered=true, colorlinks=true,   pdfstartview=FitV,
linkcolor=blue, citecolor=blue, urlcolor=blue]{hyperref}
\usepackage[english]{babel}  
\usepackage{multirow}

\let\OLDthebibliography\thebibliography
\renewcommand\thebibliography[1]{
  \OLDthebibliography{#1}
  \setlength{\parskip}{0pt}
  \setlength{\itemsep}{3pt}
}

\usepackage{fullpage}
\def\calE{\mathcal{E}}    
\def\calH{\mathcal{H}} 
 
\def\calF{\mathcal{F}}
\def\calN{\mathcal{N}}

\def\ba{\mathbf{a}}  

\def\sA{\mathscr{A}}
 
\def\bB{\mathbf{B}}
\def\bc{\mathbf{c}}
\def\bN{\mathbf{N}}
\def\br{\mathbf{r}}
\def\bX{\mathbf{X}}    
\def\bZ{\mathbf{Z}}

\def\bu{\textbf{u}}
\def\bv{\textbf{v}}
\def\bW{\mathbf{W}}   

\def\Tr{\emph{Tr}}

\def\R{\mathbb{R}}       
       
\def\P{\mathbb{P}}    
\def\E{\mathbb{E}}

\newcommand{\prob}[1]{\mathbb{P}\left({#1}\right)}

\newcommand{\rbracket}[1]{\left(#1\right)}

\newcommand{\norm}[1]{\left\|#1\right\|}
\newcommand{\innerp}[1]{\langle{#1}\rangle}

\newcommand{\argmin}[1]{\underset{#1}{\operatorname{arg}\operatorname{min}}\;}

\newcommand{\dynsys}[2]{\mathcal{S}_{{#1},{#2}}}
\newcommand{\IK}{\Phi}
\def\trCov{\mathrm{tr\,Cov}}

\newcommand{\p}{p}
\renewcommand{\top}{{\rm T}}

\def\mH{\mathcal{H}}

\def\mE{\mathcal{E}}

\def\R{\mathbb{R}}          
                       
\def\P{\mathbb{P}}

\def\nmcons{1}

\newcommand{\fcommentout}[1]{}
\newtheorem{theorem}{Theorem}

\newtheorem{assumption}[theorem]{Assumption}

\newtheorem{definition}[theorem]{Definition}
\newtheorem{example}[theorem]{Example}
\newtheorem{lemma}[theorem]{Lemma}
\newtheorem{proposition}[theorem]{Proposition}
\newtheorem{remark}[theorem]{Remark}

\newenvironment{proof}[1][Proof]{\noindent\textbf{#1.} }{\ \rule{0.5em}{0.5em}}

\numberwithin{equation}{section}
\numberwithin{theorem}{section}

\usepackage{soul}

\allowdisplaybreaks[4]

\def\Agraph{\mathbf{\Gamma}} 
\def\bgraph{\mathbf{v}}

\title{Interacting Particle Systems on Networks: \\ joint inference of the network and the interaction kernel}
\author[1]{Quanjun Lang\thanks{quanjun.lang@duke.edu; co-first authors}}
\author[2]{Xiong Wang\thanks{xiongwang@ualberta.ca; co-first authors}}
\author[3]{Fei Lu\thanks{feilu@math.jhu.edu}}
\author[3,4]{Mauro Maggioni\thanks{mauromaggionijhu@icloud.com}}
\affil[1]{Department of Mathematics, Duke University, Durham, USA. }
\affil[2]{School of Mathematics, Sun Yat-sen University, Guangzhou, China. }
\affil[3]{Department of Mathematics, Johns Hopkins University, Baltimore, USA. }
\affil[4]{Department of Applied Mathematics and Statistics, Johns Hopkins University, Baltimore, USA. }
\date{}
\begin{document}
\maketitle

\begin{abstract}
Modeling multi-agent systems on networks is a fundamental challenge in a wide variety of disciplines. 
Given data consisting of multiple trajectories, we jointly infer the (weighted) network and the interaction kernel, which determine respectively which agents are interacting and the rules of such interactions. 
Our estimator is based on a non-convex optimization problem, and we investigate two approaches to solve it: one based on an alternating least squares (ALS) algorithm, and another based on a new algorithm named operator regression with alternating least squares (ORALS). 
Both algorithms are scalable to large ensembles of data trajectories.  We establish coercivity conditions guaranteeing identifiability and well-posedness. The ALS algorithm appears statistically efficient and robust even in the small data regime, but lacks performance and convergence guarantees. The ORALS estimator is consistent and asymptotically normal under a coercivity condition. We conduct several numerical experiments ranging from Kuramoto particle systems on networks to opinion dynamics in leader-follower models.
\end{abstract}



\tableofcontents



\section{Introduction}\label{sec:intro}
Modeling multi-agent systems on networks is fundamental in diverse fields, including opinions on social networks, flows on electric power grids or airport networks, migrating cells, or the abstract space meshes in numerical computations; see e.g., \cite{boccaletti2006complex,gaskin2024inferring,tanner2003stable,olfati2007consensus,pesce2021learning,wu2020comprehensive}. In such systems, the agents (e.g., opinions, vehicles, cells, robotic units) evolve over time by interacting with one another according to an underlying network. These systems can exhibit nonlinear dynamics, with complex emergent behaviors (see \cite{DiscoveryEmergentBehaviors} and references therein). The network specifies who influences whom, and determines the strength of the interactions. 
In this work, we consider the problem of performing joint inference of both the network structure and the interaction rules from observed trajectory data, which we consider a fundamental problem in data-driven modeling of multi-agent dynamics. 

Much of the existing literature addresses only parts of this challenge, focusing on either learning the interaction function under a known network (see, e.g., \cite{LZTM19pnas,LMT21_JMLR,LMT21,Miller2023_Learning2ndorder,LearningInteractionVariables,Amini22learning,GaoInferenceComplexNetwork,hu2024learninginterpretablenetworkdynamics}) or identifying direct interactions without estimating the function (see, e.g., \cite{casadiego2017model} and the extensive list of references therein). One exception to the above is \cite{LearningTransitionOperators_1}, which explores the joint estimation with statistical and computational guarantees in the special case of {\em{linear}} dynamics on a network. Yet, in many practical settings, both the network and the functional form of the interactions are initially unknown and must be identified jointly from data. This study takes a first step in this direction by studying this joint inference problem for interacting particle systems on networks, where agents are at network vertices and the interaction kernel depends on pairwise displacements of their states.

\subsection{Problem setup}
We consider a heterogeneous dynamical system with $N$ interacting agents on a weighted directed graph $G=(V,E,\ba)$, where $\ba= (\ba_{ij}) \in [0,1]^{N\times N}$ is a matrix of edge weights. For each agent $i \in [N] := \{1,\dots,N\}$, its state at time $t$ is a vector $X^i_t \in \mathbb{R}^d$ and its evolution depends on the states of its neighbors in the network and an \emph{interaction kernel} $\IK:\mathbb{R}^d\to\mathbb{R}^d$. The system has a governing equation consisting of a system of ordinary or stochastic differential equations: 
\begin{equation}\label{eq:ips_K}
d {X}^i_t =\sum_{j\neq i}\ba_{ij}\IK(X^j_t-X^i_t) dt + \sigma d {W}^i_t, \quad i\in [N]\,, 
\end{equation}
where we write $\sum_{j\neq i}$ to denote $\sum_{j\in[N]\setminus\{i\}}$. Here, the forcing term $(W^i_t)_{i\in [N]}$ is an $\R^{N\times d}$-valued standard Brownian motion, and the diffusion coefficient  $\sigma$ is a constant; the system is deterministic when $\sigma=0$ and stochastic when $\sigma>0$. The weight matrix $\ba$ encodes which pairs of agents interact (i.e., which edges exist), along with the strength of these interactions, while $\IK$ encodes the functional rules for interactions.

In this setting, the task of joint inference is to estimate the unknown weight matrix $\ba$ and the interaction kernel $\IK$ from multi-trajectory data. 
We assume that $\ba$ is in the admissible set
\begin{equation}\label{eq:aMat_set}
	\mathcal{M} := \bigg\{\ba=(\ba_{ij})\in  [0,1]^{N\times N}\!:\,\forall i\in[N]\,\,\, \ba_{ii}=0\,\,,\,\, |\ba_{i\cdot}|^2:=\|\ba_{i\cdot}\|_{\ell^2}^2=\sum_{j=1}^N \ba_{ij}^2 = \nmcons \,  \bigg\}\,. 
\end{equation}
Here, the row-wise normalization of the weight matrix removes a trivial issue in the identifiability of $(\ba,\IK)$ due to rescaling: $(\ba,\IK)$ can be replaced by $(\lambda\ba,\lambda^{-1}\IK)$ in \eqref{eq:ips_K} for any $\lambda>0$. The choice of the $\ell^2$ normalization is somewhat arbitrary in our analysis and algorithms; other norms, such as the $\ell^1$-norm or the Frobenius norm, may be used, depending on modeling assumptions, and its choice has important for example when one studies the mean-field limit ($N\to \infty$,  see, e.g., \cite{LackerRamananWu2023} and references therein).

\begin{table}[htbp]\caption{Notations for the indices, vectors, and arrays in the system. }
{\centering
\begin{tabular}{l |  l }
\toprule
\hline 
$[N]$ : index set $\{1,\dots,N\}$ 			& $\bX_t = (X^1_t,\ldots, X^N_t) \in \R^{N\times d}$:  state vector at time $t$ \\
$i,j\in[N]$:  indices for agents    				& $\ba \in \R^{N\times N}$: graph weight matrix \\
$k\in[\p]$:  index for basis of kernel 			&   $c\in \R^{\p\times 1}$: coefficient vector of $K$ on a basis $\{\psi_k\}$\\
$m\in[M]$:  index for samples         			&  $\bB(\bX_t)= \big[\psi_k(X_t^j-X_t^i) \big]_{j,i,k}\in \R^{N\times N\times d\times\p} $: basis array  \\
$l\in[L]$: index of observation times          		& $\|\cdot\|_F$: the Frobenius norm of a matrix\\
$|\cdot|$: the Euclidean norm of a vector      	& {${\rm{Vec}}: \R^{N\times\p}\to \R^{N\p\times 1}$ is the vectorization operation.}\\
\hline 
\bottomrule 
\end{tabular} 
}\\
{* \footnotesize We use letters for vectors, bold letters for arrays/matrices
, and calligraphic letters for operators.}
\label{tab:notation}
\end{table}

We restrict our attention to parametric families of interaction kernels: we estimate the coefficient $c=(c_1,\ldots,c_p)\in \R^{\p}$ of the kernel 
$$\IK(x) = \sum_{k=1}^\p c_k \psi_k(x)$$
with respect to a given set of basis functions $\{\psi_k\}_{k=1}^{\p}$. 
We also consider the case when the true interaction kernel is not in the hypothesis space $\mathcal{H}:=\mathrm{span}\{\psi_k\}_{k=1}^{\p}$, demonstrating in our experiments that our estimator is robust to model mis-specification.

We let $\bX_t := (X^1_t,\ldots, X^N_t) \in \R^{N\times d}$ be the state vector, $\dot\bW: =[dW^i_t]_i\in\R^{N\times d}$ be the white noise in the forcing term, and $\bB(\bX_t)_i := \big(\psi_k(X_t^j-X_t^i) \big)_{j,k}\in \R^{ N\times 1 \times d\times\p}$ for each $i\in[N]$. We can then rewrite  \eqref{eq:ips_K} in tensor form:  
\begin{equation}\label{eq:ipog} 
 \begin{aligned}
\dynsys {\mathbf{a}}{c}\quad:\quad &\dot \bX_t  = \ba\odot\bB(\bX_t) c+\sigma \dot{\bW} \, \text{ with }  \ba\odot\bB(\bX_t) c:= \big(  \ba_{i\cdot}\bB(\bX_t)_i c\big)_{i\in[N]} \,,\,   \text{ where}  \\
 & \ba_{i\cdot}\bB(\bX_t)_i c  = \sum_{j\neq i} \ba_{ij}\sum_{k=1}^{\p} \psi_k(X^j_t-X^i_t)c_k \in \R^{d}\,,\quad i\in[N],  
 \end{aligned}
\end{equation}
with $\ba_{i\cdot}$ is the $i$-th row of the matrix $\ba$. We summarize the notation in Table \ref{tab:notation}.

\paragraph{Problem statement.}
Our goal is to \emph{jointly} estimate the weight matrix $\ba$ and the coefficient vector $c$, and therefore the interaction kernel $\IK$, given
\begin{equation}\label{eq:data}\vspace{-1mm}
\smash{
\textbf{Data: } \quad \{\bX_{t_0:t_L}^m\}_{m=1}^M, \quad \text{ where } t_0:t_L \text{ denotes } (t_0,t_1,t_2,\ldots,t_L)\,,\text{ with } t_l= l\Delta t\, ,
}\end{equation}
i.e., observations of the state vector at discrete times along multiple-trajectories indexed by $m$, started from initial conditions $\bX^m_{t_0}$ sampled from $\mu^{\otimes N}$, where $\mu$ is a distribution on $\R^d$. We let $T=t_L$ and $t_0=0$. 

\subsection{Main results: scalable algorithms, identifiability, and convergence}
Our estimator of the parameter $(\ba,c)$ is a minimizer of a loss function $\calE_{L, M}$: 
\begin{equation}\label{eq:est_joint}
(\widehat{\ba},  \widehat{c} )= \argmin{(\ba,c)\in \mathcal{M}\times \R^{\p}}  \calE_{L, M}(\ba,c) \,\,\text{ with }\, \calE_{L, M}(\ba,c) := \frac{1}{MT}\sum_{l=0, m = 1}^{L-1,M}\norm{\Delta\bX_{t_l}^m -\ba\odot\bB (\bX_{t_l}^m) c\Delta t}_F^2, 
\end{equation}
where $\mathcal{M}$ is the admissible set defined in \eqref{eq:aMat_set}  and $\|\cdot\|_F$ denotes the Frobenuous norm on $\R^{N\times d}$. Here $\Delta\bX_{t_l}= \bX_{t_{l+1}} - \bX_{t_{l}}$; if the system were deterministic ($\sigma=0$) and we had observations of $\dot{\bX}_{t_l}$, we use these instead. 
This loss function comes from the differential system \eqref{eq:ipog}: its scaled version $(\Delta t)^{-1} \calE_{L, M}(\ba,c) $ is the mean square error between the two sides of the system when $\sigma=0$; it is 
the scaled log-likelihood ratio (up to a constant independent of the data trajectories) for the stochastic system when $\sigma>0$. We refer to \cite[Section 2.1]{LMT21} for a detailed discussion.

Since the loss function is quadratic in either $\ba$ or $c$ separately, we consider the classical alternating minimization over $\ba$ and $c$, where the minimization steps lead to least squares problems; see Section \ref{sec:als}. This corresponds to Alternating Least Squares (ALS) in \cite{Leeuw_ALS}.

However, the nonconvexity of the loss function $\calE_{L, M}$ in $(\ba,c)$ poses a major challenge, which is further complicated by the dependence of $c$ on the basis functions in the hypothesis space. With the nonlinear constraints on the weight matrix $\mathbf{a}$, the loss function may have multiple local minima even in the limit $M\rightarrow\infty$. Thus, it poses a risk to methods such as gradient descent or the ALS. In particular, these methods do not have a performance guarantee and do not provide the tools for studying the identifiability of the nonlinear inverse problem.    

To build mathematical foundations for the inference problem, we introduce a variant of the ALS algorithm, called ORALS, that first uses an Operator Regression that estimates product matrices of $\ba$ and $c$, and then uses ALS on much simpler matrix factorization problems to obtain the factors $\ba$ and $c$ from the estimated products; see Section \ref{sec:orals}. Through ORALS, we study the identifiability of the parameters, the well-posedness of the inverse problem, and the convergence of the estimators as the sample size increases in Section \ref{sec:theory-ID-conv}. We introduce coercivity conditions, key properties in the learning theory of interaction kernels (see, for example, \cite{LMT21_JMLR,LMT21,LLMTZ21}) that guarantee the identifiability of the parameters and the well-posedness of the inverse problem. Under the coercivity conditions, we show that the regression matrices in ALS and ORALS are well-conditioned with a high probability, and the ORALS estimator is asymptotically normal.

In Section \ref{sec:numExper}, we examine the ALS and ORALS algorithms numerically in terms of the dependence of their accuracy and robustness with respect to each of the following three key parameters: sample size, level of observation noise, the strength of the stochastic force, and possible misspecification of the model (i.e. the hypothesis space does not contain the true interaction kernel). 

Overall, ALS appears to be particularly efficient and robust, both statistically and computationally, as soon the number of observations is comparable to the number of unknown parameters $(\ba, c)$; the estimator it produces converges as the sample size increases, although there are no theoretical guarantees for this convergence.  ORALS performs as well as ALS in the large sample regime, with the estimator converging at the theoretically guaranteed rate $M^{-1/2}$.
In the non-asymptotic regime, our numerical experiments, detailed in Section \ref{sec:alg_comparison}, suggest that the critical sample size (i.e., the smallest sample size sufficient for a reasonable estimation with high probability) for ALS is of order $M\gtrsim \frac{N^2+p}{NLd} $, where the numerator is the number of parameters $(\ba,c)$ to be estimated and the denominator is the number of scalar observations per sample trajectory. This critical sample size is therefore optimal in terms of the number $M$ of independent observations. This is much smaller than the critical sample size for ORALS, which appears to be of order $M\gtrsim \frac{N^2p}{NLd}$.  Thus, while ORALS leads to an asymptotic normal estimator, ALS outperforms in sample efficiency and robustness.

We also demonstrate our estimator in three applications: learning the Kuramoto model on a network, identifying a leader-follower network, and recovering multitype interaction kernels in Section \ref{sec:applications}. In particular, we consider a \emph{multi-type agent} system with $Q$ types of agents and corresponding interaction kernels $(\IK_q)_{q=1}^{Q}$, where the type of agent $i$ is denoted by $\kappa(i)$, and governing equations
\begin{equation}\label{eq:ips_K_type}
\dynsys {\mathbf{a}}{(\IK_q)_{q=1}^{Q},\kappa}\,:\qquad d {X}^i_t =
\sum_{j\neq i}\ba_{ij}\IK_{\kappa(i)}(X^j_t-X^i_t) dt + \sigma d {W}^i_t, \quad i \in [N]\,.
\end{equation}
We tackle here the challenging problem where the type $\kappa$ of each agent is not known, and it needs to be estimated together with the weight matrix $\ba$ and the interaction kernels $\phi_1,\dots,\phi_Q$. We introduce a three-fold ALS algorithm to solve this problem in Section \ref{sec:multi_type}. 

\paragraph{Our main contributions} Our main contributions are threefold.  
\begin{enumerate}
\item We investigate two algorithms for joint inference of network and interaction functions: ALS (Alternating Least Squares) and ORALS (Operator Regression followed by Alternating Least Squares). Numerical experiments indicate that both algorithms yield convergent estimators as the sample size increases, with ALS demonstrating superior sample efficiency and robustness.

\item We establish mathematical foundations for the inference problem by deriving coercivity conditions to analyze parameter identifiability and ensure the well-posedness of the nonlinear inverse problem. In particular, we prove that the ORALS estimator is asymptotically normal. 

\item We systematically assess the robustness and accuracy of our estimators and demonstrate their applications using synthetic data. Our tests quantify the effects of sample size, observation noise level, and the magnitude of stochastic forcing on accuracy, and examine robustness under model misspecification. 
 \end{enumerate} 

\subsection{Related work}

\paragraph{Learning the interaction kernel} 
When the graph is complete and undirected, i.e., $\ba_{ij}\equiv \frac 1 N$ for all $(i,j)$, we have homogeneous interactions. In this case, the learning of radial interaction kernels $\IK$ in the form $\IK(x,y) = \phi(|x-y|) \frac{x-y}{|x-y|}$ has been systematically studied in \cite{LZTM19pnas,LMT21_JMLR,LMT21,LLMTZ21} and generalized to second-order systems and non-radial interaction kernels \cite{Miller2023_Learning2ndorder}, and even to interaction kernels whose variables are learned \cite{LearningInteractionVariables}.  The learning of the interaction kernel from a single trajectory, assuming knowledge of the underlying network, has been studied in \cite{Amini22learning}.

\paragraph{Learning the network}
Another related problem is estimating the graph underlying linear Markovian dynamics on the graph when only sparse observations in space and time are given \cite{LearningTransitionOperators_1}. However, none of these works address the joint estimation problem.

\paragraph{Compress sensing and matrix sensing} The joint estimation of $\mathbf{a}$ and $c$ is closely related to compressed sensing and matrix sensing as elaborated in seminal works including  \cite{Candes2008,Candes2009,Candes2010,RFP2010,GJZ2017,ZSL2017}. The array $\{\bB (\bX_{t_l}^m)\}$ in \eqref{eq:est_joint} plays the role of sensing linear operator for the unknowns $\ba$ and $c$.  Diverging from the conventional framework of matrix sensing, where the entries of the sensing matrix are typically independent, the entries of $\bB (\bX_{t_l}^m)$ are correlated, depending on the dynamics and the basis functions. Furthermore, we have the additional constraint that the entries of the weight matrix $\ba$ are nonnegative. These differences can lead to multiple local minima for the loss function $\calE_{L,M}$, even in the limit $M\rightarrow\infty$. 

More fundamentally, our joint inference operates in the product space of a vector space and a function space (i.e., the hypothesis space for the interaction kernel), whereas matrix sensing typically deals with spaces of low-rank matrices. Consequently, there are two key differences between our coercivity conditions and the Restricted Isometry Property (RIP) condition in matrix sensing (see, e.g., \cite{Candes2008,RFP2010,ZSL2017}): the RIP condition is for each set of basis (sensing) functions, while our coercivity conditions are on both the distribution of the data and the hypothesis space $\mathcal{H}$, and independent of the choice of basis functions for $\mathcal{H}$.

\section{Two algorithms: ALS and ORALS}
\label{sec:alg}
We present two algorithms constructing the estimator in \eqref{eq:est_joint}: an {\em{Alternating Least Squares}} (ALS) approach and a two-step algorithm based on \emph{Operator Regression followed by an Alternating Least Squares} (ORALS), and discuss their computational complexity. We will present their theoretical guarantees in Section \ref{sec:theory-ID-conv} and further examine their numerical performance in Section \ref{sec:numExper}. 

ALS is computationally efficient, with well-conditioned matrices as soon the number of observations is comparable to the number of unknown parameters $(\ba, c)$, but with weak theoretical guarantees. ORALS is amenable to theoretical analysis, achieving consistency and asymptotic normality, albeit at a somewhat higher (in $N$ and $\p$) computational cost. 

\subsection{Alternating Least Squares (ALS)}\label{sec:als}
The ALS algorithm exploits the convexity in each variable by alternating between the estimation of the weight matrix $\ba$ and of the coefficient $c$ while keeping the other fixed:

\smallskip
\noindent{\em Inference of the weight matrix. } 
Given an interaction kernel, represented by the corresponding set of coefficients $c$, we estimate the weight matrix $\ba$ by directly solving the minimizer of the quadratic loss function with $c$ fixed, followed by row-normalizing the estimator.  For every $i\in[N]$, we obtain the minimizer (with $\ba_{ii}=0$) of the loss function $\calE_{L, M}(\ba,c)$ in \eqref{eq:est_joint} with $c$ fixed by solving $\nabla_{\ba_{i\cdot}}\calE_{L, M}(\ba,c)= 0$, which is a linear equation in $\ba_{i\cdot}$:
\begin{equation}\label{eq:ALS_graphest}
\begin{aligned}
\widehat\ba_{i\cdot} \mathcal{A}^{\textrm{ALS}}_{c,M,i} := \widehat \ba_{i\cdot} ([\bB(\bX_{t_l}^m)_{i}]_{l,m} c) & = [(\Delta \bX^m_{t_l})_i]_{l,m}/\Delta t\,,\qquad i\in[N], 
\end{aligned}
\end{equation}
where 
$[\bB(\bX_{t_l}^m)_{i}]_{l,m}\in\mathbb{R}^{N \times (dLM) \times\p}$, $\mathcal{A}^{\textrm{ALS}}_{c,M,i} := [\bB(\bX_{t_l}^m)_{i}]_{l,m} c\in\mathbb{R}^{N \times dLM}$ and $[(\Delta \bX_{t_l}^m)_i]_{l,m}\in\mathbb{R}^{1\times dLM}$, with $i\in[N]$ fixed and $l,m$ varying in $[L]$ and $[M]$ respectively, are obtained by appropriately stacking, over $l$ and $m$, the corresponding matrices in block-row fashion, and tensor-vector multiplications (by $\widehat\ba_{i\cdot}$ on the left and $c$ on the right) act on the appropriate tensor slices. 
We solve this linear system by least squares with nonnegative constraints \cite[Chapter 23]{lawson1995solving}, since $\ba\in\mathcal{M}$ implies that the entries of $\ba$ are nonnegative, followed by a normalization in $\ell^2$-norm to obtain an estimator $\widehat{\ba}_{i\cdot}$ in the admissible set $\mathcal{M}$ defined in \eqref{eq:aMat_set}.
We remark that should the network be partially known, such knowledge could be incorporated as additional constraints at this stage.

\smallskip
\noindent{\em Estimating the parametric interaction kernel.}  
In this step, we estimate the parameter $c$ by minimizing the loss function   $\calE_{L, M}(\ba,c)$ in \eqref{eq:est_joint} with a fixed weight matrix $\ba$ estimated above, by solving the least squares problem
\begin{equation}
\mathcal{A}^{\textrm{ALS}}_{\ba,M}\, \widehat{c} := [\ba\odot\bB(\bX_{t_l}^m)]_{l,m}\widehat{c} = [\Delta \bX_{t_l}^m]_{l,m}/\Delta t \,,
\label{eq:ALS_kernelest}
\end{equation}
where $\mathcal{A}^{\textrm{ALS}}_{\ba,M}:=[\ba\odot\bB(\bX_{t_l}^m)]_{l,m}\in\mathbb{R}^{dLMN\times\p}$ is again obtained by stacking, over $l$ and $m$, in a block-row fashion.

We alternate these two steps until the updates to the estimators are smaller than a tolerance threshold $\epsilon$ or until maximal iteration number $n_{maxiter}$ is reached,  as in Algorithm \ref{alg:ALS}.

\begin{algorithm}[t]
\caption{ALS: alternating least squares}\label{alg:ALS}
\small{
\begin{algorithmic}[]
\Procedure{ALS\_IPSonGraph}{$\{\bX^m_{t_0:t_L}\}_{m=1}^M , \{\psi_k\}_{k=1}^{\p} , \epsilon ,{n}_{maxiter}$}
  \State Construct the arrays $\{\bB(\bX_{t_l}^m)\}_{l,m}$ and $\{\Delta \bX_{t_l}^m\}$ in \eqref{eq:est_joint} for each trajectory.  
  \State Randomly pick an initial condition $\widehat{c}_0\neq0$.
  \For{$\tau = 1,\dots,{n}_{maxiter}$}  
  \State Estimate the weight matrix $\widehat{\mathbf{a}}_{\tau}$ by solving \eqref{eq:ALS_kernelest} with $c=\widehat{c}_{\tau-1}$, by nonnegative least squares, 
  
  \indent \quad\,\,\, followed by a row normalization.
  \State Estimate the parameter $\widehat{c}_\tau$ by solving \eqref{eq:ALS_graphest} with $\mathbf{a}=\widehat{\mathbf{a}}_\tau$, by least squares.
  \State Exit loop if $||\widehat{c}_\tau-\widehat{c}_{\tau-1}||\le\epsilon||\widehat{c}_{\tau-1}\|$ and $||\widehat{\mathbf{a}}_{\tau}-\widehat{\mathbf{a}}_{\tau-1}||\le\epsilon ||\widehat{\mathbf{a}}_{\tau-1}||$.
\EndFor
\Return  $\widehat{c}_\tau,\widehat{\mathbf{a}}_\tau$.
\EndProcedure
\end{algorithmic}
}
\end{algorithm}

\subsection{Operator Regression and Alternating Least Squares (ORALS)}\label{sec:orals}
ORALS divides the estimation into two stages: a statistical operator regression stage and a deterministic alternating least squares stage. The first stage estimates the entries of the matrices $\{\ba_{i,\cdot}^\top c^\top \in \R^{(N-1)\times\p} \}_{i=1}^N$ (excluding the zero entries $\ba_{ii}$) by least squares regression with regularization. It is called operator regression because we view the data as the output of a sensing operator over these matrices.
After this step, a deterministic alternating least squares stage jointly factorizes these estimated matrices to obtain the weight matrix $\ba$ and the coefficient $c$.   

\smallskip
\noindent{\em Operator Regression stage.} Consider the arrays $\{\bZ_i=\ba_{i,\cdot}^\top c^\top \in \R^{(N-1)\times\p} \}_{i=1}^N$ treated as vectors in $\R^{(N-1)\p\times 1}$, that is, $z_i={\rm Vec}(\bZ_i) = (\ba_{i,1}c_1,\ldots, \ba_{i,1}c_{\p}, \ba_{i,2}c_1, \ldots,\ba_{i,2}c_{\p}, \ldots )^\top \in \R^{(N-1)\p\times 1}$ for each $i$. They are solutions of the linear equations with sensing operators $\mathcal{A}_{i,M}  = [\mathcal{A}_{i}]_{l,m}\in \R^{dML\times (N-1){\p}}$:    
\begin{equation}\label{eq:Ai_operator}
\mathcal{A}_{i,M} z_i =  [\mathcal{A}_{i}]_{l,m} z_i := [(\ba \bB (\bX_{t_l}^m) c\Delta t)_i]_{l,m} = [(\Delta \bX^m_{t_l})_i]_{l,m}\,,
\quad i\in[N], 
\end{equation}
where, as usual, $[\cdot]_{l,m}$ denotes stacking block rows over the indices $l,m$.
 With the above notation, we can write the loss function in \eqref{eq:est_joint} as 
\begin{equation}\label{eq:est_matZ} 
	(\widehat{z}_{1,M},\ldots, \widehat{z}_{N,M}) =  \argmin{z_1,\ldots,z_N }  \calE_{L, M}(z_1,\ldots,z_N) := \frac{1}{ML}\sum_{l, m,i = 1}^{L, M,N}\big| [(\Delta\bX^m)_i]_{l,m} -[\mathcal{A}_{i}]_{l,m} z_i\big|^2 
\end{equation}
and obtain $\{\widehat{z}_{i,M}\}$ by solving this least squares problem for each $i\in[N]$. 

\smallskip
\noindent{\em Deterministic ALS stage.} The rows of $\ba$ and the vector $c$ are estimated via a joint factorization of the matrices of the estimated vectors $\{\widehat z_{i,M}\}$, denoted by $\widehat \bZ_{i,M}$, with a shared vector $c$:   
\begin{equation}\label{eq:est_Ac}
(\widehat{\ba}^{M},  \widehat{c}^{M} )= \argmin{\ba\in \mathcal{M},\, c\,\in \R^{\p}}  \calE(\ba,c):= \sum_{i=1}^N\norm{\widehat\bZ_{i,M} - \ba_{i,\cdot}^\top c^\top}_F^2, 
\end{equation}
where $\mathcal{M}$ is the admissible set in \eqref{eq:aMat_set}. A deterministic alternating least squares algorithm solves this problem: we first estimate each row of  $\ba$ by nonnegative least squares and then estimate $c$ using all the estimated $\ba$ with row-normalization. We iterate them for two steps, starting from $\widehat c_0$ obtained from rank-1 singular value decomposition, as in Algorithm \ref{alg:ORALS}. Numerical tests show that two iteration steps are often sufficient to complete the factorization, and the result is robust for more iteration steps.

Theorem \ref{thm:AN_orals} shows that the estimator obtained by ORALS is consistent and asymptotically normal under a coercivity condition. 
\begin{algorithm}[htb]
\caption{ORALS: Operator Regression and Alternating Least Squares.}\label{alg:ORALS}
\small{
\begin{algorithmic}[]
\Procedure{ORALS\_IPSonGraph}{$\{\bX^m_{t_0:t_L}\}_{m=1}^M , \{\psi_k\}_{k=1}^{\p}$}
  \State Construct the sensing operators $\mathcal{A}_{i,M}$ (from the arrays $\{\bB(\bX_{t_l}^m)\}_{l,m}$) and $\{\Delta \bX_{t_l}^m\}$ in \eqref{eq:Ai_operator} for each 
    \indent \quad\,\,\,trajectory.  
  \State Solve the vector $\widehat z_{i,M}$'s in \eqref{eq:est_matZ} by least squares with regularization; and transform them into 

    \indent \quad\,\,\,matrices  $\widehat Z_{i,M}$. 
  \State Factorize each matrix $\widehat Z_{i,M}$. Set the initial condition $\widehat{c}_0$ to be the first right singular vector. 
  \For{$\tau = 1,2$}  
  \State Estimate the weight matrix $\widehat{\mathbf{a}}_{\tau}$ by solving \eqref{eq:est_Ac} with $c=\widehat{c}_{\tau-1}$ by nonnegative least squares, 
  
  \indent \quad\,\,\,\, followed by a row normalization.
  \State Estimate the parameter $\widehat{c}_\tau$ by solving \eqref{eq:est_Ac} with $\mathbf{a}=\widehat{\mathbf{a}}_\tau$ by  least squares.
  \EndFor
\Return $\widehat{c}_\tau,\widehat{\mathbf{a}}_\tau$.
 \EndProcedure
  \end{algorithmic}
  }
\end{algorithm}

\subsection{Algorithmic details}\label{sec:alg_comparison}

\paragraph{Comparison between ALS and ORALS}
ALS minimizes over $\ba$ and $c$ separately, capturing the joint $2$-parameter structure of the problem, at every iteration. 
This is crucial to achieve a near-optimal sample complexity of $\frac{N^2+\p}{NLd}$, where $N^2+\p$ is the number of (independent) parameters to be estimated, and $NLd$ is the total number of scalar observations in each sample trajectory.
Numerical experiments (see, for example, Figure \ref{fig:conv_M_T_randomkernel}) suggest that indeed ALS starts converging to accurate estimators as soon as the sampling size is about $M\gtrsim \frac{N^2+\p}{NLd}$, and that ALS consistently and significantly outperforms ORALS at small and medium sample sizes. 
In each of the two steps at each iteration of ALS, the update of the involved parameter is non-local, making the algorithm potentially robust to local minima in the landscape of the loss function over $(\ba,c)$: we witness estimator paths of ALS overcoming local minima and bypassing ridges in the optimization landscape to converge to a global minimizer quickly. It is certainly not the case, in general, that the optimization landscape is free of local minima.
The computational cost for ALS is also smaller than ORALS, especially as a function of $N$ and $\p$.

A major drawback of ALS is the challenge in establishing global convergence of the iterations, particularly around the near-optimal sample size, but also for large sample size. Similar problems are intensively studied in matrix sensing, where certain restricted isometry property (RIP) conditions and their generalizations are sufficient to ensure the uniqueness of a global minimum or the absence of local minima (see, \cite{GJZ2017,ZSL2017,LeeStoger2022})). However, these conditions appear not to be satisfied in our setting in general, and local minima can exist: see, e.g., Figure \ref{Fig:ErrFuns} in Appendix Section \ref{Appen_RIP} for more detailed investigations. It remains an open problem to study the convergence of the ALS algorithm in this new setting.

For the ORALS estimator, Theorem \ref{thm:AN_orals} guarantees both convergence and asymptotic normality as the number of paths $M$ goes to infinity; in practice, we observe that ORALS starts constructing accurate estimators when $M\gtrsim \frac{N^2\p}{NLd}$, where $N^2p$ is the number of parameters to be estimated in the first step of ORALS. 
The second step of ORALS is a classical rank-1 matrix factorization problem: it has an accurate solution robust to the sampling errors in the matrices $\widehat \bZ_i$ estimated in the operator regression stage. These sampling errors can be analyzed with non-asymptotic bounds by concentration inequalities and asymptotic bounds by the central limit theorem. 
 
\paragraph{Computational complexity}\label{ss:compcomplexity} 
Table \ref{tab:complexity} shows the theoretical computational complexity of ALS and ORALS, and Figure \ref{fig:compute_test} in  Section \ref{s:emp_comp_perforamance} shows the practical scaling in terms of the two fundamental parameters $M$ and $N$. The computational cost is dominated by assembling the regression matrices from the input data, whereas the solution of the linear equations takes a lower order of computations.
Observe that the data size is comparable to $MLdN$, with independence in $M$ but not in $L$ or $N$, and the number of parameters being estimated is $N^2+{\p}$. It is natural therefore to assume $M\gtrsim N^2+{\p}$ or, perhaps more optimistically assuming independence in $L$ and $N$, $MLdN\gtrsim N^2+\p$.  In a non-parametric setting, we would expect $\p$ to grow with $M$ (as in \cite{LZTM19pnas,LMT21_JMLR,WangSeroussiLu2023}, where optimal choices of $\p$ are $\p\sim M^\alpha$ for some $\alpha\in(0,1)$), so the dependency of the computational complexity on $M$ and $\p$ is of particular interest.
The summary of the computational costs is in Table \ref{tab:complexity}, and empirical measurements of wall-clock time are discussed in Section \ref{s:emp_comp_perforamance}.
\begin{table}[h]
\caption{Computational complexity of ALS (per iteration) and ORALS. Recall that the size of the input data is $MLdN$.
}\label{tab:complexity}
\centering{
\begin{tabular}{ | l ||  c | c |}
\hline
                      & \hspace{10mm} ALS \hspace{10mm} & \hspace{10mm}ORALS \hspace{10mm} \\
                                  \hline
Assembling mats/vecs  & $O(MLdN^2{\p})$  		&  $O(MLdN^3\p^2)$ \\
Solving & $O(MLdN(\p^2+N^2))$		&  $O(MLdN^3+N^4\p^3)$ \\
  \hline
Total (if $MLd>N$)  & $O(MLdN(\p^2+Np+N^2))$ & $O(MLdN^3+N^4\p^3)$ \\
 \hline
\end{tabular}	
}
\end{table}

\paragraph{Ill-posedness and regularization} \label{sec:regu}
Robust solutions to least squares problems are crucial for the ALS and ORALS algorithms.  When the matrices in the least squares problems are well-conditioned (i.e., the ratio between the largest and the smallest positive singular values are not too large), the inverse problem is well-posed, and pseudo-inverses lead to accurate solutions robust to noise. 

However, regularization becomes necessary to obtain estimators robust to noise when the matrix is ill-conditioned or nearly rank deficient. This happens when the sample size is too small, the basis functions are nearly linearly dependent, or the observation matrix is ill-conditioned due to the properties of the dynamics or the choice of the hypothesis space. The basis functions can be orthonormalized with respect to an empirical approximation to $\rho_L$ \eqref{Def:exploration}, but the other issues are fundamental to the problem. In such cases, numerical tests show that the minimal-norm least squares method and the data-adaptive RKHS Tikhonov regularization in \cite{LLA22} lead to more robust and accurate estimators than the pseudo-inverse and the Tikhonov regularization with the Euclidean norm. See more details in Appendix Section \ref{sec:regu_append}.
In this study, we consider only Tikhonov regularizers that are suitable for least squares type estimators in ALS and ORALS; of course, there is a very large literature on regularization methods (see, e.g., \cite{engl1996regularization,hansen1998rank,cucker2002_BestChoicesb,gazzola2019ir} and the references therein). 


\section{Theoretical guarantees} \label{sec:theory-ID-conv}
Three fundamental issues in our inference problem are (i) the identifiability of the weight matrix and the interaction kernel, i.e., the uniqueness of the minimizer of the loss function in the large sample limit; (ii) the well-posedness of the inverse problem in terms of the condition numbers of the regression matrices in the ALS and ORALS algorithms, and (iii) the convergence of the estimators as the sample size increases.

We address these issues by introducing coercivity conditions in Section \ref{sec:CCs}. The coercivity conditions ensure identifiability and well-posedness, enabling the study of estimator convergence in Section \ref{sec:Conv}. Intuitively speaking, the coercivity conditions guarantee that the smallest eigenvalues of the regression matrices in ORALS and ALS have a uniform lower bound with a high probability as the sample size increases; see Section \ref{sec:CC-eign}.

 Hereafter, we denote $\E$ the expectation with respect to the distribution of the data trajectories, which depends on the initial distribution of $\bX_{t_0}$, the stochastic force and the noise. We let 
\begin{equation}\label{Def:r_ij}
	 \br_{ij}(t_l):=X^{j}_{t_l}-X^{i}_{t_l}\quad \text{and}\quad
	 \mathbf{r}_{ij}^{m}(t_l):=X^{j,m}_{t_l}-X^{i,m}_{t_l}\,. 
\end{equation}
We say the true parameter $(\ba^{*},\IK_*)$ is \emph{identifiable} if it is the unique zero of the loss function in the large sample limit 
\begin{align*}
	\calE_{L, \infty}(\ba,\phi) 
	&=\frac{1}{L}\sum_{i=1}^N\sum_{l=0}^{L-1}\E \bigg[\bigg|  \sum_{j\neq i} [\ba_{ij} \IK(\mathbf{r}_{ij}(t_l))-\ba_{ij}^{*} \IK_*(\mathbf{r}_{ij}(t_l))]\bigg|^2 \bigg]\,, 
\end{align*}
when the data has no noise and when the model is deterministic. We say the inverse problem is well-posed if the estimator is robust to noise or sampling error.

\subsection{Exploration measure}

\label{s:identifiabilityandconvergence}
We define a function space $L^2(\rho_L)$ for learning the interaction kernel, where $\rho_L$ is a probability measure that quantifies data exploration to the interaction kernel.  Let 
\begin{equation}\label{Def:r_ij}
	 \br_{ij}(t_l):=X^{j}_{t_l}-X^{i}_{t_l}\quad \text{and}\quad
	 \mathbf{r}_{ij}^{m}(t_l):=X^{j,m}_{t_l}-X^{i,m}_{t_l}\,.
\end{equation}
These pairwise differences $\{\br_{ij}^m(t_l)\}$ are the independent variable of the interaction kernel. Thus, we define $\rho_L$ as follows.  
\begin{definition}[Exploration measure]
With observations of $M$ trajectories at the discrete times $\{t_l\}_{l=0}^{L-1}$,
we introduce an empirical measure, and its large sample limit, on $\R^d$, defined as
\begin{align}
	\rho_{L,M}(d\br) &:= \frac{1}{ (N-1)NLM}\sum_{l=0}^{L-1}\sum_{m=1}^{M} \sum_{1\leq i\neq j\leq N} \delta_{\br_{ij}^{m}(t_l)}(d\br), \label{Def:empirical} \\
	 \rho_L(d\br) &:=\frac{1}{(N-1)N L}\sum_{l=0}^{L-1} \sum_{1\leq i\neq j\leq{N}} \E[ \delta_{\br_{ij}(t_l)}(d\br)] 
	 	 \label{Def:exploration}\,,
\end{align} 
where $\sum_{1\leq i\neq j\leq{N}}$ stands for $\sum_{i=1}^N \sum_{j=1,j\neq i}^{N}$. 
\end{definition}
The empirical measure depends on the sample trajectories, but $\rho_L$ is the large sample limit, uniquely determined by the distribution of the stochastic process  $\bX_{t_0:t_{L-1}}$, and hence data-independent.

\subsection{Two coercivity conditions}\label{sec:CCs}
The joint inference of the network $\ba$ and the kernel $\Phi$ suffer non-identifiability in general, as shown in Example \ref{example:non-identifiable}. 
We introduce two types of coercivity conditions to ensure the identifiability of the system and the well-conditionedness of the regression matrices in ALS and ORALS when the basis functions are orthonormal.  
The first one is a joint type, including two conditions that we call rank-1 and rank-2 joint coercivity, which guarantee that the bilinear forms defined by the loss function in terms of either the kernel or the weight matrix are coercive (recall that a bilinear function $f(x,y)$ is coercive in a Hilbert space $\calH$ if $f(x,x)\geq c \|x\|_{\calH}^2$ for any $x\in\calH$, see \cite{LaxFunc}). 
  
\begin{definition}[Joint coercivity conditions] \label{def:jointCC} 
	The system \eqref{eq:ips_K} is said to satisfy a \emph{rank-1 joint coercivity condition} on a hypothesis function space $\calH\subset L^2(\rho_L)$ with constant $c_{\calH}>0$ if for all $\IK\in\calH$ and all $\ba\in \mathcal{M}$,
	\begin{align}\label{eq:jointCC}
		\frac{1}{L}\sum_{l=0}^{L-1} \E \bigg[\bigg|  \sum_{j\neq i} \ba_{ij} \IK(\mathbf{r}_{ij}(t_l))\bigg|^2 \bigg]\
		&\geq  c_{\calH} |\ba_{i\cdot}|^2  \|\IK\|_{\rho_L}^2 \,, \quad \forall ~ i\in[N] \,. 
	\end{align}
	We say system \eqref{eq:ips_K} satisfies a \emph{rank-2 joint coercivity condition} on $\calH$ if there exists a constant $c_{\calH}>0$ such that for all $\IK_1,\IK_2\in\calH$ with $\langle\IK_1,\IK_2\rangle_{L^2(\rho_L)}=0$, all $\ba^{(1)},\ba^{(2)}\in \mathcal{M}$,  and all $i\in[N]$,
	\begin{equation}\label{eq:jointCC2}
		\begin{aligned}
			\frac{1}{L}\sum_{l=0}^{L-1} &\E \bigg[\bigg|  \sum_{j\neq i} [\ba_{ij}^{(1)} \IK_1(\mathbf{r}_{ij}(t_l))+\ba_{ij}^{(2)} \IK_2(\mathbf{r}_{ij}(t_l))]\bigg|^2 \bigg] \geq c_{\calH} \left[|\ba_{i\cdot}^{(1)}|^2  \|\IK_1\|_{\rho_L}^2 +|\ba_{i\cdot}^{(2)}|^2  \|\IK_2\|_{\rho_L}^2 \right]\,.
		\end{aligned}
	\end{equation}
\end{definition}
Note that \eqref{eq:jointCC2}, by taking $\IK_2=0$, implies \eqref{eq:jointCC}.  
	 
The rank-1 joint coercivity condition \eqref{eq:jointCC} ensures that the regression matrices in any iteration of ALS are invertible with the smallest singular values bounded from below; see Proposition \ref{prop:Abar_orals}. However, it does not guarantee identifiability. The stronger rank-2 joint coercivity condition \eqref{eq:jointCC2}  does provide a sufficient condition for identifiability. The rank-2 joint coercivity condition is sufficient for identifiability (the proof is postponed to Appendix \ref{Append:A}):

\begin{proposition}[Rank-2 Joint coercivity implies identifiability]\label{Prop:r2JCC_Ident}
	Let the true parameters be $\ba^{*}\in \mathcal{M}$ and $\IK_* \in \calH\backslash \{0\}\subset L^2_\rho$. Assume the \emph{rank-2 joint coercivity condition} holds with $c_{\calH}>0$. Then, we have identifiability, i.e. $(\ba^{*},\IK_*)$  is the unique solution to  $\calE_{L, \infty}(\ba,\IK)=0$. 
\end{proposition}

The joint coercivity conditions may be viewed as extensions of the Restricted Isometry Property (RIP) in matrix sensing  \cite{RFP2010}  to our setting of joint estimation, with the coercivity constant corresponding to the lower bound in the RIP conditions. 
Important differences are that (a) our conditions take into account the randomness in the initial conditions and in the forcing term, in the large sample limit (expectation) sense, while the RIP conditions in matrix sensing, in the large sample size limit, are automatically satisfied as the measurement operator converges to an identity; (b) coercivity is independent of basis selection, unlike RIP which is dependent on a choice of basis (see e.g., \cite{Liu_NIPS2011}).
Furthermore, several works, including \cite{BahmaniRomberg17a, GJZ2017, LeeStoger2022, Chandrasekher2022alternating}, show that a small enough RIP constant (which would correspond to a large coercivity constant $c_{\calH}$ in our setting)  implies that suitable optimization algorithms attain the global minimizer. 
This, however, is often not the case in our setting, as discussed in Appendix \ref{Appen_RIP}, and local minima may exist in our problem (see Figure \ref{Fig:ErrFuns}).

We now introduce the \emph{interaction kernel coercivity condition}, which guarantees identifiability and well-posedness, as well as the invertibility of the regression matrix in ORALS with high probability for large sample size. 
It therefore ensures the uniqueness of the minimizer of the loss function and the identifiability of both the weight matrix and the kernel, since the second stage in ORALS is similar to a rank-1 factorization of a matrix, which always has a(n essentially) unique solution.

\begin{definition}[Interaction kernel coercivity condition]\label{def:kernelCC}
The system \eqref{eq:ips_K} satisfies the interaction kernel coercivity condition in a hypothesis function space $\mathcal H\subseteq L^2(\rho_L)$ with a constant $c_{0,\calH}\in (0,1)$, if for any $\IK\in\calH$ and $i\in[N]$
\begin{equation}\label{eq:kernelCC}
\frac{1}{L(N-1)}\sum_{l=0}^{L-1}\sum_{j\neq i} \E[\trCov( \IK(\mathbf{r}_{ij}(t_l)) \mid \calF_l^i)]\geq c_{0, \calH} \|\IK\|_{\rho_L}^2,   
\end{equation}
where $\calF_l^i := \calF(\bX_{t_{l-1}}, X_{t_l}^i)$ is the $\sigma$-algebra generated by $(\bX_{t_{l-1}}, X_{t_l}^i)$. Here $\trCov( \IK(\mathbf{r}_{ij}(t_l)) \mid \calF_l^i)$ is the trace of the covariance matrix of the $\R^d$-valued random variable $\IK(\br_{ij}(t_l))$ conditional on $\calF_l^i$. 
\end{definition}
Condition \eqref{eq:kernelCC} is inspired by the well-known De Finetti theorem (see \cite[Theorem 1.1]{Kallenberg2005}), which shows that an exchangeable infinite sequence of random variables is conditionally independent relative to some latent variable. 
This condition holds, for example, when $L=1$ and the components $\{X^i\}_{i=1}^N$ are independent, because $\br_{ij}=X^j-X^i$ and $\br_{ij'}=X^{j'}-X^i$ are pairwise independent conditioned on $X^i$; see \cite[Section 2]{WangSeroussiLu2023} for a discussion in the case of radial interaction kernels.  
The interaction kernel coercivity condition implies the joint coercivity conditions; see Proposition \ref{prop:Kernelcc-Jcc}. 
Also, we can prove that the interaction kernel coercivity condition holds for radial kernels, with initial conditions having certain Gaussian distributions and $L=1$ (see Proposition {\rm\ref{prop:kernelCC_example}}).
The rank-1 joint coercivity condition can also be viewed an extension of the classical coercivity condition in \cite{BFHM:LearningInteractionRulesI} and {\rm \cite[Definition 1.2]{LLMTZ21}}, which was introduced for homogeneous systems (with $\mathbf{a}\equiv1$ except for $0$'s on the diagonal) with radial interaction kernel, i.e., $\IK(x)= \tilde\IK(|x|)\frac{x}{|x|}$.
For homogeneous systems, 
we present a detailed discussion on the relation between these conditions in Section {\rm \ref{sec:CC-relation}}.

\subsection{Coercivity and the smallest singular values of regression matrices}\label{sec:CC-eign}
We show that coercivity conditions imply that the normal matrices in ORALS and ALS are nonsingular, with their eigenvalues bounded from below by a positive constant, with high probability.  We consider hypothesis spaces satisfying the following conditions. 
\begin{assumption}[Uniformly bounded basis functions]\label{assum:CC}
The basis functions of the hypothesis space $\mathcal H=\mathrm{span} \{ \psi_1,\cdots,\psi_{\p} \}$ are orthonormal in $L^2(\rho_L)$ and uniformly bounded, i.e., $\sup_{k\in[\p]} \|\psi_k\|_\infty$ $ \leq L_{\mH}$.
\end{assumption}

he next proposition shows that the smallest singular values of the regression matrices in ORALS and ALS are bounded from below by the coercivity constants, with high probability. 
We defer its proof to Appendix \ref{sec:CC_min_eig}.
\begin{proposition}\label{prop:Abar_orals} 
Assume $\{\psi_k\}_{k\in[p]}$ satisfies Assumption {\rm\ref{assum:CC}} and $\calH=\mathrm{span} \{\psi_k\}_{k\in[p]}$. Then:
\begin{itemize}
\item[(i)] under the kernel coercivity condition \eqref{eq:kernelCC}, the matrix in the Operator Regressions stage of ORALS is well-conditioned: for each $i\in [N]$, the matrix $\mathcal{A}_{i,M}$ in \eqref{eq:Ai_operator} satisfies $\frac{1}{M} \sigma^2_{\min}(\E[\mathcal{A}_{i,M}]) $ $> c_{\mH}$; moreover, for $\epsilon>0$ and any $M$,
\begin{align}\label{Ineq:Conc_calA}
	\P\bigg\{ \frac{1}{M}\sigma^2_{\min}(\mathcal{A}_{i,M})> c_{\mH}-\epsilon \bigg\}\geq 
	1-2{\p}N\exp\bigg( -\frac{M\epsilon^2/2}{2(\p N L_{\mH}^2)^2+\p NL_{\mH}^2\varepsilon/3}\bigg)\,;
\end{align}
\item[(ii)] under the rank-1 joint coercivity condition \eqref{eq:jointCC}, the matrices in the least squares problems in the ALS algorithm are well-conditioned:
\begin{itemize}
 \item[(a)] in the estimation of $\ba_i$ with a given nonzero $c\in \R^{\p}$, we have that  $\frac{1}{M}\sigma_{\min}^2(\E[\mathcal{A}^{\textrm{ALS}}_{c,M,i}])\ge c_{\mH}||c||^2$ for each $i\in[N]$ and the matrix in \eqref{eq:ALS_graphest} is well-conditioned. Moreover, for any $M$ and $\epsilon>0$,
 \begin{equation}
 \label{e:Conc_AALS1}
	\P\bigg\{ \frac{1}{M} \sigma^2_{\min}(\mathcal{A}^{\textrm{ALS}}_{c,M,i}) \geq c_{\mH} \| c\|^2-\epsilon \bigg\}\ge 1-2N\exp\rbracket{-\frac{M\varepsilon^2/2}{(\p L_{\mH}^2)^2+pL_{\mH}^2\varepsilon/3}}\,;
 \end{equation}
\item[(b)] in the estimation of $c\in \R^{\p}$ with a given $\ba$ with $\|\ba_i\|=1$, we have that $\frac{1}{M}\sigma_{\min}^2(\E[\mathcal{A}^{\textrm{ALS}}_{\ba,M,i}])\ge c_{\mH}$ for each $i\in[N]$ and the matrix in \eqref{eq:ALS_kernelest} is well-conditioned.  Moreover, for any $M$ and $\epsilon>0$,
\begin{equation}
\label{e:Conc_AALS2}
	\P\bigg\{\frac{1}{M} \sigma^2_{\min}(\mathcal{A}^{\textrm{ALS}}_{\ba,M,i})\geq c_{\mH}-\epsilon \bigg\}\ge1-2{\p} \exp\rbracket{-\frac{M\varepsilon^2/2}{(N L_{\mH}^2)^2+NL_{\mH}^2\varepsilon/3}}\,.
\end{equation}
\end{itemize}
\end{itemize}
\end{proposition}
Note that already in this result, the bounds \eqref{e:Conc_AALS1} and \eqref{e:Conc_AALS2} for ALS only require $M\gtrsim (N^2+p^2)(\log N+\log p)$ to the lower bound to be positive (where $p^2$ may be replaced by $p$ with more refined arguments, such as the PAC-Bayes argument applied in the proof of \cite[Lemma 3.12]{WangSeroussiLu2023}). In contrast, the bound \eqref{Ineq:Conc_calA} for ORALS requires $M\gtrsim (pN)^2\log(pN)$, in line with our discussion of the expected sample size requirements of ORALS and ALS.

\subsection{Convergence and asymptotic normality of the ORALS estimator}
\label{sec:Conv}
Convergence of the ORALS estimator follows from the kernel coercivity condition. We will prove that the estimator is consistent (i.e., it converges almost surely to the true parameter) and is asymptotically normal. 
Here, for simplicity, we consider the case when the data are generated by an Euler-Maruyama discretization of the SDE \eqref{eq:ipog}. The case of discrete-time data from continuous paths can be treated by careful examinations of the stochastic integrals and their numerical approximations, using arguments similar to those in \cite{LMT21}.  

\begin{theorem}\label{thm:AN_orals}
Assume that the kernel coercivity condition \eqref{eq:kernelCC} holds on $\calH:=\mathrm{span} \{\psi_k\}_{k\in[p]}$ with $\{\psi_k\}_{k\in[p]}$ satisfies Assumption {\rm\ref{assum:CC}}, and that the data \eqref{eq:data} is generated by the Euler-Maruyama scheme 
\begin{equation}\label{eq:Euler}
 	\Delta\bX_{t_l}:=\bX_{t_{l+1}}- \bX_{t_{l}}  = \ba_* \bB(\bX_{t_l}) c_* \Delta t+\sigma \sqrt{\Delta t}\bW_{l}, 
\end{equation}
where  $\ba_*$ and $c_*$ are the true parameters, $\{\bW_{l}\}_l$ are independent, with distribution $\calN(0,I_{{Nd}})$, and $\|(\ba_{*})_{i}\|_2=1$ for each $i\in[N]$.   
Then we have:
\begin{enumerate}
	\item[(i)] The estimator $\widehat{z}_{i,M}$ in \eqref{eq:est_matZ} is asymptotically normal for each $i$. More precisely, $\widehat{z}_{i,M}  = z_{i} + {\xi}_{i,M} $, where $z_{i}={\rm Vec}(\bZ_i)$, with $\bZ_i= (\ba_{*})_{i}^\top c_*^\top$, and ${\xi}_{i,M}$ is a centered $\R^{(N-1){\p}}$-valued random vector s.t.
	{$\sqrt{M}{\xi}_{i,M} \xrightarrow[]{d} \overline{{\xi}}_{i,\infty} \sim \calN(0,\sigma^2 \Delta t \overline{\mathcal{A}}_{i,\infty}^{-1} )$}.

	\item[(ii)] Starting from any $c_0\in\R^{\p}$ such that $c_*^\top c_0 \neq 0$, the first iteration $\widehat c^{M,1}$ and second iteration estimator $\widehat \ba^{M,2}$ for the deterministic ALS in \eqref{eq:est_Ac} are consistent up to a change of sign and are asymptotically normal:
\begin{align*}
	\sqrt{M}[\widehat c^{M,1} - \emph{sgn}(c_*^\top c_0) c_*] & \overset{d}{\to} \frac 1N\sum_{i=1}^N \pmb{\xi}_i^\top (\ba_{*})_{i}^\top \,,\\
	\sqrt{M}[(\widehat\ba_i^{M,2})^\top - \emph{sgn}(c_*^\top c_0) (\ba_*)_i^\top] & \overset{d}{\to}  |c_*|^{-2} [\pmb{\xi}_i c_*-(\ba_{*})_{i} \pmb{\xi}_i c_* (\ba_{*})_{i}^\top], 
\end{align*} 
where the random matrix $\pmb{\xi}_i\in \R^{(N-1)\times\p}$ is the vectorized form of the Gaussian vector $\overline{{\xi}}_{i,\infty} $ in (i), i.e., $\overline{{\xi}}_{i,\infty} =\rm{Vec}(\pmb{\xi}_i)$.
\end{enumerate}
\end{theorem}

Convergence of the ALS estimator remains an open question. It involves two layers of challenges: the convergence in the iterations, and the convergence as the sample size increases. The restricted isometry property (RIP)  conditions with sufficiently small RIP constants are typically stronger than the joint coercivity conditions used here, but enable one to construct estimators via provably convergent optimization algorithms from data of small size  \cite{RFP2010,BahmaniRomberg17a, GJZ2017, LeeStoger2022, Chandrasekher2022alternating}.  However, these conditions appear to be seldom satisfied in our setting. 

Figure \ref{Fig:RIP_Coer_relation} shows the relations between the coercivity, RIP conditions, and their consequences.

\tikzset{unode/.style = {
    circle, 
    draw=cyan!30!black, 
    thick,
    fill=cyan!80!black,
    inner sep=2.3pt,
    minimum size=2.3pt }}
    
\tikzset{roundnode/.style={
	circle, 
	draw=black!80, 
	fill=gray!10, 
	very thick, 
	minimum size=7mm}}

\tikzset{squarednode/.style={
	rectangle, 
	draw=black!80, 
	fill=gray!10, 
	very thick, 
	minimum size=5mm}}

\tikzset{squarednode2/.style={
	rectangle, 
	draw=orange!80, 
	fill=yellow!10, 
	very thick, 
	minimum size=5mm}}
	
\tikzset{ellipsednode/.style={
	ellipse, 
	draw=black!80, 
	fill=cyan!20, 
	very thick, 
	minimum size=5mm}}
\tikzset{ellipsednode2/.style={
	ellipse, 
	draw=black!80, 
	fill=orange!20, 
	very thick, 
	minimum size=5mm}}

\begin{figure}[hbt]
\centering
\scalebox{0.7}{ 
\begin{tikzpicture}
\begin{scope}[xshift=2cm,yshift=2.5cm]
    \node[squarednode, align=center] (A) at (0,0){Rank 2-Joint \\ \textbf{coercivity}};
    \node[squarednode2, align=center] (B) at (5,0){Rank 2-RIP \\ condition};
    \node[squarednode2, align=center] (C) at (5,4){Rank 1-RIP \\ condition};
    \node[squarednode, align=center] (D) at (0,4){Rank 1-Joint \\ \textbf{coercivity}}; 
    \node[roundnode, align=center] (E) at (-5,2){Kernel \\ \textbf{coercivity}}; 
    \node[] (F) at (0,2){}; 
    \node[ellipsednode, align=center] (G) at (-5,5.2){ALS  well-conditioned};
    \node[ellipsednode, align=center] (H) at (-5,-1.2){Identifiability};
   \node[ellipsednode, align=center] (I) at (-9,2){\rotatebox{90}{\parbox{4.5cm}{ORALS well-conditioned \\ Asymptotic normality}}};
    
    \draw[ultra thick, black!100!gray, fill=black!70, opacity=0.3, rounded corners] (-1.5,-0.8) --
      (1.5,-0.8) --  (1.5,4.8) --  (-1.5,4.8) -- cycle;  
    \draw[ultra thick, orange!70, fill=orange!70, opacity=0.3, rounded corners] (6.4,-0.8) --
      (3.6,-0.8) --  (3.6,4.8) --  (6.4,4.8) -- cycle; 
      
	    \draw[line width=5pt, shorten <=2pt, shorten >=42pt, -latex] (E) to [auto]node[above] {Prop \ref{prop:Kernelcc-Jcc} $\qquad\qquad$ } (F);
	    \draw[line width=4pt, shorten <=2pt, shorten >=2pt, -latex] (C) to [auto]node[midway,above] {Prop \ref{Prop:RIPtoJCC}} (D);
	    \draw[line width=4pt, shorten <=4pt, shorten >=4pt, -latex] (A.north) to [out=90,in=-90]node[midway,above,sloped] {Def \ref{def:jointCC}} (D.south);
	    \draw[line width=4pt, shorten <=4pt, shorten >=4pt, -latex] (B.north) to [out=90,in=-90] node[midway,above,sloped] {Def \ref{Def:RIP}}(C.south);
	    \draw[line width=4pt, shorten <=2pt, shorten >=2pt, -latex] (B) to [auto]node[midway,below] {Prop \ref{Prop:RIPtoJCC}} (A);
	    \draw[line width=4pt, shorten <=2pt, shorten >=2pt, -latex] (D) to [auto]node[midway,below,sloped] {Prop \ref{prop:Abar_orals}} (G);
	    \draw[line width=4pt, shorten <=2pt, shorten >=2pt, -latex] (A) to [auto]node[midway,above,sloped] {Prop \ref{Prop:r2JCC_Ident}} (H);
	    \draw[line width=5pt, shorten <=2pt, shorten >=2pt, -latex] (E) to [auto]node[above]{Thm \ref{thm:AN_orals}} (I);
	    \draw[line width=5pt, shorten <=4pt, shorten >=4pt, -latex] (E.north) to [auto]node[midway,above,sloped] {Prop \ref{prop:Abar_orals}} (G.south);
	    \draw[line width=5pt, shorten <=4pt, shorten >=4pt, -latex] (E.south) to [auto]node[midway,below,sloped] {Prop \ref{prop:Kernelcc-Jcc}} (H.north);
\end{scope} 
\end{tikzpicture}
}
\caption{The coercivity conditions: connections with RIP conditions, identifiability, and well-conditionedness of ALS and ORALS algorithms. }
\label{Fig:RIP_Coer_relation}
\end{figure}
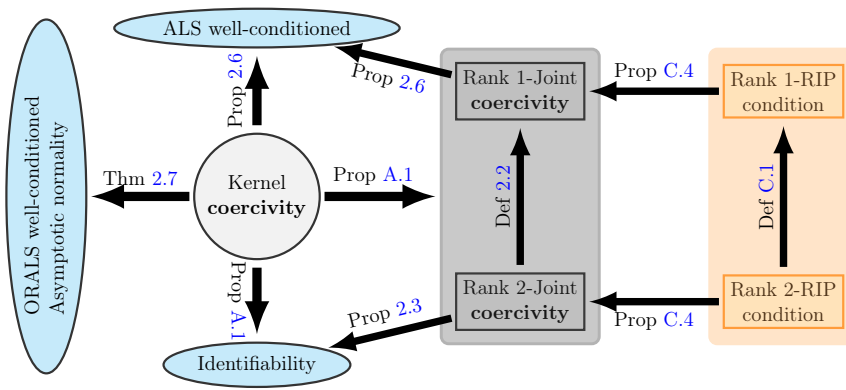

\vspace{-2mm}
\subsection{Trajectory prediction}
In the above, we have studied the accuracy of our estimator in terms of the Frobenius norm on the graph weight matrix and $L^2(\rho_L)$ norm on the interaction kernel. It is also of interest to ask whether the dynamics generated by our estimated system are close to the ones of the true system; in particular, whether we can control the trajectory prediction error by the error of the estimator. The following proposition provides an affirmative answer, similar to the previous results in \cite[Proposition 2.1]{LMT21}, at least for short times: 
\begin{proposition}[Trajectory prediction error]
\label{p:trajerrbounds}
	Let $(\widehat\ba,\widehat c)$ be an estimator of $(\ba, c)$ in the system \eqref{eq:ipog}, where $\widehat\ba$ and $\ba$ are row-normalized. 
	Assume that the basis functions $\{\psi_k\}_{k=1}^{\p}$ are in $\{\psi\in C_b^1(\R^d): \|\psi\|_\infty+\|\nabla \psi\|_\infty\leq C_0 \}$, for some $C_0>0$. Denote by $(\widehat \bX_t)_{0\leq t\leq T}$ and $(\bX_t)_{0\leq t\leq T}$ the solutions to the systems $\dynsys {\widehat\ba}{\widehat c}$ and $\dynsys {\mathbf{a}}{c}$, associated to $(\widehat \ba,\widehat c)$ and $(\ba, c)$ respectively, started from the same initial condition sampled from $\mu$, and driven by the same realization of the stochastic force $\bW_t$. 
	Then,
\begin{equation}\label{Ineq:Traj_err}
	\sup_{t\in[0,T]} \E\left[\|\widehat{\bX}_t-\bX_t \|_F^2\right]\leq C_1T^2 e^{2C_1 C_2 T} \left( C_2 \|\ba-\widehat\ba\|_F^2+ \|\widehat c-c\|_2^2 \right) \,,
\end{equation}
with $C_1 := 2{\p}C_0^2 $ and $C_2 := \|\widehat c\|_2^2+ \| c\|_2^2$.
\end{proposition}

We remark that while these bounds are useful only for short times, for systems that have stable attractors, or exhibit stable emergent behaviors, one may be cautiously optimistic about the possibility that the attractors, or emergent behaviors, will be reproduced by the system $\dynsys {\widehat\ba}{\widehat c}$, as rather extensively experiments in \cite{DiscoveryEmergentBehaviors} suggest.


\section{Numerical experiments}\label{sec:numExper}
We examine the ALS and ORALS algorithms numerically in terms of the dependence of their accuracy on three key parameters: sample size, level of observation noise, and the strength of the stochastic force. Also, we examine their robustness with respect to possible misspecification of the model (i.e. the hypothesis space does not contain the true interaction kernel). The code implementing these tests is available in \url{https://github.com/LearnDynamics/ips-graph.git}. 

ALS appears to be particularly efficient and robust, both statistically and computationally, as soon the number of observations is comparable (up to constants and possibly logarithmic terms) to the number of unknown parameters $(\ba, c)$; and its estimator converges as sample size increases, although it does not have theoretical guarantees; ORALS performs as well as ALS in the large sample regime, with both estimators converging at the theoretical rate $M^{-1/2}$.

The settings of the systems in our experiments are as follows. There are $N =6$ agents in a relatively sparse network in which each agent is influenced by $|\calN_i|\equiv 2$ other agents, selected uniformly at random. The non-zero off-diagonal entries of the weight matrix are randomly sampled independently from the uniform distribution in $[0,1]$ followed by a row-normalization, i.e., $a_{ij}\in [ 0,1]$, $a_{ii}=0$, and $\sum_{j = 1}^N a_{ij}^2 = 1$ for each $i\in[N]$.
 The state vector $X_t^i$ is in $\R^d$ with $d=2$. The interaction potential is a version of the Lennard-Jones potential $\IK(x) = \phi(|x|)\frac{x}{|x|}$ with a cut-off near $0$: the interaction kernel $\phi$ given by
\begin{equation}\label{eq:LJ1_kernel}
	\phi(x) = \left\{
	\begin{aligned}
		&-\frac{1}{3}x^{-9} + \frac{4}{3}x^{-3}, & x \geq 0.5; \\
		&-160, & 0 \leq x < 0.5. 
	\end{aligned}
	\right.
\end{equation}

We consider a parametric from $\phi = \sum_{k = 1}^{\p} c_k\psi_k$ with basis functions
$$
\{\psi_{1+k} = x^{-9}\mathbbm{1}_{[0.25k+0.5, +\infty]}\}_{k=0}^2\ \cup\ 
\{\psi_{4+k} = x^{-3}\mathbbm{1}_{[0.25k+0.5, +\infty]}\}_{k=0}^2\ \cup\  
\{\psi_{7+k} = \mathbbm{1}_{[0, 0.25k+0.5]}\}_{k=0}^3\,. 
$$
Thus, the true parameters $c^*$ has zero components except for $(c_1^*, c_4^*, c_7^*) = (-1/3, 4/3, -160)$. Note that we do not assume or enforce sparsity in our estimation procedure.

The multi-trajectory synthetic data \eqref{eq:data} are generated by the Euler-Maruyama scheme with $\Delta t = 10^{-4}$, and with initial condition $\bX_{t_1}=(X_{t_1}^i, i = 1,\dots, N)$ sampled component-wise from a initial distribution $\mu_0$. 
The distribution $\mu_0$, stochastic force $\sigma$, the observation noise strength $\sigma_{obs}$, and total time $T$, will be specified in each of the following tests.  The number of iterations in ALS is limited to $10$ in all examples.

We report the following measures of estimation error, called the (relative) {\it graph error}, {\it kernel error}, and {\it trajectory error} respectively:
\begin{align*}
	\varepsilon_{\ba}  = \frac{\norm{\ba_* - \widehat \ba}_F}{\norm{\ba}_F}\quad,  \, \quad
	\varepsilon_K  = \frac{||{\IK - \widehat \IK}||_{L^2_\rho}}{\norm{\IK}_{L^2_\rho}}\quad,  \, \quad
	\varepsilon_{\bX}  = \frac{1}{M'}\sum_{m' = 1}^{M'}\frac{||({\bX^{m'}_t)_t - (\widehat \bX^{m'}_t)_t}||_{L^2(0, T)}}{\norm{(\bX^{m'}_t)_t}_{L^2(0, T)}},
\end{align*}
where $(\bX^{m'}_t)_t$ and $(\widehat \bX^{m'}_t)_t$ denote trajectories started from new random initial conditions, generated with the true graph and interaction kernel and with the estimated ones, respectively. The measure $\rho$ is the exploration measure defined in \eqref{Def:exploration}; since it is unknown, we use a large set of observations independently of the training data set to estimate it; note that, of course, such estimate of $\rho$ is not used in the inference procedure -- it is only used to assess and report the errors above. 

\begin{figure}[t]
	\centering
	\includegraphics[width=\textwidth]{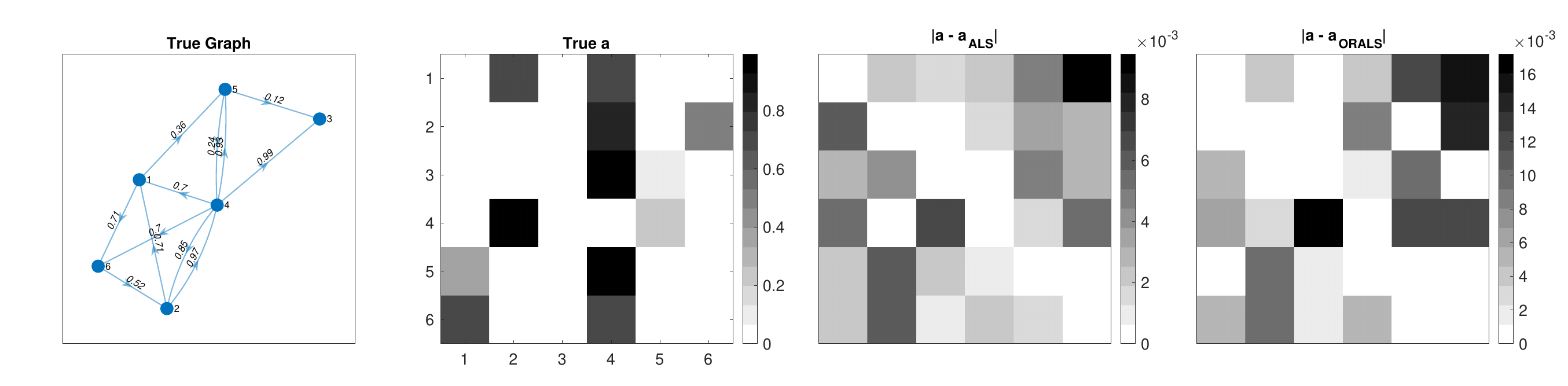}
        \includegraphics[width=\textwidth]{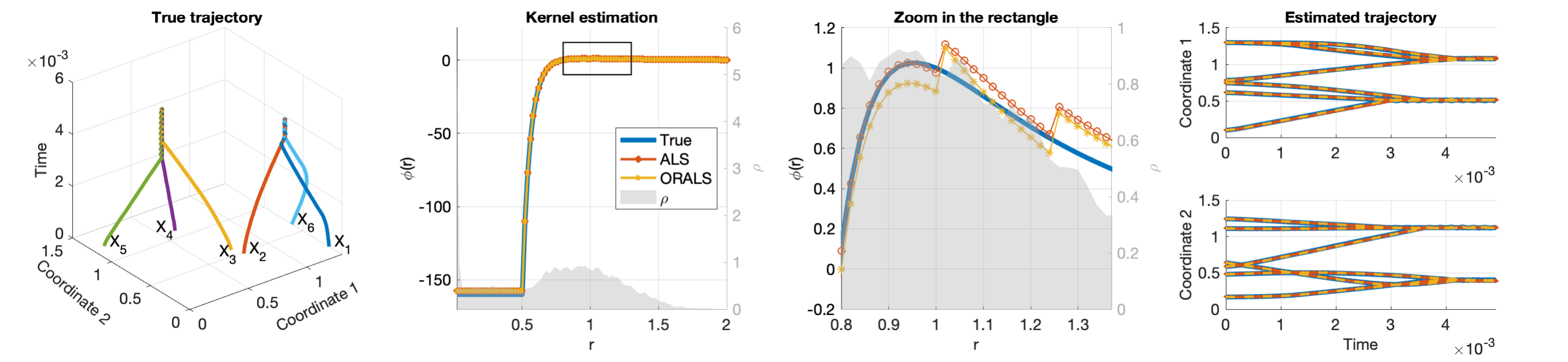}
	\caption{ \textbf{Top:} a typical weight matrix estimation. The first two columns show the true graph and its weight matrix. The two columns on the right show the entry-wise errors of the ALS and ORALS estimators.  
	\textbf{Bottom:} Estimator of interaction kernel and trajectory prediction. The left column shows a true trajectory. The middle two columns show the true and estimated kernels with a zoom-in to show the details in a rectangular region. The fourth column presents the true (the same as in column 1) and predicted trajectories. Note that $X_3$ and $X_2$ do not converge to the same cluster, in both the true and estimated trajectory, even though they are close at time 0, since there are no edges between them in the graph. 
		}
	\label{fig:graph_true_est}
\end{figure}

\subsection{A typical estimator and its trajectory prediction}

In this section, we show a typical instance of the estimators. The initial distribution $\mu_0$ is the uniform distribution over the interval $[0, 1.5]$,  the training dataset has $M = 10^3$ trajectories, the stochastic force has $\sigma =  10^{-3}$, the observation noise has $\sigma_{obs} = 10^{-3}$, and time $T = 0.005$ (i.e., making observations at $L=50$ time instances). 
Figure \ref{fig:graph_true_est} shows the graph, the kernel, the trajectory, and their estimators. 
Our algorithms return accurate estimates of the graph and the kernel; see the estimation errors in Table \ref{tab:error_typical}. We also present the mean and SD of the trajectory prediction errors of $100$ independent trajectories sampled from the initial distribution.

\begin{table}[h!]
\centering
\begin{small}
\begin{tabular}{ |p{1.1cm}||p{2.8cm}|p{2.8cm}|p{2.8cm}||p{4cm}|}
 \hline
  & Graph error $\varepsilon_\ba$ & Kernel error $\varepsilon_K$ & Traj. error $\varepsilon_{\bX}$	& Exp. traj. error $\varepsilon_{\bX}$ \\
 \hline
 	ALS  & $8.47\times10^{-3}$ & $1.45\times10^{-2}$    & $6.1\times10^{-3}$ 	& $6.19\times10^{-3}\pm 8.12\times 10^{-4}$\\
 \hline
 	ORALS  & $1.67\times10^{-2}$ & $1.47\times10^{-2}$ & $6.6\times10^{-3}$	& $7.41\times10^{-3} \pm 1.07\times 10^{-3}$\\
 \hline
\end{tabular}
\end{small}
\caption{Relative errors of the estimators in Figure \ref{fig:graph_true_est} in a typical simulation, and in the fourth column, the mean and SD of trajectory prediction errors of 100 random trajectories.  }\label{tab:error_typical}
\end{table}

\begin{figure}[t]
\centering
\includegraphics[width=0.8\textwidth]{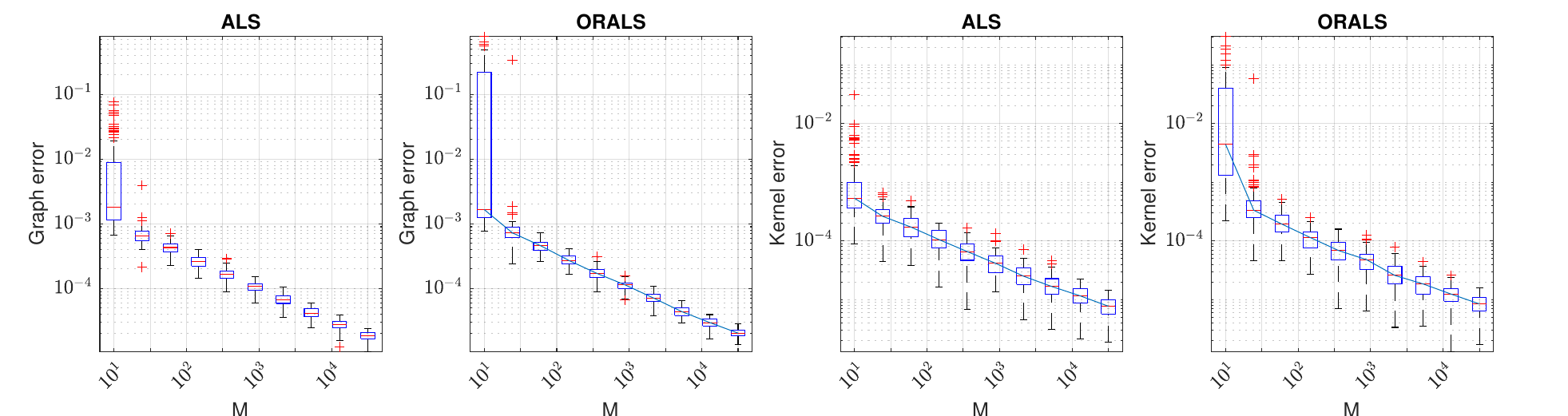}
\includegraphics[width=0.8\textwidth]{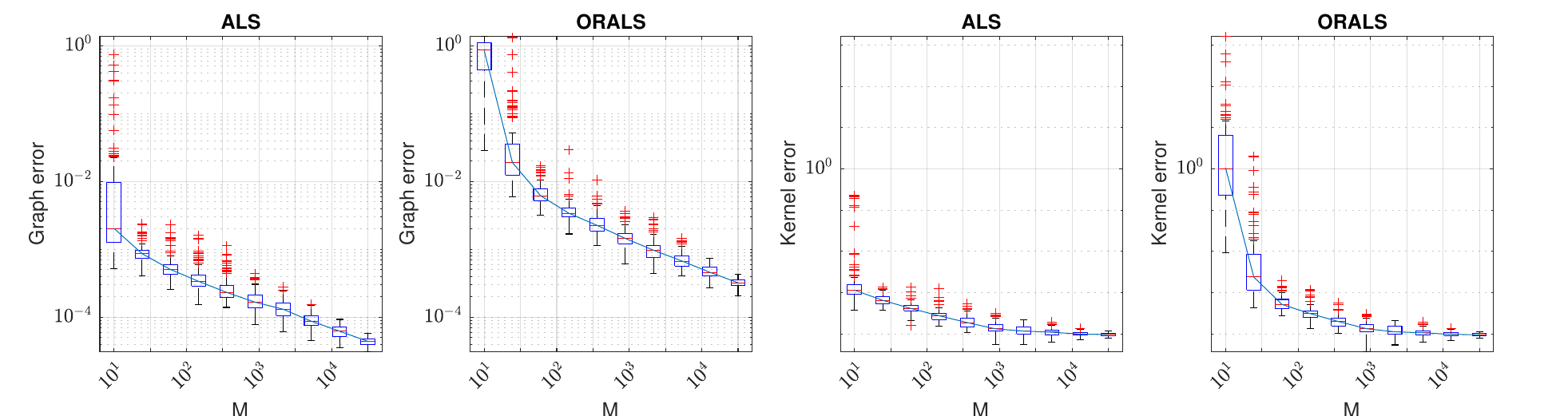}
\caption{Convergence with sample size M increasing, $M\in \{10, 24, 59, 146,359,879, 2154,$ $  5274,12915, 31622\}$, in 100 independent experiment runs. The top row shows almost perfect rates of $M^{-1/2}$ for both algorithms for the case of noiseless data and a well-specified basis. For the case of noisy data, the bottom row shows robust convergence with the errors decaying until they reach $10^{-4}$, the variance of observation noise. 
}\vspace{-2mm}
\label{fig:conv_M}
\end{figure} 

\subsection{Convergence and robustness}
We examine the convergence rates of the estimators in terms of the sample size $M$ and their robustness to basis misspecification and noise in data. 
Thus, we consider two cases: a case with noiseless data and a well-specified basis $\{\psi_1,\psi_4,\psi_7\}$, which we aim to show the convergence rate of $M^{-1/2}$, as proved for the parametric setting; and a case with noisy data with $\sigma_{obs} = 10^{-2}$  and the above basis functions $\{\psi_k\}_{k=1}^7$, which we aim to test the robustness of the convergence. 

Figure \ref{fig:conv_M} shows that both ALS and ORALS yield convergent estimators as the sample size $M$ increases. Here, the data trajectories are generated from the system with a stochastic force with $\sigma=10^{-2}$. In either case, the box plots show the relative errors in 100 random simulations. In each simulation, we compute a sequence of estimators from $M$ sample trajectories, for increasing values of $M$. 
In each box plot, the central mark indicates the median, and the bottom and top edges of the box indicate the 25th and 75th percentiles, respectively. The whiskers extend to the most extreme data points not considered outliers, and outliers are plotted individually using the ``+'' marker. 

In the case of noiseless data and well-specified basis, the top row shows nearly perfect decay rates of $M^{-1/2}$ for both the graph errors and the kernel errors and for both ALS and ORALS algorithms. 
For ORALS, this convergence rate agrees with Theorem \ref{thm:AN_orals}. ALS has similar convergence rates, even though it does not have a theoretical guarantee for convergence.  

\begin{figure}[h!]
\centering
\includegraphics[width=0.44\textwidth]{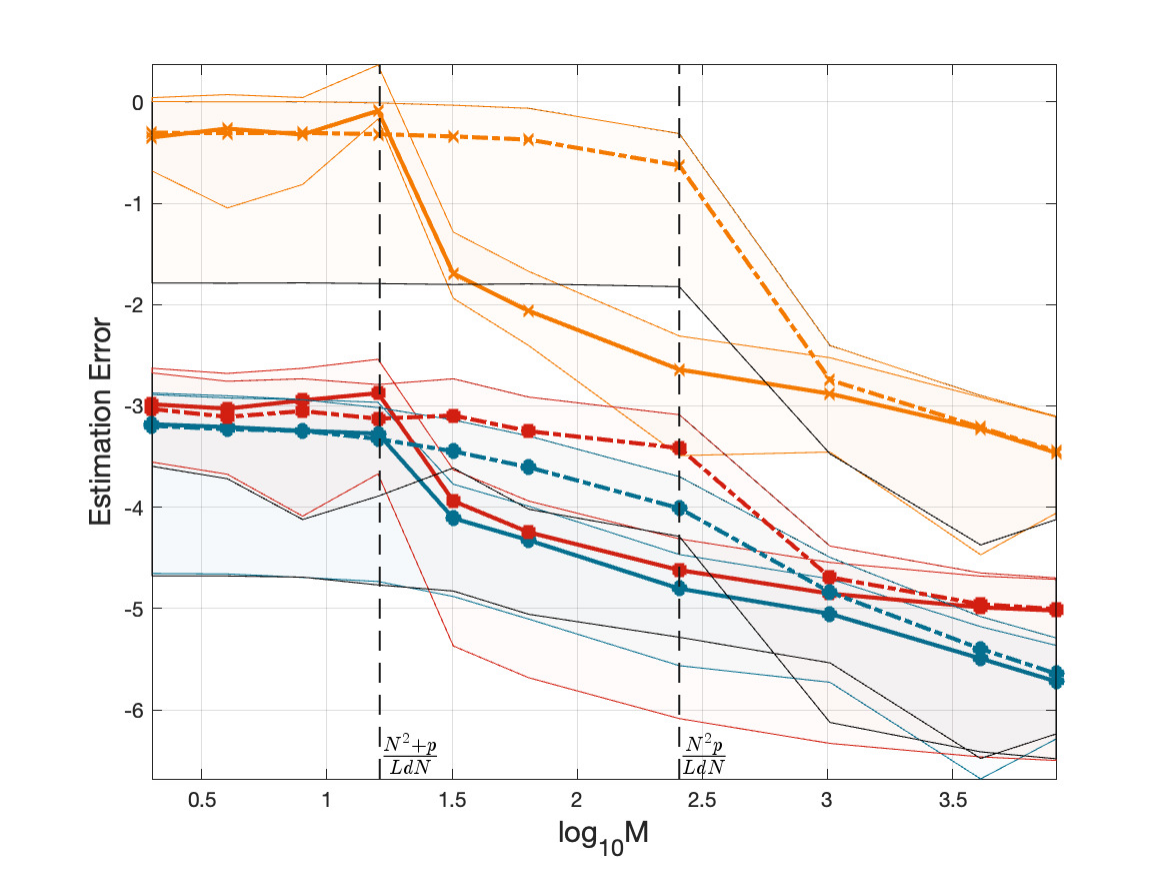}
\includegraphics[width=0.44\textwidth]{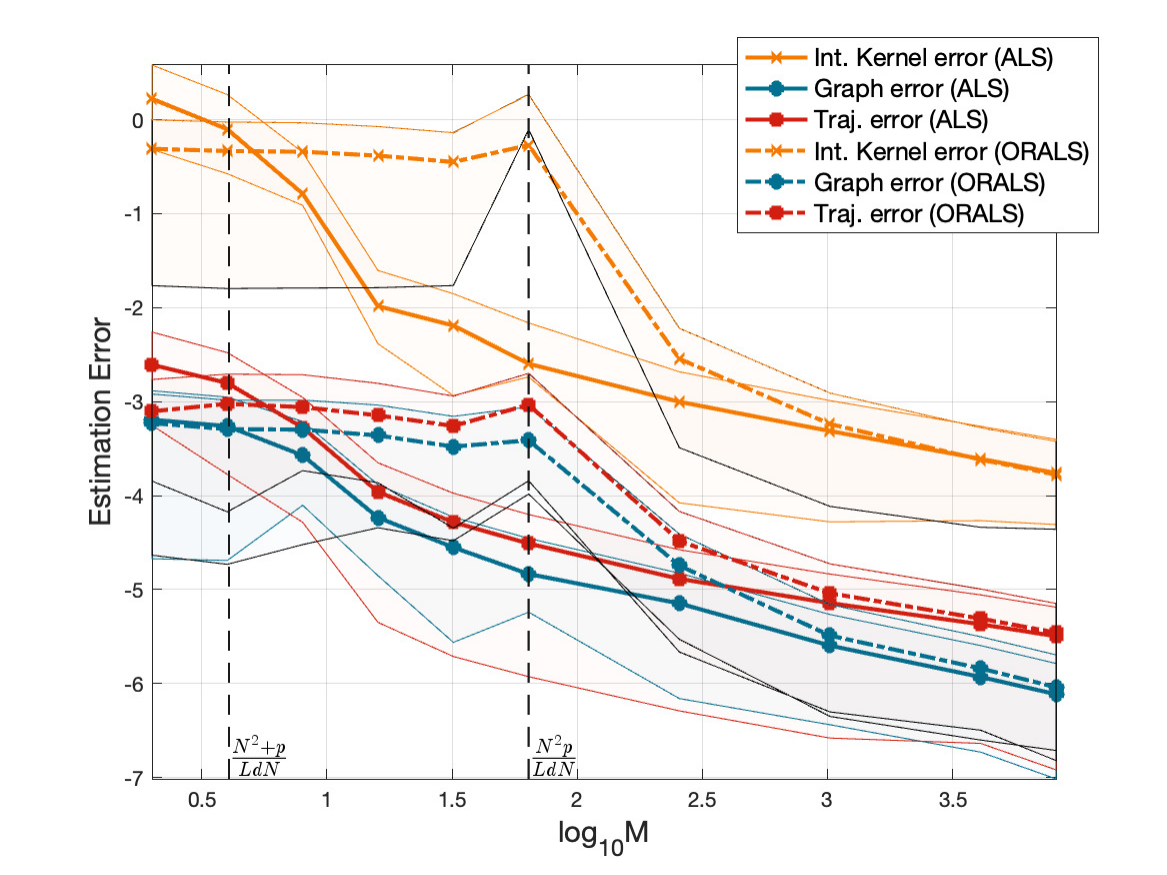}
\includegraphics[width=\textwidth]{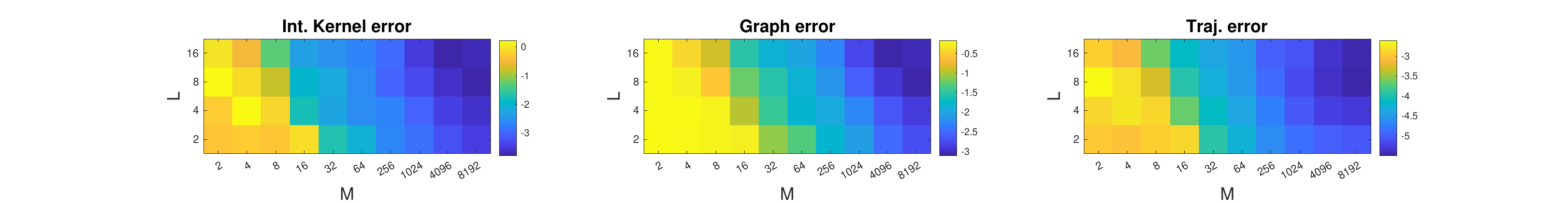}
\includegraphics[width=\textwidth]{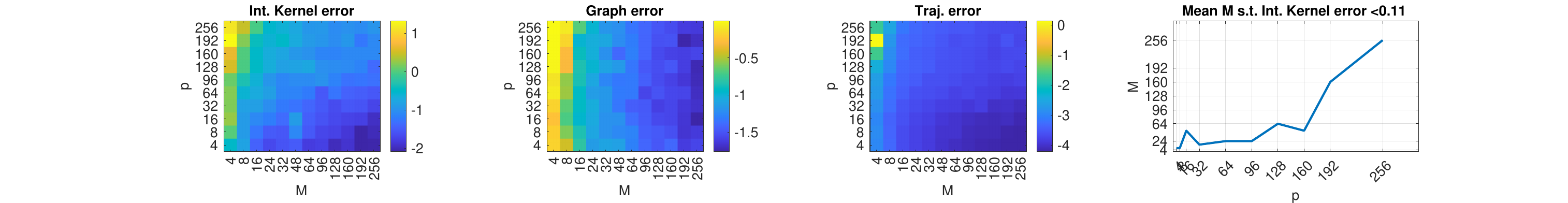}
\includegraphics[width=\textwidth]{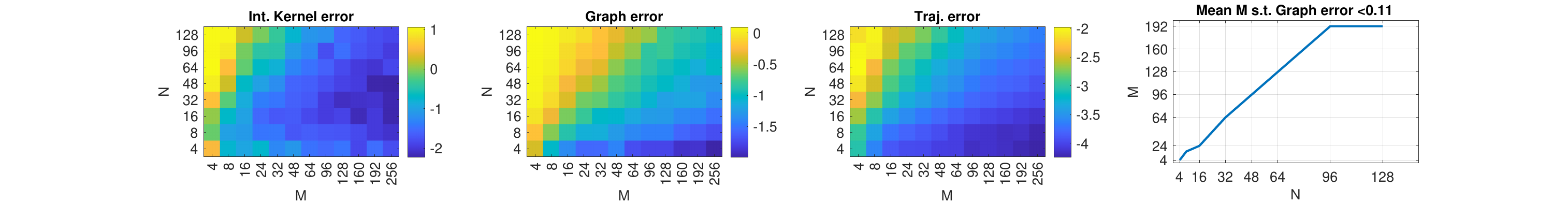}
\caption{Top: estimation errors as a function of $M$ (with all other parameters fixed), for both ALS (continuous lines) and ORALS (dashed lines), for a random Fourier interaction kernel with $\p=16$, $N=32$, $L=2$ (left) and $L=8$ (right). In the small and medium sample regime, between the two vertical bars, ALS significantly and consistently outperforms ORALS; for large sample sizes, the two estimators have similar performance. 
Second row: the performance of the ALS estimator improves not only as $M$ increases but also as $L$ increases, suggesting that multiple (dependent) observations along each trajectory can increase the accuracy of the estimator. 
Third row: performance of the ALS estimator as a function of $M$ and $p$ (with fixed $N=8$, $L=8$), and a graph showing, as a function of $p$, the mean $M$ (over 16 runs) such that the kernel estimator error is $<0.1$, suggesting a near-optimal linear scaling with $p$. Note that the graph estimation error appears quite insensitive to $p$, almost as if ALS were able to achieve a full decoupling of the problem into interaction kernel estimation and graph estimation. 
Fourth row: performance of the ALS estimator as a function of $M$ and $N$ (with fixed $p=8$, $L=8$), and a graph showing, as a function of $N$, the mean $M$ (over 16 runs) such that the graph estimation error is $<0.11$, suggesting a near-optimal linear scaling with $N$, as if the observations over $N$ agents where independent. 
}
\label{fig:conv_M_T_randomkernel}
\end{figure}

In the case of noisy data and misspecified basis, as shown in the bottom row, the decay rate remains clear for the graph errors, but plateaus at the level of observation noise $\sigma_{obs}^2 = 10^{-4}$. ALS's graph errors are about half a digit smaller than the ORALS' graph errors; while both algorithms lead to similar kernel errors when the sample size is large, the ALS' kernel errors are much smaller when the sample size is small. Thus, ALS is more robust to noise and misspecification than ORALS, and it can lead to reasonable estimators even if the sample size is small,  which we further examine next. 

\subsection{Critical sample sizes}\label{sec:sample_size}

We further examine the performance of our estimators as a function of the number of sample paths $M$ and the trajectory length $L$, 
so that the total number of observations is $ML$, each a $d$-dimensional vector. Here we consider an interaction kernel $\IK(x) =\phi(|x|)\frac{x}{|x|}$ with $\phi(r)=\sum_{k=1}^{\p} w_k/k \sin(2\pi kr)/(r+0.1)$, where $w_k \overset{\mathrm{i.i.d.}}{\sim} \mathcal{N}(0,1)$. 

 Figure \ref{fig:conv_M_T_randomkernel} shows the results for $N =32$, $\p =16$, $d=1$, $\sigma=10^{-4}$ and observation noise $10^{-4}$. The top panel shows the results with $L=2$ (left) and $L=8$ (right). The first dashed vertical bar is in correspondence of $M=(N^2+{\p})/(NLd)$
 ; the second dashed vertical bar is at $M=(N^2{\p})/(NLd)$: since we have a total of $MLdN$ scalar observations, and $N^2+{\p}$ parameters to estimate, the first one corresponds to a nearly information-theoretic optimal sampling complexity, and ALS appears to start performing well around that level of samples, albeit, because of the lack of independence in $L$, on the right the transition to convergence for ALS is not as marked as on the left; the second one appears to be consistent with the sample size at which ORALS starting to get a good performance.  
In the small and medium sample regime, between the two vertical bars, ALS significantly and consistently outperforms ORALS; for large sample sizes, the two estimators have similar performance. 
The bottom panel shows the performance of the ALS estimator as a function of both $M$ and $L$ (recall that $L= T/dt$). The performance improves not only as $ M $ increases but also as $ L $ increases, at least for this particular system. 
Furthermore, Figure \ref{fig:largeN_als} suggests that ALS is robust in the medium sample size regime, achieving a reasonable estimation when $M=256$ and $p=67$ for the above system with $N=128$ and $L=50$.

\begin{figure}[t]
\centering
	\includegraphics[width=0.99\textwidth]{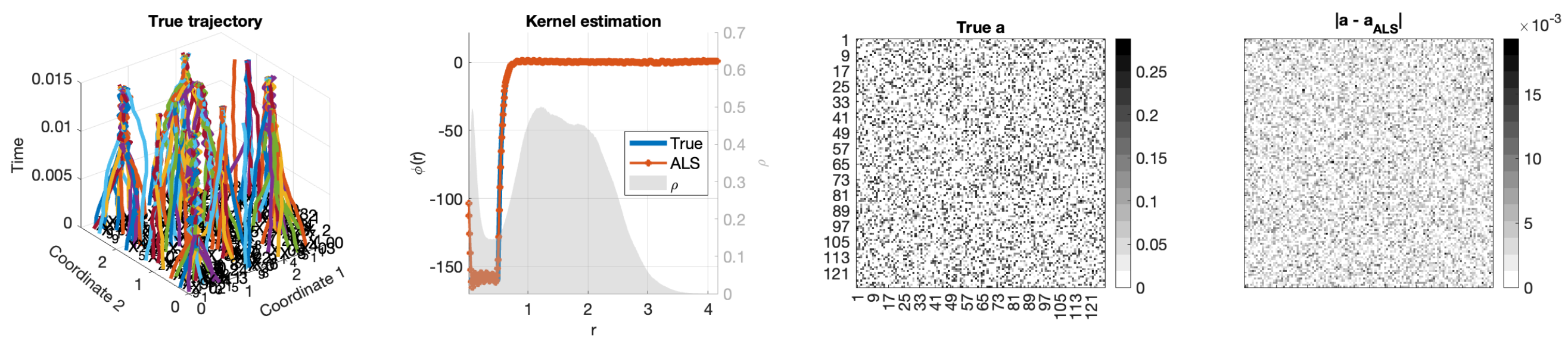}
	\caption{The system as in Figure \ref{fig:graph_true_est}, but with $N=128$ agents, $M=256$, and a dictionary of size $p=67$, consisting of the same dictionary as above, and trigonometric polynomials of degree up to $64$ added. The relative errors in graph, interaction kernel and trajectory prediction are all of order $5\cdot 10^{-2}$.}
	\label{fig:largeN_als}
\end{figure}

The main takeaways are about the critical and effective sample sizes. (i) ALS achieves good performance as soon as the number of samples is comparable to $(N^2+{\p})/(NLd)$, which acts as a critical sample size, around which the estimator quickly switched from performing poorly to performing well, while the critical sample size for ORALS appears to be of order $(N^2{\p})/(NLd)$.  (ii) In this particular system, the effective sample size appears to increase with $L$, potentially as fast as the product $ML$, notwithstanding the dependence among time steps in a single trajectory. However, this benefit arises from both the nonlinearity and the stochastic forcing; one should not expect a strictly linear gain in $L$ for arbitrary dynamical systems. For example, in a system quickly reaching a stable configuration, increasing observation time and/or $L$ may not lead to an improvement of the estimators.

\subsection{Dependence on noise level and stochastic force}

Numerical tests also show that the estimator's error decays linearly in the scale of the stochastic force and the noise level. The linear decay rate in the scale of the stochastic force agrees with Theorem \ref{thm:AN_orals}, where the variance of the error for the ORALS estimator is proportional to $\sigma^2$.  We refer to Sect.\ref{sec:decay_force_noise} for details.

\vspace{2mm}
\section{Applications}\label{sec:applications}

\subsection{Kuramoto model on a network, with misspecified hypothesis spaces}\label{sec:Kuramoto}

We consider the Kuramoto model with network
\begin{equation}\label{eq:ips_Kuramoto}
d {X}^i_t =  {\kappa}\sum_{j\in \mathcal{N}_i}a_{ij}\sin(X^j_t-X^i_t) dt + \sigma d {W}^i_t, \quad i=1,\ldots, N. 
\end{equation}
When $a_{ij}\equiv  1$ and $\sigma=0$, it reduces to the classical Kuramoto model of N coupled oscillators, where $X_t^i$ represents the phase of the $i$-th oscillator. Here $\kappa$ represents the coupling constant. The Kuramoto model was introduced to study the behavior of systems of chemical and biological oscillators \cite{kuramoto1975self} 
and has been extended to study flocking, schooling, vehicle coordination, and electric power networks (see \cite{dorfler2014synchronization,ghosh2022synchronized} and the reference therein).
In this example, our goal is to jointly estimate from multi-trajectory data the weight matrix $\ba$, and the coefficient $c$ of the (true) interaction kernel $\IK(x) = \sin(x)$ over the \emph{misspecified} hypothesis space $$\calH = \mathrm{span}\{ \cos(x), \sin(2x), \cos(2x), \dots, \cos(7x), \sin(7x)\}\,,$$
which does not contain $\IK$, and over the hypothesis space $\calH_\phi := \mathrm{span}\{ \calH ,\IK\}$.

We consider a system with $N =10$ oscillators, using the uniform distribution over the interval $[-2, 2]$ as the initial distribution, as well as a stochastic force with $\sigma = 10^{-4}$, an observation noise with $\sigma_{obs} = 10^{-3}$, time $T = 0.1$ and $\Delta t = 0.001$ (therefore, $L=100$). 
We compare the kernel estimation result using $\calH$ and $\calH_\phi$, with the number of observed trajectories $M \in \{8, 64, 512\}$. In Figure \ref{fig:kuramoto}, we present the true graph and a typical trajectory; in particular, we present the kernel estimators' mean, with one SD range represented by the shaded region, from 20 independent simulations. 
The successful joint estimation results suggest ALS and ORALS may overcome the discrepancy between the true kernel and the hypothesis space, making them applicable to nonparametric estimation.  

Due to the network structure, the system can have interesting synchronization patterns. The bottom left of Figure \ref{fig:kuramoto} shows an example of such a pattern: groups of particles moving in clusters, with each cluster having a similar angular velocity robust to the perturbation by the stochastic force. These synchronization patterns appear dictated by the network structure, and appear robust to the initial condition. In general, it is nontrivial to predict when these synchronization patterns emerge and what their features are depending on the network; for a recent study in the case of random Erd\"os-R\'enyi graphs, we refer the reader to \cite{abdalla2023expander} and references therein.

\begin{figure}[t]
\centering
\includegraphics[width=\textwidth]{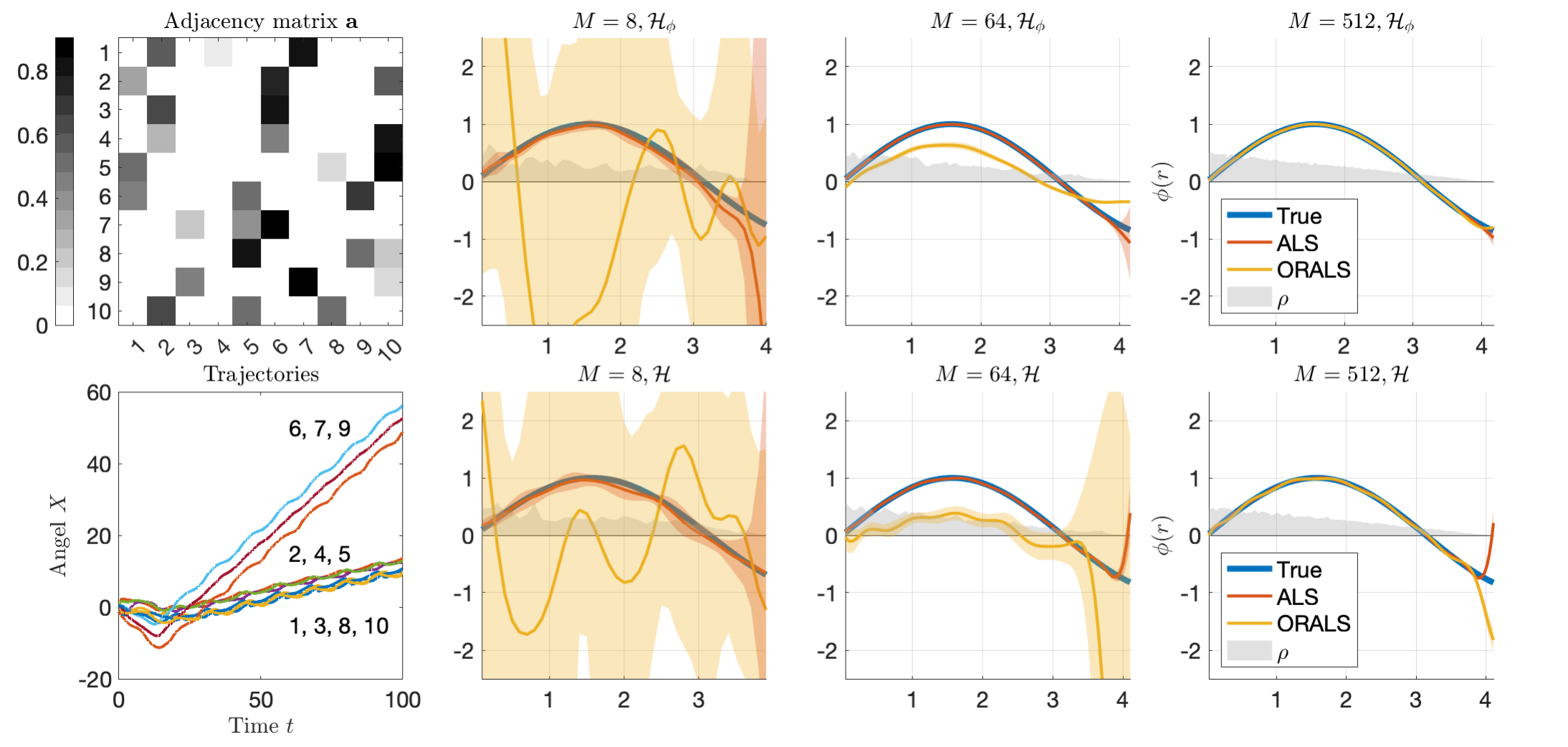}
\caption{The first column shows the true weight matrix $\ba$ and a trajectory of the system with an interesting clustering pattern. In the remaining columns, we show the estimators of the interaction function with misspecified and well-specified hypothesis spaces, i.e., $\phi\notin\calH$ (top row) and $\phi\in\calH_\phi$ (bottom row), respectively, with M ranging in $[8, 64, 512]$. 
Our estimators appear robust to basis misspecification, albeit with performance worse than in the well-specified case. 
}
\label{fig:kuramoto}
\end{figure}


\subsection{Estimating a leader-follower network}\label{sec:leader-follower}
\label{s:leaderfollowerdynamics}
Consider the problem of identifying the leaders and followers in a system of interacting agents from trajectory data. In this system, the leader agents make a stronger influence through more connections to other agents than the follower agents. Such a system can describe opinion dynamics on a social network \cite{wang2006theoretical,MT14,Ding2018discrete,hasanyan2020leader} and collective motion of pigeon flocks \cite{nagy2010hierarchical}. 
We consider the leader-follower model in the form \eqref{eq:ips_K} with 
with an interaction kernel (named influence function in the opinion dynamics literature)
\begin{align*}
	\IK(x) = -\IK_1(x)-0.1\IK_2(x)
\end{align*}
with the basis functions being $\IK_1(x)= \mathbbm{1}_{\{x\leq 1\}}$ and $\IK_2(x)= \mathbbm{1}_{\{1< x\leq 1.5\}}$. The weight matrix $\ba$ represents a leadership network, with the weights on the directed edges to be understood as a measure of impact or influence.

We identify the agents as leaders or followers by first estimating the weight matrix $\ba$ from data by the ALS algorithm and then using the $K$-means method (e.g., \cite[Chapter 9]{Bishop2006Book}) to analyze the impact feature and the influence feature extracted from the matrix. The detailed algorithm of clustering is the following:

\smallskip
\noindent\textbf{Step 1: Identify the leaders.} Given the weight matrix, observe that for any agent $A_i$, the row-wise sum $\|\ba_{i\cdot}\|_{\ell_1}=\sum_{j\neq i} |\ba_{ij}|$ represents its impact on other agents in the system, and the column-wise sum $\|\ba_{\cdot i}\|_{\ell_1}=\sum_{j\neq i} |\ba_{ji}|$ corresponds to the influence of the system on $i$. We posit that leadership can be characterized as the sum of impact on the system and influence from others:
\[	L_i = \alpha\|\ba_{i\cdot}\|_{\ell_1}+\beta\|\ba_{\cdot i}\|_{\ell_1}\,, \quad \text{with } \alpha+\beta=1\,, \alpha>\beta\,.
\]
Typically, the impact factor $\alpha$ is expected to surpass the influence factor $\beta$ when discussing leadership. 
We identify the leaders and followers by applying the $K$-means method to cluster the leadership features $\{L_i\}_{i=1}^N$. We represent leaders and followers by a partition of the index set: $[ N] = S_1\bigcup S_2 
= \{i_{1},\cdots,i_{\widetilde{N}}\}\bigcup \{j_{1},\cdots,j_{\widetilde{N}'}\}$ with $N=\widetilde{N}+\widetilde{N}'$, representing leaders and followers, respectively.

\smallskip
\noindent\textbf{Step 2: Classify the Followers.} We further classify each follower in a group according to their leader. We start by setting the $\widetilde{N}$ groups to be  $\{G^{1}=\big\{i_1\},\cdots,G^{\widetilde{N}}=\{i_{\widetilde{N}}\}\big\}$. To classify follower $j\in S_2$, we consider another leadership feature:
\begin{align*}
	\widetilde{L}_j^{k} = \alpha \sum_{i\in G^{k}}|\ba_{ij}|+\beta\sum_{i\in G^{k}} |\ba_{j i}|\,, \quad \forall ~ k=1,\cdots,{\widetilde{N}} \,.
\end{align*} 
Then we find the largest $\widetilde{L}_j^{k_0}$ and classify agent $j$ to group $k_0$ and set this group to be $\{G^{k_0},j\}$. We continue this procedure until all followers are classified.

We consider a system with $N =20$ agents with two leaders. Each of them has an impact on $80\%$ of the group members. On the other hand, the agents are labeled as followers because they affect $20\%$ of the group members.  
Also, we set a uniform distribution over the interval $[0, 4]$ as the initial, as well as a stochastic force with $\sigma = 10^{-7}$, an observation noise with $\sigma_{obs} = 10^{-7}$.  We set $\Delta t = 0.01$ with $L = 100$ steps, hence the total time $T = 1$.

Figure \ref{Fig:leader-follower} demonstrates the identified network of the agents via the above method with $(\alpha, \beta) = (0.8, 0.2)$. In this experiment, we have two leaders, labeled as $A1$ (red group) and $A6$ (blue group), out of $N=20$ agents,  and we consider three sample sizes $M\in \{15,30,100\}$. The figure shows the identification of the leader-follower network depends on sample size: we can identify the leader-follower network accurately when the sample size is large, e.g., $M=100$. The error of graph estimation is $0.0018$. But when the sample is too small, e.g., $M=15$ and $M=30$, the inference can have large errors: the errors of graph estimation are $0.1254$ when $M=15$ and $0.0094$ when $M=30$. Nevertheless, the leaders and followers are correctly identified; see more detailed results in Appendix \ref{Append:addition}.  
This example suggests that we can consistently identify and cluster leaders and followers from a small sample size.

\begin{figure}[h!]
\centering
\includegraphics[width=0.24\textwidth]{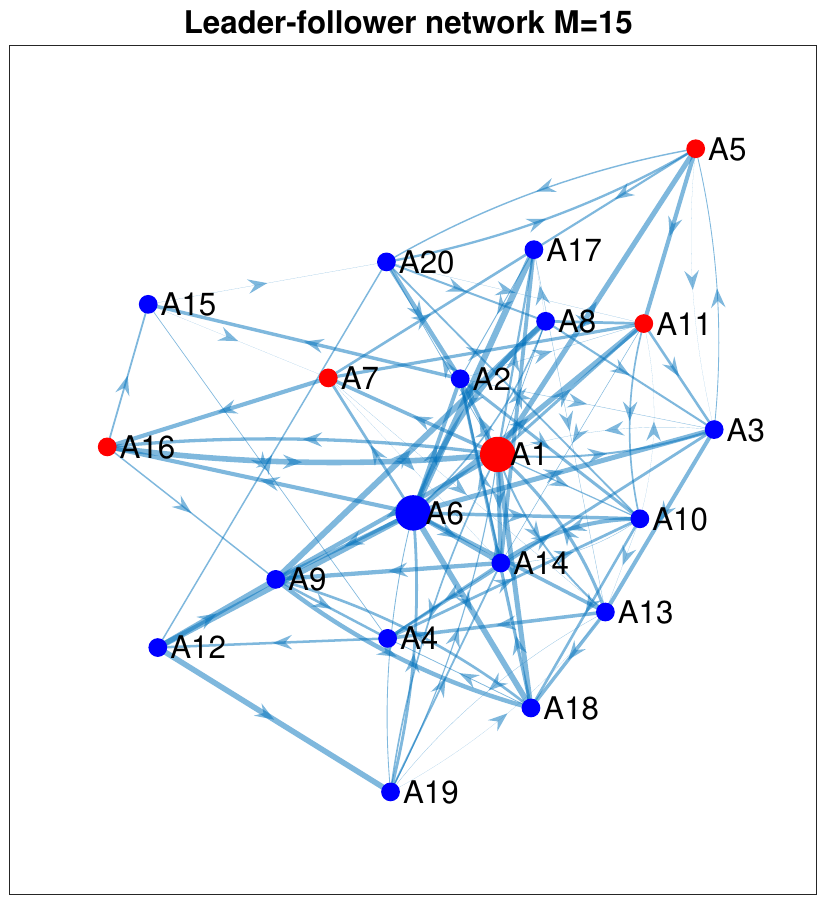}
\includegraphics[width=0.24\textwidth]{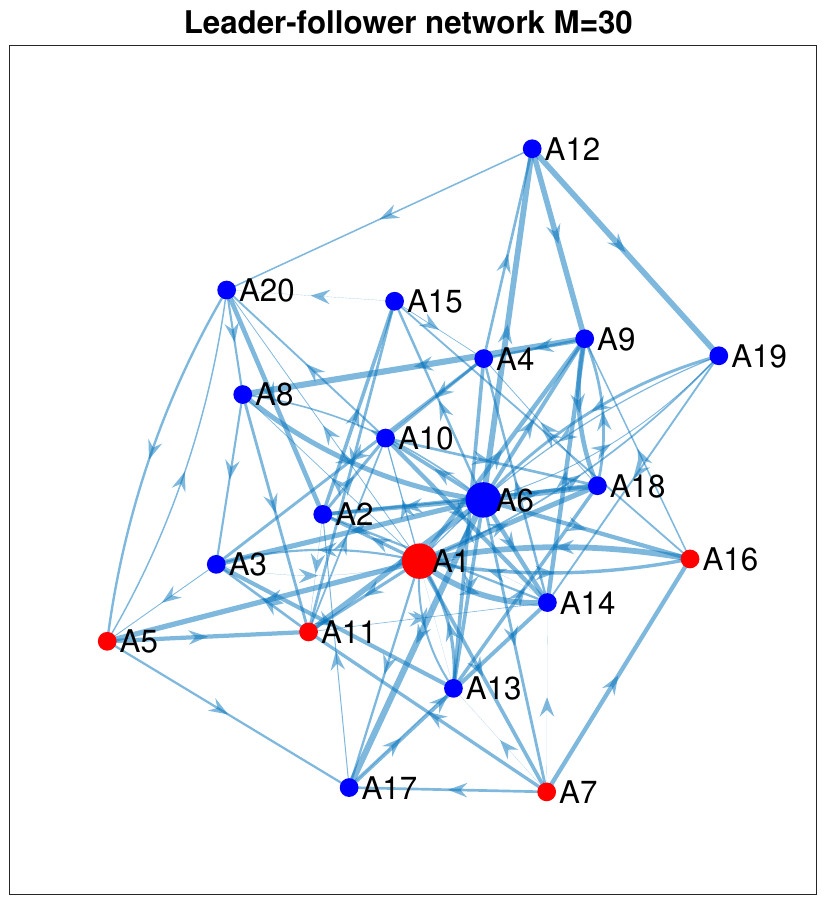}
\includegraphics[width=0.24\textwidth]{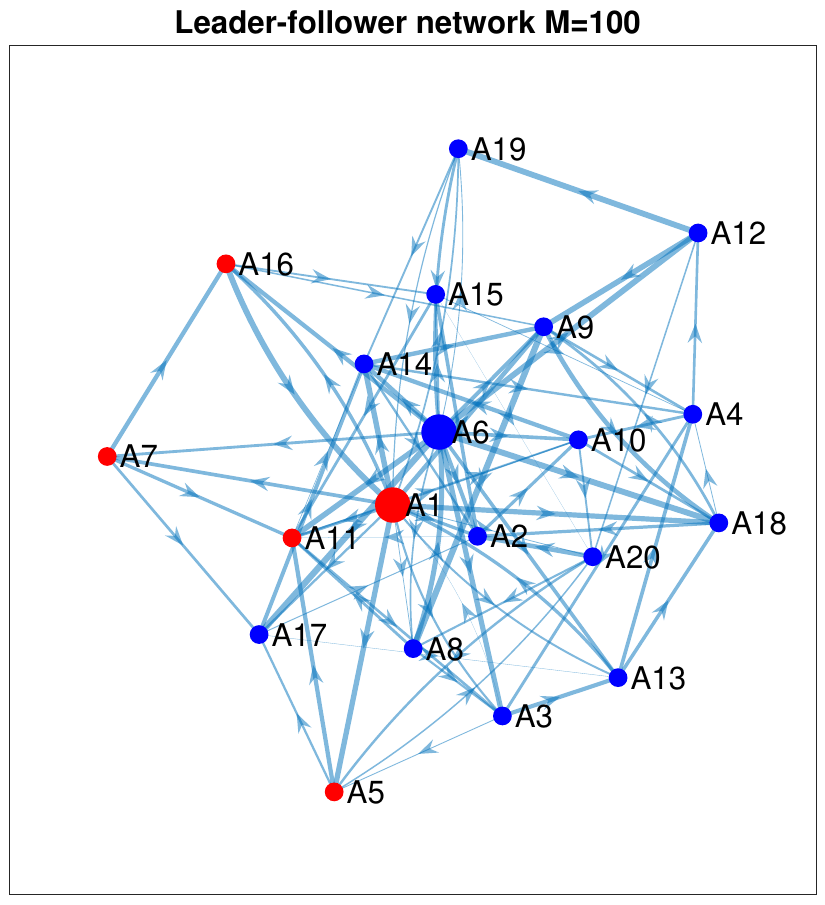}
\includegraphics[width=0.24\textwidth]{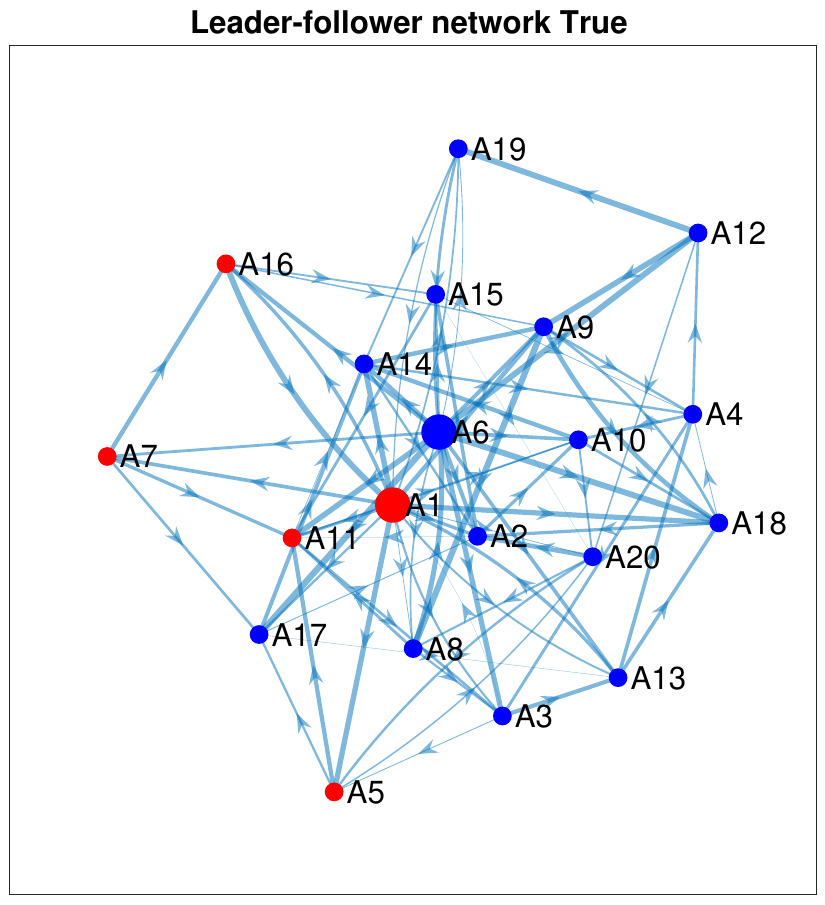}
\caption{Estimated networks of leaders and followers from datasets with sample sizes $M\in \{15,30,100\}$ and the ground truth. When $M=100$, the estimated network is accurate. When $M=30$, the leaders-follower network is correctly identified, though the weight matrix is less accurate. When $M=15$, the sample size is too small for a meaningful inference; but the clustering is still reliable. 
}
\label{Fig:leader-follower}
\end{figure}

\subsection{Multitype interaction kernels}\label{sec:multi_type}
We consider further the joint inference of a generalized model with multiple types of agents distinguished by their interaction kernels. 
Specifically, consider the system  \eqref{eq:ips_K_type} with $Q$ types of interaction kernels, where $\kappa(i)$ denotes the type of kernel for the agent $i\in[N]$, and $\IK_q$ is the interaction kernel for agents of type $q$. 
Given a hypothesis space $\mH = \text{span}\{\psi_k\}_{k = 1}^p $ that includes these kernels,  there a exists coefficients matrix $\bc \in \R^{p \times N}$ such that 
\begin{equation*}
	 \IK_{\kappa(i)}(x) = \sum_{k = 1}^p \bc_{ki}\psi_k(x)
\end{equation*}
with $\bc_{\cdot i} = \bc_{\cdot \kappa(i)}$, i.e. the matrix $\bc$ has $Q$ distinct columns corresponding to the types. Using the same tensor notation as before, we have 
\begin{equation}\label{eq:ipog_multitype} 
 \begin{aligned}
\dynsys {\mathbf{a}}{\bc}\quad:\quad &\dot \bX_t  = \ba \bB(\bX_t) \bc+\sigma \dot{\bW} = \big(  \ba_{i\cdot}\bB(\bX_t)_i \bc_{\cdot i}\big)_{i\in[N]}  +\sigma \dot{\bW} \,,\quad  \text{ where}  \\
 &\ba_{i\cdot}\bB(\bX_t)_i \bc_{\cdot i}  = \sum_{j\neq i} \ba_{ij}\sum_{k=1}^{\p} \psi_k(X^j_t-X^i_t)c_{ki} \in \R^{d}\,, i\in[N]. & 
\end{aligned}
\end{equation}

Our goal is to jointly estimate the weight matrix $\ba$ and the matrix $\bc$, which represents the $Q$ kernels without knowing the type function $\kappa$, from data consisting of multiple trajectories.

Since $\bc$ has $Q$ distinct columns, we have $\text{rank}(\bc) \leq Q$, which is a weaker condition. However, the low-rank property of $\bc$ is sufficient for us to apply ALS. Using SVD, we decompose $\bc$ as
\begin{equation*}
	 \bc = \bu \bv^\top
\end{equation*}
where $\bu \in \R^{p\times Q}$ is called the {\em coefficient matrix}. This is because $\bu$ represents the orthogonalized coefficients of the $Q$ interaction kernels on the basis $\{\psi_k\}$.  And the {\em type matrix} ${\bv} \in \mathbb{R}^{N \times Q}$ is assumed to be orthonormal, i.e., $\bv^\top \bv  = I_Q$, as it represents the type of the $i$-th particle with each row of $\bv$ represents the weight of the orthogonalized $Q$ interaction kernels that the kernel $\IK_{\kappa(i)}$ has. Such normalization condition avoids the simple non-identifiability issue, as demonstrated in the admissible set of $\ba$. We write the above system as 
\begin{equation*}
 \begin{aligned}
\dynsys {\mathbf{a}}{\bu, \bv}\quad:\quad &\dot \bX_t  = \ba \bB(\bX_t)\bu\bv^\top+\sigma \dot{\bW} = \big(  \ba_{i\cdot}\bB(\bX_t)_i \bu\bv_{i\cdot}^\top\big)_{i\in[N]}  +\sigma \dot{\bW} \,,\quad  \text{ where}  \\
 &  \ba_{i\cdot}\bB(\bX_t)_i \bu\bv_{i\cdot}^\top  = \sum_{j\neq i} \ba_{ij}\sum_{k=1}^{\p} \psi_k(X^j_t-X^i_t)\sum_{q = 1}^Q \bu_{kq}\bv_{iq} \in \R^{d}\,, i\in[N].  
\end{aligned}
\end{equation*}
With data of multiple trajectories $\{\bX_{t_0:t_L}^m\}_{m = 1}^M$, the loss function is defined as 
\begin{equation}\label{eq:est_joint_multitype_AUV}
 \begin{aligned}
(\widehat{\ba},  \widehat{\bu} ,  \widehat{\bv} ) 	&= \argmin{\substack{(\ba,\bu, \bv)\in \mathcal{M}\times \R^{p\times Q}\times\R^{N \times Q} \\ \bv^\top\bv = I_Q}}  \calE_{L, M}(\ba,\bu, \bv), \,\,\,\text{ with} \\
\calE_{L, M}(\ba,\bu, \bv) &:= \frac{1}{MT}\sum_{l=1, m = 1}^{L,M}\norm{\Delta\bX_{t_l}^m -\ba \bB (\bX_{t_l}^m) \bu\bv^\top\Delta t}_F^2. 
 \end{aligned}
\end{equation}

We introduce a \emph{three-fold ALS} algorithm to solve the above optimization problem.  Notice that the loss function  \eqref{eq:est_joint_multitype_AUV} is quadratic in each of the unknowns $\ba, \bu, \bv$ if we fix the other two. The \emph{three-fold ALS} algorithm alternatively solves for each of the unknowns while fixing the other two. In each iteration, this algorithm proceeds as follows: solving $\ba$ via least squares with nonnegative constraints, next solving $\bu$ by least square, and then solving $\bv$ via least squares followed by an ortho-normalization step, which is an orthogonal Procrustes problem \cite{gower2004procrustes}.  Additionally, we add an optional $K$-means step to ensure that  $\bc$ has only $Q$ distinct columns. The details of the algorithm are postponed to Section \ref{sec:3fold_ALS}.

Figure \ref{fig:multitype_kernel} numerically compares the three-fold ALS with and without the $K$-means step. We consider the example of $N = 10$ agents and $Q=2$ types of kernels corresponding to short-range and long-range interactions, which are constructed through $p = 16$ linear spline basis functions on $[0, 5]$. The correspondence between agents and the type of kernels is uniformly randomly generated. We use the data of $M = 400$ independent trajectories, with a uniform distribution over the interval $[0, 5]$ as initial distribution, $\Delta t = 10^{-3}, L = 50$ so that $T = 0.05$, and the stochastic force and the observation noise have $\sigma = \sigma_{obs} = 10^{-3}$. The weight matrix is randomly generated with entries sampled from the uniform distribution on $[0,1]$, followed by a row-normalization. The true kernels are constructed on spline basis functions, representing short-range interaction (Type 1) and long-range interaction (Type 2).

Figure \ref{fig:multitype_kernel} reports the error decay in the iteration number and the comparison between the estimated and true kernels. It shows that the algorithm using $K$-means at each step performs better than the one without the $K$-means since it preserves more model information.

\begin{figure}[htb]
\centering
\includegraphics[width=\textwidth]{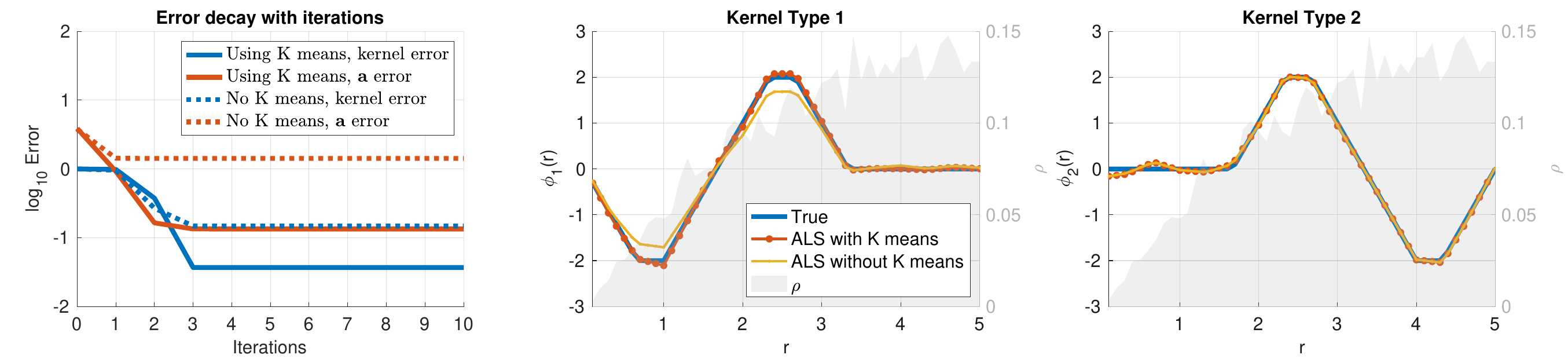}
\caption{Estimation of two types of kernels: short range and long range. The first panel shows the error decay with respect to iteration numbers. The algorithm using $K$-means decays faster and reaches lower errors than the algorithm without $K$-means. The right two columns show the estimation result of the two kernels. The classification is correct for both of the algorithms, and the one with $K$-means yields more accurate estimators, particularly for the kernel Type 1. }
\label{fig:multitype_kernel}
\end{figure}

\paragraph{Model selection} We further test the robustness of the three-fold ALS algorithm for model selection when the number $Q\in \{1,2\}$ is unknown. We apply the algorithm with both $Q\in \{1,2\}$ on two datasets that are generated with $Q_{true}=1$ and $Q_{true}=2$ respectively. Table \ref{tab:misspecified_Q} shows that the three-fold ALS can select the correct model through trajectory prediction errors. It reports the means and SDs of trajectory prediction using  $10$ test trajectories,  $\Delta t = 10^{-2}$ and $L = 500$ time steps. Note that the total time length is $T = 5$. When $Q_{true}=1$, the error of the estimators with misspecified $Q=2$ is relatively accurate, because the estimated two types of kernels are both close to the true kernel, as examined in Figure  \ref{fig:multitype_kernel_single_kernel_traj}. Thus, the algorithm effectively identifies the correct model. 

\begin{table}[h!]
\centering
\begin{tabular}{ |p{5cm}||p{5cm}| |p{5cm}| }
 \hline
  							& $Q_{true}=1$  & $Q_{true}=2$ \\
 \hline
 	Estimated with $Q=1$  &  \boldmath{$1.22\times 10^{-2} \pm 8.23 \times 10^{-3}$} &  $2.06 \times 10^{-1} \pm 6.88 \times 10^{-2}$\\
 \hline
 	Estimated with  $Q=2$   &  $1.44\times 10^{-2} \pm 7.40 \times 10^{-3}$ & \boldmath{$1.12\times 10^{-2} \pm 2.80 \times 10^{-3}$}\\
 \hline
\end{tabular}
\caption{Model selection: single- v.s. two- types of kernels. The table shows the Mean and SD of trajectory prediction errors in 10 independent numerical experiments, where the number of kernel types is unknown. Smaller errors indicate a correct model. The model is correctly identified in both cases (highlighted in bold). }
\label{tab:misspecified_Q}
\end{table}

\begin{figure}[htb]
\centering
\includegraphics[width=0.7\textwidth]{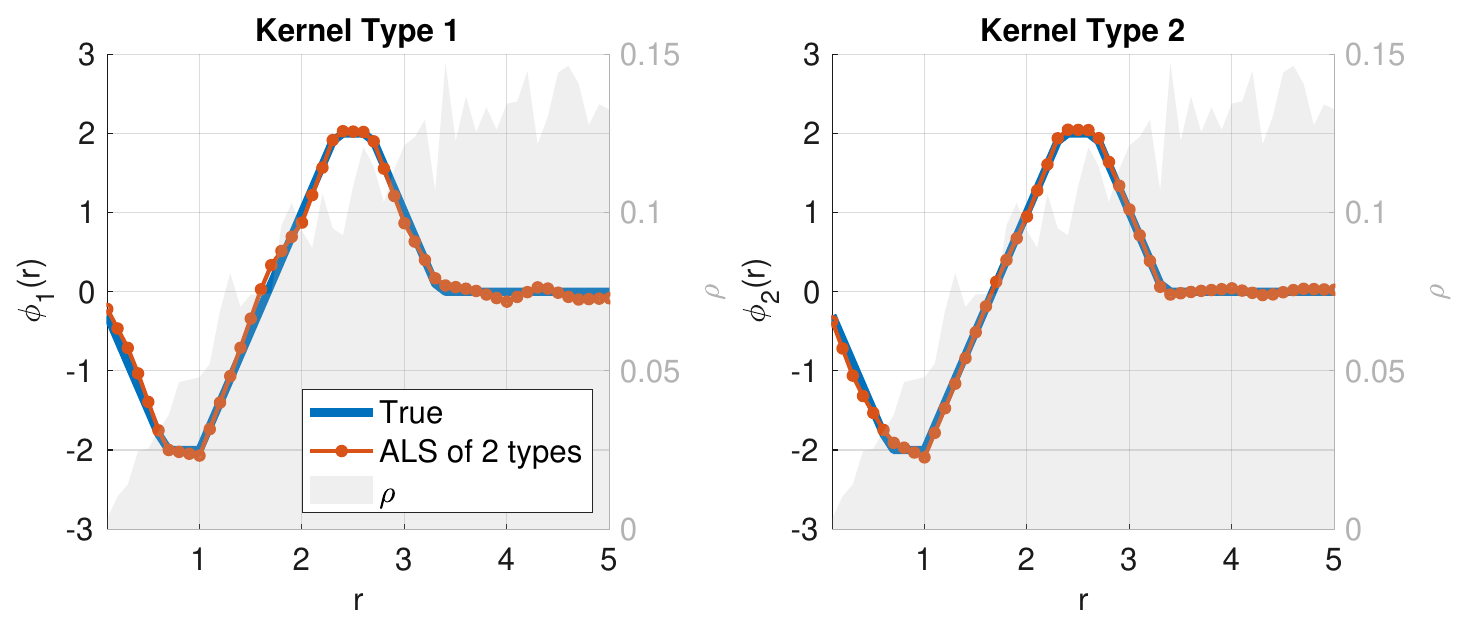}
\caption{Estimated kernels in a misspecified case: estimating two types of kernel when data is generated using a single kernel. The algorithm outputs two types of kernels, but both are close to the true kernel.}
\label{fig:multitype_kernel_single_kernel_traj}
\end{figure}

\section{Conclusion}
We have proposed a robust estimator for joint inference of networks and interaction kernels in interacting particle systems on networks, implemented with computationally two scalable algorithms: ALS and ORALS. We have tested the algorithms on several classes of systems, including deterministic and stochastic systems with various types of networks and with single and multi-type kernels. We have also examined the non-asymptotic and asymptotic performance of the algorithms: the ALS is robust for small sample sizes and misspecified hypothesis spaces, and both algorithms yield convergent estimators in the large sample limit. 
  
Our joint inference problem leads to a non-convex optimization problem that resembles those in compressed sensing and matrix sensing. However, diverging from the conventional framework of matrix sensing, our data are correlated, our joint estimation is in a constrained parameter space and a function space, and the Restricted Isometry Property (RIP) condition rarely holds with a small RIP constant. These differences can lead to an optimization landscape with multiple local minima.  

We introduce coercivity conditions that guarantee the identifiability and the well-posedness of the inverse problem. These conditions also ensure that the ALS and ORALS algorithms have well-conditioned regression matrices and asymptotic normality for the ORALS estimator. Also, we have established connections between the coercivity and RIP conditions, providing insights into further understanding of the joint estimation problem.  
  
Interacting particle systems on networks offer a wide array of versatile models applicable across multiple disciplines. We considered here estimating the Kuramoto model on a network, classifying agent roles within leader-follower dynamics, and learning systems with multiple types of interaction kernels. Our algorithms are adaptable to various scenarios and applications,r and amenable to be extended to more general settings, including models with more general interaction kernels. 

We expect further applications of the algorithms for the construction of effective reduced heterogeneous models for large multi-scale systems. 
Also, other future directions include generalizations to nonparametric joint estimations, further understanding of the convergence and stability of the ALS algorithm, regularizations enforcing the low-rank structures, and learning from partial observations.  

\appendix
\section{Theoretical analysis}

\subsection{Coercivity conditions: connections and examples}\label{Append:A}

First, we prove Proposition \ref{Prop:r2JCC_Ident}, which states that the rank-2 joint coercivity implies identifiability. 

\begin{proof}[Proof of Proposition \ref{Prop:r2JCC_Ident}]
	Notice that
	\begin{align*}
		\calE_{L, \infty}^{(i)}(\ba,\IK): = & \E \bigg[\bigg|  \sum_{j\neq i} [\ba_{ij}^{*} \IK_*(\mathbf{r}_{ij}(t_l))-\ba_{ij} \IK(\mathbf{r}_{ij}(t_l))]\bigg|^2 \bigg] \\
		= & \E \bigg[\bigg|  \sum_{j\neq i} \Big[\ba_{ij}^{*}-p_{\IK}\ba_{ij} \Big] \IK_*(\mathbf{r}_{ij}(t_l))+\ba_{ij} \Big[p_{\phi}\IK_*(\mathbf{r}_{ij}(t_l)) -\IK(\mathbf{r}_{ij}(t_l))\Big]\bigg|^2 \bigg], 
	\end{align*}
	where 
$		p_{\IK} = \frac{\innerp{\IK,\IK_*}_{\rho_L}}{\|\IK_*\|_{\rho_L}^2}$ and $p_{\IK}\IK_*-\IK \perp \IK_* \text{ in } \calH
$.
	Then, the rank-2 joint coercivity \eqref{eq:jointCC2} implies
	\begin{align*}
		\calE_{L, \infty}^{(i)}(\ba,\phi)&\geq c_{\calH}|\ba_{i\cdot}^{*}-p_{\phi} \ba_{i\cdot}|^2 \|\phi_*\|_{\rho_L}^2 +c_{\calH}|\ba_{i\cdot}|^2\| p_{\phi}\IK_* -\IK \|_{\rho_L}^2\,. 
	\end{align*}
	Hence $\calE_{L, \infty}(\ba,\IK)=\sum_i \calE_{L, \infty}^{(i)}(\ba,\IK) = 0$ and $c_\calH>0$ imply that
$
		|\ba_{i\cdot}^{*}-p_{\IK} \ba_{i\cdot}|^2=0$ and $\| p_{\IK}\IK_* -\IK \|_{\rho_L}^2=0\,, \forall i\in [N]
$,
	since $\IK_*\neq 0$ and $0\neq \ba\in \mathcal{M}$. Because $\ba^{(*)},\ba\in \mathcal{M}$, the only choice for $|\ba_{i\cdot}^{*}-p_{\IK} \ba_{i\cdot}|^2=0$ is both $p_{\IK}=1$ and $\ba^{*}=\ba$. Consequently, $\| p_{\IK}\IK_* -\IK \|_{\IK_L}^2=\| \IK_* -\IK \|_{\rho_L}^2=0$ yields $\IK_*=\IK$ in $L^2_\rho$. 
\end{proof}

The interaction kernel coercivity is stronger than the joint coercivity:
\begin{proposition}[Interaction kernel coercivity implies joint coercivity]\label{prop:Kernelcc-Jcc}
Assume that for all $i\in[N]$, $\{\br_{ij}(t)=X_t^j-X_t^i\}_{j=1,j\neq i}^N$ are pairwise independent conditional on $\calF^i_t$. Then, the kernel coercivity \eqref{eq:kernelCC} with $c_{0,\calH}$ implies that the joint coercivity conditions \eqref{eq:jointCC} and \eqref{eq:jointCC2} hold with  $c_{1,\mathcal H}= C_{\ba,N}^{(1)}c_{0,\calH}$ and $c_{2,\mathcal H}= C_{\ba,N}^{(2)}c_{0,\calH}$, respectively, where $C_{\ba,N}^{(1)}=\frac 1N\sum_{i=1}^N \sum_{j\neq i} \ba_{ij}^2$ and $C_{\ba,N}^{(1)}=\frac 1N\sum_{i=1}^N \sum_{j\neq i} [|\ba_{ij}^{(1)}|^2+|\ba_{ij}^{(2)}|^2]$. 
\end{proposition}

\begin{proof} Without loss of generality, we consider only the case when $L=1$. 
By assumption, the random variables $\br_{ij}$ and $\br_{ij'}$ are independent, conditioned on $\calF^i$, if $j\neq j'$ and $j,j' \neq i$. Then, by Lemma \ref{Lem:Var_lem} with $f_j(\cdot)=\ba_{ij}\IK(\cdot)$ for each fixed $i$, we get
	\begin{align*}
		\frac 1N \sum_{i=1}^N  \E \bigg[\bigg|\sum_{j\neq i} \ba_{ij}\IK(\mathbf{r}_{ij})\bigg|^2 \bigg]
		& = \frac 1N \sum_{i=1}^N \E \bigg[ \E\bigg( \bigg| \sum_{j\neq i} \ba_{ij}\IK(\mathbf{r}_{ij})\bigg|^2  \mid \calF^i\bigg) \bigg]
		\geq \frac 1N \sum_{i=1}^N \sum_{j\neq i} \ba_{ij}^2 \E [\trCov(\IK(\mathbf{r}_{ij}) \mid \calF^i) ] \\
		&\geq \frac 1N \sum_{i=1}^N  \sum_{j\neq i} \ba_{ij}^2 c_{0,\calH}\|\IK\|_{\rho_L}^2= C_{\ba,N}^{(1)}c_{0,\calH} \|\IK\|_{\rho_L}^2\,,
	\end{align*}
	where the last inequality follows from \eqref{eq:kernelCC}. 
Therefore, by combining the above inequalities, we obtain that \eqref{eq:jointCC} holds with the constant $c_{1,\mathcal H}= C_{\ba,N}^{(1)} c_{0,\calH}$. This proves the rank-1 joint coercivity condition.

To prove the rank-2 joint coercivity condition \eqref{eq:jointCC2}, it suffices to show that for all $i\in[N]$, 
\begin{align}\label{Ineq:rank-2_cc_suff}
	\E \bigg[\bigg|\sum_{j\neq i} [\ba_{ij}^{(1)}\IK_1(\mathbf{r}_{ij})+\ba_{ij}^{(2)}\IK_2(\mathbf{r}_{ij})]\bigg|^2 \bigg] \geq c_{2,\calH} \left[|\ba_{i\cdot}^{(1)}|^2  \|\IK_1\|_{\rho_L}^2+ |\ba_{i\cdot}^{(2)}|^2  \|\IK_2\|_{\rho_L}^2\right]\,
\end{align}
for any vectors $\ba_{i\cdot}^{(1)},\ba_{i\cdot}^{(2)}\in \mathcal{M}$ and any two functions $\IK_1,\IK_2\in\calH$ with $\langle\IK_1,\IK_2\rangle=0$. As in the rank-1 case, we have by Lemma \ref{Lem:Var_lem} that
\begin{align*}
	\E \bigg[\bigg|\sum_{j\neq i} & [\ba_{ij}^{(1)}\IK_1(\mathbf{r}_{ij})+\ba_{ij}^{(2)}\IK_2(\mathbf{r}_{ij})]\bigg|^2 \bigg] \\
	&= \sum_{j\neq i} |\ba_{ij}^{(1)}|^2 \E [\trCov(\IK_1(\mathbf{r}_{ij}) \mid \calF^i) ] 
	+\sum_{j\neq i} |\ba_{ij}^{(2)}|^2 \E [\trCov(\IK_2(\mathbf{r}_{ij}) \mid \calF^i) ] \\
	&+\E \left[\bigg| \sum_{j\neq i} \ba_{ij}^{(1)}\E\big[ \IK_1(\mathbf{r}_{ij}) \mid \calF^i \big] \bigg|^2 \right]
	+\E \left[\bigg| \sum_{j\neq i} \ba_{ij}^{(2)}\E\big[ \IK_2(\mathbf{r}_{ij}) \mid \calF^i \big] \bigg|^2 \right] \\
	&+2\E\left[ \sum_{j\neq i}\ba_{ij}^{(1)}\E\Big[\IK_1(\mathbf{r}_{ij}) \mid \calF^i \Big] \cdot \sum_{j\neq i}\ba_{ij}^{(2)}\E\Big[\IK_2(\mathbf{r}_{ij}) \mid \calF^i \Big] \right]\\
	&\geq \sum_{j\neq i} |\ba_{ij}^{(1)}|^2 \E [\trCov(\IK_1(\mathbf{r}_{ij}) \mid \calF^i) ] 
	+\sum_{j\neq i} |\ba_{ij}^{(2)}|^2 \E [\trCov(\IK_2(\mathbf{r}_{ij}) \mid \calF^i) ]\,.
\end{align*}
This confirms \eqref{Ineq:rank-2_cc_suff} with $c_{2,\calH}=C_{\ba,N}^{(2)}c_{0,\calH}$, where $C_{\ba,N}^{(2)}=\frac 1N\sum_{i=1}^N \sum_{j\neq i}$ $ [|\ba_{ij}^{(1)}|^2+|\ba_{ij}^{(2)}|^2]$.
\end{proof}

\begin{remark}[Sufficient but not necessary for identifiability]\label{Rmk:discuss_ker_coerc}
Together with Proposition \ref{Prop:r2JCC_Ident}, we conclude that interaction kernel coercivity implies identifiability. We shall also see that it is a sufficient condition to ensure that the operator regression stage is well-posed. We do not expect it however to be necessary for the identifiability of the weight matrix and the kernel.  

Heuristically, the proof for Proposition \ref{Ineq:rank-2_cc_suff} suggests that the kernel coercivity condition \eqref{eq:kernelCC} is not only a sufficient condition for rank-1 and rank-2 joint coercivity but also for `higher rank' joint coercivity conditions, resembling a `full rank' version of the joint coercivity condition.
\end{remark}

\begin{lemma}\label{Lem:Var_lem}
	Suppose $\{X^i\}_{i=1}^N$ are $\R^d$-valued random variables such that for each $i$, conditional on an $\sigma$-algebra $\calF^i$, the random variables $\{\mathbf{r}_{ij}=X^j-X^i\}_{j=1,j\neq i}^N$ are independent. Then, for any square-integrable functions $\{f_j:\R^d\to\R^d\}_{j=1}^N$, we have 
	\begin{equation}\label{Lem:Var_ineq}
		\E \bigg[\bigg|\sum_{j\neq i} f_j(\mathbf{r}_{ij})\bigg|^2  \mid \calF^i \bigg]\geq \sum_{j\neq i} \trCov\left(f_j(\mathbf{r}_{ij})\mid \calF^i \right)\,,\quad \forall~i\in[N]\,.
	\end{equation}
\end{lemma}
\begin{proof} 
It suffices to consider the case $i=1$ as the proofs for different $i$'s are the same. That is, we aim to prove
$
		\E \bigg[\bigg|\sum_{j=2}^N f_j(\mathbf{r}_{1j})\bigg|^2  \mid \calF^1 \bigg]\geq \sum_{j=2}^N \trCov\left(f_j(\mathbf{r}_{1j})\mid \calF^1 \right)
$.
By the conditional independence assumption, we have 
$
\E[\innerp{f_j(\mathbf{r}_{1j}), f_{j'}(\mathbf{r}_{1j'}) }_{\R^d} \mid   \calF^{1}] =\innerp{\E[f_j(\mathbf{r}_{1j})  \mid   \calF^{1} ], \E[ f_{j'}(\mathbf{r}_{1j'}) \mid   \calF^{1}] }_{\R^d}
$. 
Using this fact for the second equation below, we have
	\begin{align}
		 & \E \bigg[\bigg|\sum_{j=2}^N f_j(\mathbf{r}_{1j})\bigg|^2  \mid \calF^1\bigg] 
		=\E\bigg[ \sum_{j=2}^N  \Big| f_j(\mathbf{r}_{1j}) \Big|^2 +\sum_{\substack{j,j'=2 \\ j\neq j'}}^N \innerp{ f_j(\mathbf{r}_{1j}), f_{j'}(\mathbf{r}_{1j'})}_{\R^d} \mid  \calF^{1} \bigg] \notag\\
		=&\sum_{j=2}^N \E\Big[ \big| f_j(\mathbf{r}_{1j}) \big|^2 \mid \calF^{1} \Big] + \sum_{\substack{j,j'=2 \\ j\neq j'}}^N \Big\langle\E\big[ f_j(\mathbf{r}_{1j})\mid \calF^{1}\big],  \E\big[ f_{j'}(\mathbf{r}_{1j'}) \mid \calF^{1} \big] \Big\rangle_{\R^d} \notag \\
		=&\sum_{j=2}^N \bigg\{ \E\Big[ \big| f_j(\mathbf{r}_{1j}) \big|^2 \mid \calF^{1}\Big]-\bigg | \E\Big[f_j(\mathbf{r}_{1j}) \mid \calF^{1} \Big]\bigg| ^2\bigg\}   +\bigg| \sum_{j=2}^N \E\Big[ f_j(\mathbf{r}_{1j}) \mid \calF^{1} \Big] \bigg|^2 \,. \label{eq:kernelCC-drop}
	\end{align}
Then, we obtain  \eqref{Lem:Var_ineq} with $i=1$ by noticing the fact that  
$\trCov(f_j(\mathbf{r}_{1j}) \mid X^1)=\Big[\E[ | f_j(\mathbf{r}_{1j}) |^2 \mid  \calF^{1}]-\big | \E[ f_j(\mathbf{r}_{1j}) \mid \calF^{1} ]\big|^2\Big]$. 
\end{proof}

We now show that the interaction kernel coercivity condition holds in $\calH=L^2_\rho$ for radial kernels when $L=1$ and the initial distribution is standard Gaussian.  
\begin{proposition}\label{prop:kernelCC_example}
	Let $L=1$, $\IK(x)= \phi(|x|)\frac{x}{|x|}$, and the components of $(X^1_{t_1},\ldots,X^N_{t_1})$ be i.i.d. standard Gaussian random vectors in $\R^d$.
	The interaction kernel coercivity condition in \eqref{eq:kernelCC} holds in $\calH=L^2_{\rho}$ for $d=1,2,3$. 
\end{proposition}

\begin{proof}
We first simplify the interaction kernel coercivity condition by using the symmetry of the distribution and $L=1$.  
Since $\{X^i_{t_1}\}_{i=1}^N$ are identically distributed, so are the random variables $\{\br_{ij}= X^i_{t_1}-X^j_{t_1}\}$, and we have $\E[\trCov( \IK(\mathbf{r}_{ij}) \mid X^i_{t_1})] = \E[\trCov( \IK(\mathbf{r}_{12}) \mid X^1_{t_1})]$ for all $1\leq i\neq j\leq{N}$. 
 Additionally, since  since $L=1$, we have $\|\IK\|_{\rho_L}^2= \E[|\IK(\br_{12})|^2]$. Consequently,  the interaction kernel coercivity condition \eqref{eq:kernelCC} can be written as  
\begin{equation*}
\frac{1}{(N-1)}\sum_{j\neq i}  \E[\trCov( \IK(\mathbf{r}_{ij}) \mid X^i_{t_1})] = \E[\trCov( \IK(\mathbf{r}_{12}) \mid X^1_{t_1})] \\
 \geq c_{\mathcal H}^0 \|\IK\|_{\rho_L}^2 \, 	
\end{equation*}	 
for all $ \IK\in \calH$. It is equivalent to 
$
\E[ \big| \E[\IK(\br_{12})\mid X^1_{t_1}]\big|^2 ] \leq (1-c_{0,\calH}) \E[|\IK(\br_{12})|^2] 
$
by recalling that  
$\E[\trCov( \IK(\mathbf{r}_{12}) \mid X^1_{t_1})] = \E[ |\IK(\mathbf{r}_{12})|^2 ] - \E[\big|\E[\IK(\mathbf{r}_{12}) \mid X^1_{t_1}]\big|^2 ]$. 
Furthermore, since  $\{X^i_{t_1}\}_{i=1}^N$ are independent and identical, we have $\E[\big|\E[\IK(\mathbf{r}_{12}) \mid X^1_{t_1}]\big|^2 ] = \E[\innerp{\IK(\mathbf{r}_{12}), \IK(\mathbf{r}_{13}) }]$. Thus, to verify the interaction kernel coercivity condition, we only need to prove  
$
\E\left[\innerp{\IK(\mathbf{r}_{12}), \IK(\mathbf{r}_{13}) } \right] \leq (1-c_{0,\calH}) \E[|\IK(\br_{12})|^2]
$.
In particular, when $\IK(x) = \phi(|x|)\frac{x}{|x|}$, the above inequality reduces to  
\begin{align}\label{Cond_LbdPhi1}
	\E\left[\phi(| \br_{12}| )\phi(| \br_{13}| )\frac{\innerp{\br_{12},\br_{13}}}{|\br_{12}||\br_{13}|}\right] \leq  (1-c_{0,\calH})\E[| \phi(|\br_{12}|)|^2]\,. 
\end{align} 
	
Next, we prove \eqref{Cond_LbdPhi1} when $\{X^i_{t_1}\}_{i=1}$ are i.i.d.~Gaussian. 
Recall that if $X,Y\overset{i.i.d.}{\sim} \mu(x)={(2\pi)^{-d/2}}\exp(-|x|^2/2)$, then $X-Y \sim {(4\pi)^{-d/2}}\exp(-|x|^2/4)$ and
$|X-Y| \sim \rho(r)=C_d r^{d-1}e^{-\frac{r^2}{4}}\mathbf{1}_{\{r\geq 0\}}\,,$ 
where $C_d={1}/({2^{d-1}\Gamma(\frac d2)})$ and $\Gamma(\cdot)$ is the Gamma function.
In particular, one has $\rho(r)=e^{-\frac{r^2}{4}}\mathbf{1}_{\{r\geq 0\}}$, $\rho(r)=\frac{1}{2} re^{-\frac{r^2}{4}}\mathbf{1}_{\{r\geq 0\}}$ and $\rho(r)=\frac{1}{2\sqrt{\pi}} r^2e^{-\frac{r^2}{4}}\mathbf{1}_{\{r\geq 0\}}$ when $d=1$ ,$d=2$ and $d=3$, respectively.

Without loss of generality, we only need to consider $\E[| \phi(|\br_{12}|)|^2]= \|\phi\|_{L^2_\rho}^2=1$. By direct computation, the left-hand side of \eqref{Cond_LbdPhi1} is  
\begin{align}
\E\left[\phi(| \br_{12}| )\phi(| \br_{13}| )\frac{\innerp{\br_{12},\br_{13}}}{|\br_{12}||\br_{13}|}\right]
		&=\frac{1}{(2\sqrt{3}\pi)^d}\int_{\R^{2d}} \phi(|u|)\phi(|v|)\frac{\langle u, v\rangle}{|u||v|} e^{-\frac{(|u|^2+|v|^2-\langle u, v\rangle)}{3}}du dv \nonumber \\
		&=\frac{1}{(2\sqrt{3}\pi)^d}\int_{0}^{\infty}\int_0^{\infty}\phi(r)\phi(s) e^{-\frac{(r^2+s^2)}{3}} G_d(r,s) r^{d-1}s^{d-1} dr ds \label{Eq:crossK_d},
	\end{align}
	where the second equality follows from a polar coordinate transformation with 
	\begin{align}
		G_d(r,s)&=\int_{S^{d-1}} \int_{S^{d-1}} \langle \xi,\eta \rangle e^{\frac {rs}3 \langle \xi,\eta \rangle}d\xi d\eta \,. \label{Eq:crossKG_d}
	\end{align}
	We apply Cauchy-Schwarz inequality to \eqref{Eq:crossK_d} and $\|\phi\|_{L^2_\rho}^2=1$ to obtain that 
	\begin{align}
	\E\left[\phi(| \br_{12}| )\phi(| \br_{13}| )\frac{\innerp{\br_{12},\br_{13}}}{|\br_{12}||\br_{13}|}\right]
		&\leq \frac{1}{(2\sqrt{3}\pi)^d} \bigg[ \int_{0}^{\infty}\int_0^{\infty}|\phi(r)\phi(s)|^2 e^{-\frac{(r^2+s^2)}{4}} r^{d-1}s^{d-1} dr ds \bigg]^{\frac 12} \nonumber\\
		&\qquad \ \, \cdot \bigg[ \int_{0}^{\infty}\int_0^{\infty}|G_d(r,s)|^2 e^{-\frac{5(r^2+s^2)}{12}} r^{d-1}s^{d-1} dr ds \bigg]^{\frac 12} \label{Eq:crossK1_d}\\
		&= \frac{2^{d-1}\Gamma(\frac d2)}{(2\sqrt{3}\pi)^d} \bigg[ \int_{0}^{\infty}\int_0^{\infty}|G_d(r,s)|^2 e^{-\frac{5(r^2+s^2)}{12}} r^{d-1}s^{d-1} dr ds \bigg]^{\frac 12}=: I(d,G_d)\,. \nonumber
	\end{align}
	Thus, \eqref{Cond_LbdPhi1} holds with $1-c_{0,\calH}\geq I(d,G_d)$, equivalently, $c_{0,\calH}\leq 1-I(d,G_d) $. We compute $I(d,G_d)$ when $d=1$, $d=2$ and $d=3$ separately below.

	By \eqref{Eq:crossK1_d}, it is easy to see the key is the estimation of $G_d(r,s)$ such that $I(d,G_d)< 1$. Notice that $\int_{S^{d-1}} \langle \xi,\eta \rangle e^{\frac {rs}3 \langle \xi,\eta \rangle}d\xi$ is invariant with respect to any $\eta\in S^{d-1}$. Without loss of generality, we can select $\eta=e_1=(1,0,\cdots,0)\in S^{d-1}$ and write \eqref{Eq:crossKG_d} as
	\begin{align*}
	G_d(r,s)&= \int_{S^{d-1}} \int_{S^{d-1}} \langle \xi,e_1 \rangle e^{\frac {rs}3 \langle \xi,e_1 \rangle}d\xi d\eta=|S^{d-1}| \int_{S^{d-1}} \xi_1 e^{\frac {rs}3 \xi_1} d\xi  \,.
\end{align*}
	
\noindent	\textbf{Case $d=1$:} We have $S^{d-1}=\{-1,1\}$ and $|S^{d-1}|=2$. Thus, $G_1(r,s)=|S^{d-1}| \int_{S^{d-1}} \xi e^{\frac {rs}3 \xi} d\xi=2[e^{\frac{rs}{3}}-e^{-\frac{rs}{3}}]$. Plugging in $d=1$ and $G_1(r,s)^2=4[e^{\frac{2rs}{3}}+e^{-\frac{2rs}{3}}-2]$ into \eqref{Eq:crossK1_d}, we have by symmetry 
	\begin{align*}
		I(d,G_d)& =\frac{1}{\sqrt{3\pi}}\bigg[ \int_{0}^{\infty}\int_0^{\infty} e^{-\frac{5(r^2+s^2)}{12}}[e^{\frac{2rs}{3}}+e^{-\frac{2rs}{3}}-2]drds\bigg]^{\frac 12} \\
		&=\frac{1}{\sqrt{3\pi}}\bigg[\frac 12 \int_{\R^2}e^{-\frac{5(r^2+s^2)}{12}+\frac{2rs}{3}} drds-\frac 12 \int_{\R^2}e^{-\frac{5(r^2+s^2)}{12}} drds\bigg]^{\frac 12}
		=\frac{1}{\sqrt{3\pi}} \sqrt{2\pi-\frac{6\pi}{5}}=  \sqrt{\frac{4}{15}}\,.
	\end{align*}
	Hence, \eqref{Cond_LbdPhi1} holds with the coercivity constant $c_{0,\calH}\leq 1-\sqrt{\frac{4}{15}}\simeq 0.4836$.
	
\noindent\textbf{Case $d\geq 2$:} We can proceed to write
\begin{align*}
	G_d(r,s)&= |S^{d-1}| \int_{-1}^1 \int_{\{\sum_{i=2}^d\xi_i^2=1-\xi_1^2\}} \xi_1 e^{\frac {rs}3 \xi_1} d\xi = |S^{d-1}| |S^{d-2}|  \int_{-1}^1 \xi_1 (1-\xi_1^2)^{\frac{d-1}{2}}e^{\frac {rs}3 \xi_1} d\xi_1 \\
	&= |S^{d-1}| |S^{d-2}| \int_0^1 \xi (1-\xi^2)^{\frac{d-1}{2}} [e^{\frac {rs}3 \xi}-e^{-\frac {rs}3 \xi}] d\xi,
\end{align*}
where $|S^{n-1}|=\frac{2\pi^{\frac n2}}{\Gamma(n/2)}$ is the surface area of a $n$-dimensional sphere.	Thus, we have by Cauchy-Schwarz inequality
\begin{align*}
	I(d,G_d)&=\bar{C}_{d,1} \bigg[ \int_{0}^{\infty}\int_0^{\infty}\bigg|\int_0^1 \xi (1-\xi^2)^{\frac{d-1}{2}} [e^{\frac {rs}3 \xi}-e^{-\frac {rs}3 \xi}] d\xi \bigg|^2 e^{-\frac{5(r^2+s^2)}{12}} r^{d-1}s^{d-1} dr ds \bigg]^{\frac 12}, 
\end{align*}
where
$
	\bar{C}_{d,1} =\frac{2^{d-1}\Gamma(\frac d2)}{(2\sqrt{3}\pi)^d}\cdot |S^{d-1}| |S^{d-2}|= \frac{2^{d-1}\Gamma(\frac d2)}{(2\sqrt{3}\pi)^d}\cdot \frac{2\pi^{\frac d2}}{\Gamma(\frac d2)}\frac{2\pi^{\frac {d-1}2}}{\Gamma(\frac {d-1}2)}=\frac{2/\sqrt{\pi}}{3^{\frac d2}\Gamma(\frac {d-1}2)}
$.
We proceed by applying the Cauchy-Schwarz inequality and obtain that
\begin{align}
	\bigg|\int_0^1 \xi (1-\xi^2)^{\frac{d-1}{2}} [e^{\frac {rs}3 \xi}-e^{-\frac {rs}3 \xi}] d\xi \bigg|^2 &\leq \int_0^1 \xi^2 (1-\xi^2)^{d-1} d\xi \cdot \int_0^1 [e^{\frac {rs}3 \xi}-e^{-\frac {rs}3 \xi}]^2 d\xi \nonumber \\
	&= \bar{C}_{d,2}\bigg[\frac{3}{2rs}(e^{\frac {2rs}3}-e^{-\frac {2rs}3})-2 \bigg]\,,\label{Ineq:KG_Est}
\end{align}
with $\bar{C}_{d,2}=\frac{\sqrt{\pi} ~\Gamma(d)}{4\Gamma(d+3/2)}$. Letting 
$	J_0(d):=\int_{0}^{\infty}\int_0^{\infty} e^{-\frac{5(r^2+s^2)}{12}} \bigg[\frac{3}{2rs}(e^{\frac {2rs}3}-e^{-\frac {2rs}3})-2 \bigg]r^{d-1}s^{d-1} dr ds\,,
$ 
we can bound $I(d,G_d)$ above using the estimate \eqref{Ineq:KG_Est} as
$
	I(d,G_d)\leq \bar{C}_{d,1} \sqrt{\bar{C}_{d,2} J_0(d)} =:J(d)
$.
One can evaluate the function $J_0(2)$ and $J_0(3)$ directly:
$
	J_0(2)=6\arctan(3/4)-\frac{72}{25}$, $J_0(3)=8\pi-\frac{216\pi}{125}=\frac{784 \pi}{125}
$.
Combining the exact values of $\bar{C}_{d,1}$ and $\bar{C}_{d,2}$, we can evaluate the upper bounds of $J(d)$ when $d=2$ and $d=3$, obtaining $J(2)\simeq 0.1259$ and $J(3)\simeq 0.2661$.
	Therefore, we conclude that
	 \eqref{Cond_LbdPhi1} holds with $c_{0,\calH} \simeq 0.8731$ when $d=2$ and $c_{0,\calH} \simeq 0.7339$ when $d=3$.
\end{proof}

\begin{example}[A non-identifiable deterministic example]
\label{example:non-identifiable}
We present a non-identifiable deterministic system with $N=3$ particles when observing only one specific sample. 
That is, consider the system \eqref{eq:ips_K} with $N=3$, $\sigma=0$. Let the data be only one point,  $X^1_{t_0}=0$, $X^2_{t_0}=1$ and $X^3_{t_0}=2$. 
Let us construct two network-coefficient pairs $(\ba^{(1)},c^{(1)})$ and $(\ba^{(2)},c^{(2)})$ on the hypothesis function space $\calH=\mathrm{span}\{\sin(\frac{\pi x}{2}),\cos(\frac{\pi x}{2})\}$ with $c^{(1)}=[\frac{2}{2+\sqrt{2}},\frac{2}{2+\sqrt{2}}]^T$, $c^{(2)}=[0,1]^T$ and
\begin{align*}
	\ba^{(1)}=\begin{bmatrix}
		0 & 0 & 1\\
		1 & 0 & 0\\
		1 & 0 & 0
	\end{bmatrix}\,,\quad
	\ba^{(2)}=\begin{bmatrix}
		0 & \frac{\sqrt{2}}{2} & \frac{\sqrt{2}}{2}\\
		\frac{\sqrt{10-4\sqrt{2}}}{4} & 0 & \frac{\sqrt{2+\sqrt{2}}}{4}\\
		0 & 1 & 0
	\end{bmatrix}\,.
\end{align*}
One then can verify that $\ba^{(1)} \bB(\bX_{t_0})_i c^{(1)}=\ba^{(2)} \bB(\bX_{t_0})_i c^{(2)}$ for $i=1,2,3$ by noting that
\begin{align*}
	\bB(\bX_{t_0})_1
	=\begin{bmatrix}
		0 & 1 \\
		\frac{\sqrt{2}}{2} & \frac{\sqrt{2}}{2} \\
		1 & 0
	\end{bmatrix}\,,\quad
	\bB(\bX_{t_0})_2
	=\begin{bmatrix}
		-\frac{\sqrt{2}}{2} & \frac{\sqrt{2}}{2} \\
		0 & 1 \\
		\frac{\sqrt{2}}{2} & \frac{\sqrt{2}}{2}
	\end{bmatrix}\,,\quad
	\bB(\bX_{t_0})_3
	=\begin{bmatrix}
		-1 & 0 \\
		-\frac{\sqrt{2}}{2} & \frac{\sqrt{2}}{2} \\
		0 & 1 
	\end{bmatrix}\,.
\end{align*}
Consequently,  the loss function $\calE(\ba,c)$ in \eqref{eq:est_joint} can not distinguish two network-coefficient pairs $(\ba^{(1)},c^{(1)})$ and $(\ba^{(2)},c^{(2)})$ on the hypothesis function space $\calH$. This results in a lack of identifiability of jointly inferring the network $\ba$ and the kernel $\Phi$ in the system \eqref{eq:ips_K}. 

This example illustrates the fundamental identifiability challenge in joint inference problems, motivating the joint coercivity conditions below to ensure identifiability. As Proposition \ref{prop:kernelCC_example} demonstrates, we see that randomness is helpful to avoid non-identifiability.
\end{example}

\subsection{Coercivity and invertibility of normal matrices}\label{sec:CC_min_eig}

\begin{proof}[Proof of Proposition \ref{prop:Abar_orals} Part (i): regression matrices in ORALS]  
To study the singular value of $\mathcal{A}_{i,M}$ in \eqref{eq:Ai_operator}, it suffices to consider the smallest eigenvalue of the normal matrix $\overline{\mathcal{A}}_{i,M}:= \frac{1}{ML}\sum_{l=1, m = 1}^{L, M} [\mathcal{A}_{i}]_{l,m}^\top [\mathcal{A}_{i}]_{l,m} \in \R^{(N-1)\p\times (N-1)\p}$ since $\frac 1 M \sigma_{min}^2(\mathcal{A}_{i,M}) =\lambda_{min}(\overline{\mathcal{A}}_{i,M} ) $. 
We only need to discuss $i=1$. Also, to simplify notation, we consider only $L=1$, i.e., only the time instance $t=t_1$. Let $\mathbb{S}^{(N-1){\p}}:= \{u= (u_{j,k})\in \R^{(N-1){\p}}: \sum_{j=2}^N \sum_{k=1}^{\p} u_{j,k}^2=1\}$ and $f_j^u:=\sum_{k=1}^{\p}u_{j,k} \psi_k \in  \mathcal H$. Note that 
	\[ \sum_{j=2}^N  \| f_j^u \|_{\rho_L}^2 
	= \sum_{j=2}^N \sum_{k=1}^{\p} u_{j,k}^2 \|\psi_k \|_{\rho_L}^2 = 1,  \quad \forall u\in \mathbb{S}^{(N-1){\p}}.\] 
With these notations, we can write $\lambda_{\min}(\overline{\mathcal{A}}_{i,M})$ as: 
\begin{align}
	\lambda_{\min}(\overline{\mathcal{A}}_{1,M}) &= \min_{u\in \mathbb{S}^{(N-1){\p}}}u^\top\overline{\mathcal{A}}_{1,M} u =\min_{u\in \mathbb{S}^{(N-1){\p}}} \frac{1}{M}\sum_{m = 1}^{M} \Big|\sum_{j=2}^N \sum_{k=1}^{\p} u_{j,k}\psi_k(\mathbf{r}_{1j}^{m}(t_1))\Big|^2 \nonumber \\
	& = \min_{u\in \mathbb{S}^{(N-1){\p}}} \frac{1}{M}\sum_{m = 1}^{M} \Big|\sum_{j=2}^N f_j^u(\mathbf{r}_{1j}^{m}(t_1)) \Big|^2. 
\end{align}

First, we show that the minimal eigenvalue in the large sample limit is bounded from below. In fact, for each $u$, by the Law of large numbers and Lemma \ref{Lem:Var_lem}, we obtain    
\begin{align*}
	 u^\top \overline{\mathcal{A}}_{1,\infty} u   &= u^\top \E[\overline{\mathcal{A}}_{1,M}] u
	 =   \lim_{M\to \infty} \frac{1}{M}\sum_{m = 1}^{M} \Big|\sum_{j=2}^N f_j^u(\mathbf{r}_{1j}^{m}(t_1)) \Big|^2
	= \E \Big[ \Big|\sum_{j=2}^N f_j^u(\mathbf{r}_{1j}^{m}(t_1))\Big|^2 \Big] \\
	&=\E \bigg[\E \Big[ \Big|\sum_{j=2}^N f_j^u(\mathbf{r}_{1j}^{m}(t_1))\Big|^2\mid \calF_{t_1}^1\Big] \bigg]
	\geq \E \bigg[ \sum_{j=2}^N \trCov\Big(f_j^u(\mathbf{r}_{1j}^{m}(t_1)) \mid \calF_{t_1}^1\Big) \bigg] \geq \sum_{j=2}^N c_{\mH}\| f_j^u \|_{\rho_L}^2  = c_\calH, 
\end{align*}
where the last inequality follows from the interaction kernel coercivity condition \eqref{eq:kernelCC}. 

Next, we apply a matrix version of Bernstein concentration inequality to obtain the non-asymptotic bound (e.g., \cite[Theorem 6.1]{Tropp2012}) to obtain \eqref{Ineq:Conc_calA}. 
We write
$\widebar{Q}_{M}=\overline{\mathcal{A}}_{1,M}-\overline{\mathcal{A}}_{1,\infty}=\frac{1}{M}\sum_{m = 1}^{M} [\mathcal{A}_{1,m}^\top $ $ \mathcal{A}_{1,m}-\overline{\mathcal{A}}_{1,\infty}]=:\frac{1}{M}\sum_{m = 1}^{M} Q_m$,
and notice that $\{[\mathcal{A}_{1,m}^\top \mathcal{A}_{1,m}-\overline{\mathcal{A}}_{1,\infty}]\}_{m=1}^M$ has zero mean. Because $\|Q_m\|\leq \p N L_{\mH}^2$ and the matrix variance of the sum can be bounded as
\begin{align*}
	V(\widebar{Q}_{M}):=\frac{1}{M^2}\|\sum_{m=1}^M\E[Q_m Q_m^{\top}]\|\leq 2(\p N L_{\mH}^2)^2/M\,,
\end{align*}  
we obtain
$	\P\{ \|\widebar{Q}_{M}\|\geq \varepsilon \}\leq 2{\p}N \exp\rbracket{-\frac{M\varepsilon^2/2}{2(\p N L_{\mH}^2)^2+\p NL_{\mH}^2\varepsilon/3}}\,.
$ 
So, for $0<\varepsilon<c_{\mH}$
\begin{align*}
	\P \bigg\{ \lambda_{\min}(\overline{\mathcal{A}}_{1,M})> c_{\mH}-\epsilon \bigg\}&\geq \P \bigg\{ |\lambda_{\min}(\overline{\mathcal{A}}_{1,M})-\lambda_{\min}(\overline{\mathcal{A}}_{1,\infty})|\leq \epsilon \bigg\} \\
	&\geq \P \bigg\{\|\widebar{Q}_{M}\|> c_{\mH}-\epsilon \bigg\}\geq 1-2{\p}N \exp\rbracket{-\frac{M\varepsilon^2/2}{2(\p N L_{\mH}^2)^2+\p NL_{\mH}^2\varepsilon/3}}
\end{align*}
where we used $|\lambda_{\min}(\overline{\mathcal{A}}_{1,M})-\lambda_{\min}(\overline{\mathcal{A}}_{1,\infty})|\leq \|\widebar{Q}_{M}\|$.
\end{proof}

\bigskip
  
\begin{proof}[Proof of Proposition \ref{prop:Abar_orals} part (ii): matrices in ALS] Recall that here we assume the joint-coercivity condition (which is weaker than the kernel coercivity condition assumed in part (i)).  
The proof is based on the standard concentration argument combined with the lower bound for the large sample limit for the matrix in the normal equations corresponding to \eqref{eq:ALS_graphest}, which are:
\begin{equation}\label{eq:Abar_a}
\begin{aligned}
\widehat \ba_{i\cdot} & =  \overline{\Agraph}_{i,M}^{\dagger} \overline{\bgraph}_{i,M}, \quad  \text{with } \\
\overline{\Agraph}_{i,M} & = \mathcal{A}^{\textrm{ALS}}_{c,M} (\mathcal{A}^{\textrm{ALS}}_{c,M})^\top = \frac{1}{ML}\sum_{l=1, m = 1}^{L, M} \Agraph_{i}^m(t_l),   \quad \Agraph_{i}^m(t_l):=  (\bB(\bX_{t_l}^m)_i c) ( \bB(\bX_{t_l}^m)_i] c)^\top \in \R^{N\times N}, \\
\overline{\bgraph}_{i,M} & = [(\Delta \bX_{t_l})_i ^m)]_{l,m} (\mathcal{A}^{\textrm{ALS}}_{c,M})^\top =  \frac{1}{MT}\sum_{l=1, m = 1}^{L, M}  \bgraph_{i}^m(t_l), \quad \bgraph_{i}^m(t_l) :=  (\Delta \bX_{t_l})_i ^m) (\bB(\bX_{t_l}^m)_i c)^\top  \in \R^{N\times 1},
\end{aligned}
\end{equation}
where, for each $i$, we treat the array $\bB(\bX_{t_l}^m)_i c\in \R^{N\times 1\times d}$ as a matrix in $\R^{N\times d}$, and we set $\ba_{ii}=0$ so that we are effectively solving a vector in $\R^{N-1}$.  
When $\overline{\Agraph}_{i,M}$ is rank-deficient, or even when it has a large condition number, the inverse may be replaced by the Moore-Penrose pseudoinverse.

\noindent\textbf{Part (a).} 
Let $c=(c_1,\cdots,c_{\p})^\top\in \R^{\p\times 1}$ be nonzero and denote $\IK= \sum_{k=1}^{\p} c_k \psi_k$. 
Recall that $\overline{\Agraph}_{i,M} = \frac{1}{ML}\sum_{l=1, m = 1}^{L, M} \Agraph_{i}^m(t_l)$  with 
\[ 
\Agraph_{i}^m(t_l)= \bB(\bX_{t_l}^m)c c^\top \bB(\bX_{t_l}^m)^\top = \Big[ \innerp{\IK(\br_{ij}^m(t_l)),\IK(\br_{ij'}^m(t_l))}_{\R^d} \Big]_{1\leq j,j'\leq{N}, j\neq i}.
\]
Without loss of generality, we assume $L= 1$. We only need to consider $i=1$, and the cases $i= 2,\cdots, N$ are similar.   
For any $a\in \mathbb{S}^{N-1}$, note that 
\[
a^\top \overline{\Agraph}_{i,M} a  =  \frac{1}{M}\sum_{m = 1}^{M} a^\top \Agraph_{1}^m (t_1) a= \frac{1}{M}\sum_{m = 1}^{M} \Big|\sum_{j=2}^N a_j\IK(\br_{1j}^m(t_1))\Big|^2. 
\]
Then, the joint coercivity condition \eqref{eq:jointCC} implies that 
\begin{align*}
	a^\top \overline{\Agraph}_{1,\infty}  a
	&=  \E\bigg[ \Big|\sum_{j=2}^{N} a_{j} \IK(\mathbf{r}_{1j}(t_1)) \Big|^2 \bigg]
	\geq c_{\mH} \| a\|^2 \|\IK\|_{\rho_L}^2 = c_{\mH}  \| c\|^2\,,
\end{align*}
where the last equality follows from $\| a\|^2=1$ and $\|\IK\|_{\rho_L}^2 
	=\| \sum_{k=1}^{\p} c_k \psi_k \|_{\rho_L}^2= \|c\|^2$.   
	Thus, 
\begin{align}\label{Ldb:ALS_eig1}
	\lambda_{\min}(\overline{\Agraph}_{1,\infty})&=\min_{a \in \mathbb{S}^{N-1}}a^\top \overline{\Agraph}_{1,\infty}  a 
	\geq c_{\mH} \| c\|^2\,.
\end{align}

Next, we show that the lower bound holds for the smallest eigenvalue of the empirical normal matrix with a high probability based on the matrix Bernstein inequality. The proof closely parallels that of \eqref{Ineq:Conc_calA}, and we omit some details.
Setting
$
	\widebar{Q}_{M}^{(1)}=\overline{\Agraph}_{1,M}-\overline{\Agraph}_{1,\infty}=\frac{1}{M}\sum_{m = 1}^{M} [\Agraph_{i}^m(t_1)-\overline{\Agraph}_{1,\infty}]
$,
the matrix Bernstein inequality reveals that
\begin{align}\label{Bernstein:Q1}
	\P\{ \|\widebar{Q}_{M}^{(1)}\|\geq \varepsilon \}\leq 2N \exp\rbracket{-\frac{M\varepsilon^2/2}{(\p L_{\mH}^2)^2+pL_{\mH}^2\varepsilon/3}}\,.
\end{align}
The rest is the same as the proof of \eqref{Ineq:Conc_calA}.
\bigskip

\noindent\textbf{Part (b)}. Fix $\ba\in \R^{N\times N}$ with each row normalized, namely, $\|\ba_i\|=1$ fpr every $i\in[N]$. Let $c\in \R^{\p}$ with $\|c\|=1$ and let $K= \sum_{k=1}^{\p} c_k \psi_k$. The normal equations for \eqref{eq:ALS_kernelest} and their solution are
\begin{equation}\label{eq:Abar_c}
\begin{aligned}
\widehat{c} &= \overline{A}_M^{\dagger}\overline{b}_M, \text{ where } \\
\overline{A}_M &:= (\mathcal{A}^{\textrm{ALS}}_{c,M})^\top \mathcal{A}^{\textrm{ALS}}_{c,M} = \frac{1}{ML}\sum_{l=1, m = 1}^{L, M} A_l^m, \quad A_l^m= (\ba \bB(\bX_{t_l}^m))^\top \ba\odot\bB(\bX_{t_l}^m) \in \R^{\p\times\p} \\
\overline{b}_M &:= (\mathcal{A}^{\textrm{ALS}}_{c,M})^\top [\Delta \bX_{t_l}^m]_{l,m} = \frac{1}{MT}\sum_{l=1, m = 1}^{L, M} b_l^m,  \quad b_l^m= ( \ba \bB(\bX_{t_l}^m) )^\top \Delta \bX_{t_l}^m \in \R^{\p\times 1}\,,
\end{aligned}
\end{equation} 
so that
$c^\top \overline{A}_M c = \frac{1}{ML}\sum_{l=1, m = 1}^{L, M} c^\top A_l^m c  \in \R^{\p\times\p}$, where
\[c^\top A_l^m c=  c^\top \bB(\bX_{t_l}^m)^\top \ba \ba^\top \bB(\bX_{t_l}^m )c  =\frac{1}{N}\sum_{i=1}^{N}  \Big|\sum_{j=2}^{N} a_{ij} \IK(\mathbf{r}_{ij}(t_l)) \Big|^2 . 
\]
Again, without loss of generality, we can assume $L=1$, and as the argument before, we get from the joint coercivity condition \eqref{eq:kernelCC} that 
\begin{align*}
	c^\top \overline{A}_{\infty}   c  = c^\top  \E [\overline{A}_{M} ]  c
	&=\frac{1}{N}\sum_{i=1}^{N} \E \bigg[\Big|\sum_{j=2}^{N} a_{ij} \IK(\mathbf{r}_{ij}(t_l)) \Big|^2 \bigg]\geq c_{\mH}  \frac{1}{N}\sum_{i=1}^{N} \| \ba_i\|^2 \|\IK\|_{\rho_L}^2  = c_\calH,
\end{align*}
where the last equality follows from the fact that $\|\ba_i\|^2=1$ and $\|\IK\|_{\rho_L}^2 = \|c\|^2=1$. 
Thus, 
$	\lambda_{\min}(\overline{A}_{\infty})=\min_{c\in \in \mathbb{S}^{p}}  c^\top \overline{A}_{\infty}   c  \geq c_{\mH}\,.
$ 

Lastly, same as in the proof of (a), we define
$
	\widebar{Q}_{M}^{(2)}=\overline{A}_M -\overline{A}_{\infty} = \frac{1}{M}\sum_{m = 1}^{M} A_0^m
$
and then obtain a similar result as in \eqref{Bernstein:Q1} switching $N$ and $\p$. So,
$	\P \bigg\{ \lambda_{\min}(\overline{A}_{i,M})\geq c_\calH-\epsilon \bigg\} \geq 1-2{\p} \exp\rbracket{-\frac{M\varepsilon^2/2}{(N L_{\mH}^2)^2+NL_{\mH}^2\varepsilon/3}}\,.
$ 
The proof is completed.
\end{proof}

\subsection{Convergence of the ORALS estimator}
\begin{proof}[Proof of Theorem \ref{thm:AN_orals}]
We consider the normal equations associated with \eqref{eq:Ai_operator}.

To prove part (i),  recall that for each $i$ fixed, $\{[\mathcal{A}_{i}]\in \R^{Ld\times (N-1){\p}}\}_{m=1}^M$ are independent identically distributed for each $m$, hence by Law of large numbers
$\overline{\mathcal{A}}_{i,M}  = \frac{1}{ML}\sum_{l=1, m = 1}^{L, M} [\mathcal{A}_{i}]_{l,m}^\top [\mathcal{A}_{i}]_{l,m}\rightarrow \overline{\mathcal{A}}_{i,\infty}$ a.s. for $M\to \infty$. 
By Proposition \ref{prop:Abar_orals}, $\overline{\mathcal{A}}_{i,\infty}$ is invertible, with smallest eigenvalue no smaller than $c_\calH$, and $\overline{\mathcal{A}}_{i,M}$ is invertible with smallest eigenvalue larger than $c_\calH/2$  w.h.p., with Gaussian tails in $M$.
By standard argument employing the Borel-Cantelli lemma, $\overline{\mathcal{A}}_{i,M}^{-1} \to \overline{\mathcal{A}}_{i,\infty}^{-1}$ a.s.~ as $M\to \infty$. 
	
	Meanwhile, making use of \eqref{eq:Euler} and the notation 
$\mathcal{A}_{i,m}(t_l) z_i = (\ba \bB (\bX_{t_l}^m) c\Delta t)_i $ in \eqref{eq:Ai_operator}, we have 
	   \begin{align*}
	  \overline{v}_{i,M}  =  \frac{1}{ML}\sum_{l=1, m = 1}^{L, M} [\mathcal{A}_{i}]_{l,m}^\top [(\Delta\bX)_i]_{l,m} = \overline{\mathcal{A}}_{i,M} z_i +\widetilde v_{i,M} ,
\end{align*}
where $\widetilde v_{i,M} := \sigma \sqrt{\Delta t} \frac{1}{ML}\sum_{l=1, m = 1}^{L, M} [\mathcal{A}_{i}]_{l,m}^\top (\Delta\bW_{t_l}^m)_i$. Note that $\widetilde v_{i,M}$ is a sum of $M$ independent square integrable samples since the basis functions are uniformly bounded under Assumption \ref{assum:CC}. Thus, by Central Limit Theorem, we have $\sqrt{M}\widetilde v_{i,M}$ converges in distribution to a  $\calN(0,\sigma^2 \Delta t \overline{\mathcal{A}}_{i,\infty} )$-distributed Gaussian vector.   
Hence, together with the above fact that $\overline{\mathcal{A}}_{i,M}^{-1} \to \overline{\mathcal{A}}_{i,\infty}^{-1}$ a.s. as $M\to \infty$, we have by Slutsky's theorem that the random vector  
\begin{equation}\label{eq:xi_M}
{\xi}_{i,M}:=\overline{\mathcal{A}}_{i,M} ^{-1} \widetilde v_{i,M} \, \xrightarrow{d}\, \overline{{\xi}}_{i,\infty} \overset{d}{\sim} \calN(0,(\sigma \Delta t)^2 \overline{\mathcal{A}}_{i,\infty}^{-1} ) 	
\end{equation}
where $\overline{\mathcal{A}}_{i,M} ^{-1}$ is the pseudo-inverse when the matrix is singular. Consequently, the estimator  
$
	\widehat{z}_{i,M} = \overline{\mathcal{A}}_{i,M} ^{-1}\overline{v}_{i,M}  = z_i+{\xi}_{i,M}
$
is asymptotically normal.

Part (ii) follows from the explicit form of the 1-step and 2-step iteration estimators. Denote $\pmb{\xi}_{i,M}\in \R^{(N-1)\times\p}$ the matrix converted from ${\xi}_{i,M}\in \R^{(N-1)\p\times 1}$ in \eqref{eq:xi_M}, i.e., ${\xi}_{i,M}=\rm{Vec}(\pmb{\xi}_{i,M})$. Then, as $M\to \infty$,  $\sqrt{M}\pmb{\xi}_{i,M}$ converges in distribution to the centered Gaussian random matrix $\pmb{\xi}_i$, the inverse vectorization of  the Gaussian vector $\overline{\xi}_{i,\infty} $ in \eqref{eq:xi_M}.

Starting from  $c_0\in \R^{\p\times 1}$ with $c_*^\top c_0 \neq 0$, the first step of the deterministic ALS minimizes the loss function $\calE_M(\ba,c_0)$ with respect to $\ba$ to obtain, for $i\in[N]$,
	\begin{align*}
		(\widetilde\ba^{M,1})_i^\top&=|c_0|^{-2}\widehat\bZ_{i,M} c_0 =|c_0|^{-2}[  (c_*^\top c_0) (\ba_{*})_{i}^\top+ \pmb{\xi}_{i,M}c_0]\,. 
	\end{align*}
		Then, noting that $\|(\ba_{*})_{i}\|=1$, we have
	\begin{align*}
		\|(\widetilde\ba^{M,1})_i^\top\|^2 &= |c_0|^{-4} (c_*^\top c_0)^2 \|(\ba_{*})_{i}^\top+\eta_{i,M}^{(1)} \|^2\, = |c_0|^{-4} (c_*^\top c_0)^2( 1+\varepsilon_{i,M}^{(1)})\,, 
	\end{align*}
	where we have let
	\begin{equation}\label{Eq:eta+eps}
		\eta_{i,M}^{(1)}:=(c_*^\top c_0)^{-1}\pmb{\xi}_{i,M}c_0\in \R^{N\times 1}, \quad \varepsilon_{i,M}^{(1)}:=2 (\ba_{*})_{i} \eta_{i,M}^{(1)} +\|\eta_{i,M}^{(1)}\|^2. 
	\end{equation}
	Hence, the normalized 1-step estimator can be written as 
	\begin{align*}
		(\widehat\ba^{M,1})_i^\top =(\widetilde\ba^{M,1})_i^\top/\|(\widetilde\ba^{M,1})_i^\top\| =\frac{(\ba_{*})_{i}^\top+ \eta_{i,M}^{(1)}}{\|(\ba_{*})_{i}^\top+\eta_{i,M}^{(1)} \|}=\frac{(\ba_{*})_{i}^\top+ \eta_{i,M}^{(1)}}{\sqrt{1+\varepsilon_{i,M}^{(1)}}} \,.
	\end{align*}
	Thus, the difference between $(\widehat\ba^{M,1})_i^\top$ and $(\ba_{*})_{i}^\top$ is
	\begin{align}
		(\widehat\ba^{M,1})_i^\top-(\ba_{*})_{i}^\top &=\frac{1-\sqrt{1+\varepsilon_{i,M}^{(1)}}}{\sqrt{1+\varepsilon_{i,M}^{(1)}}} (\ba_{*})_{i}^\top+\frac{\eta_{i,M}^{(1)}}{\sqrt{1+\varepsilon_{i,M}^{(1)}}} \nonumber \\
		&=\frac{-\varepsilon_{i,M}^{(1)}}{\sqrt{1+\varepsilon_{i,M}^{(1)}}(1+\sqrt{1+\varepsilon_{i,M}^{(1)}})} (\ba_{*})_{i}^\top+ \frac{\eta_{i,M}^{(1)} }{\sqrt{1+\varepsilon_{i,M}^{(1)}}} \label{Eq:aM1_diff}
	\end{align}
	where $\eta_{i,M}^{(1)}$ and $\varepsilon_{i,M}^{(1)}$ are defined in \eqref{Eq:eta+eps}. 	
	
	By Slutsky's theorem, we get $\sqrt{M}\eta_{i,M}^{(1)}\overset{d}{\to} (c_*^\top c_0)^{-1}  \pmb{\xi}_i c_0$, and by Lemma \ref{Lem:AN_lem1} we obtain 
	\begin{align*}
		\sqrt{M}\varepsilon_{i,M}^{(1)} &= 2(c_*^\top c_0)^{-1}\sqrt{M}{(\ba_{*})_{i}\pmb{\xi}_{i,M}c_0}+(c_*^\top c_0)^{-2}\sqrt{M}{\|\pmb{\xi}_{i,M}c_0\|^2} 
	\, \xrightarrow{d}\, 2(c_*^\top c_0)^{-1} (\ba_{*})_{i} \pmb{\xi}_i c_0 \,.
	\end{align*}
	Consequently, the asymptotic normality of $(\widehat\ba^{M,1})_i$ follows from 
	\begin{align}
		\sqrt{M}[(\widehat\ba^{M,1})_i^\top-(\ba_{*})_{i}^\top] \overset{d}{\to} (c_*^\top c_0)^{-1} [  \pmb{\xi}_i c_0- (\ba_{*})_{i} \pmb{\xi}_i c_0 (\ba_{*})_{i}^\top]\,. \label{Eq:AN_aM1}
	\end{align}
	Note that the limit distribution depends on the initial condition $c_0$. This dependence on $c_0$ will be removed in the 2nd-iteration.

Next, by minimizing the loss function $\calE(\widehat\ba^{M,1}, c)$ with respect to $c$, we obtain $\widehat c^{M,1}$: 
	\begin{align}
		\widehat c^{M,1}&= \bigg[\sum_{i=1}^N (\widehat\ba^{M,1})_i (\widehat\ba^{M,1})_i^\top \bigg]^{-1} \sum_{i=1}^N \widehat\bZ_{i,M}^\top (\widehat\ba^{M,1})_i^\top\,. \label{Eq:cM1}
	\end{align}
Note that $\sum_{i=1}^N (\widehat\ba^{M,1})_i (\widehat\ba^{M,1})_i^\top= N$ since $\|(\widehat\ba^{M,1})_i\|=1$. Thus,  
	\begin{align}
		\widehat c^{M,1}-c_*&=\bigg[\frac 1N \sum_{i=1}^N (\ba_*)_i (\widehat\ba^{M,1})_i^\top-1\bigg] c_* + \frac 1N\sum_{i=1}^N \pmb{\xi}_{i,M}^\top (\widehat\ba^{M,1})_i^\top \nonumber \\
		&=\frac 1N \sum_{i=1}^N \bigg[(\ba_*)_i \frac{[(\ba_{*})_{i}^\top+ \eta_{i,M}^{(1)}]}{\sqrt{1+\varepsilon_{i,M}^{(1)}}}-1 \bigg]c_*+\frac 1N\sum_{i=1}^N \xi_{i,M}^\top \frac{(\ba_{*})_{i}^\top+ \eta_{i,M}^{(1)}}{\sqrt{1+\varepsilon_{i,M}^{(1)}}} \label{eq:cM_diff} \\
		&=\frac 1N \sum_{i=1}^N \frac{-\varepsilon_{i,M}^{(1)}}{\sqrt{1+\varepsilon_{i,M}^{(1)}}(1+\sqrt{1+\varepsilon_{i,M}^{(1)}})} c_* + \frac 1N \sum_{i=1}^N  \frac{ (\ba_{*})_{i}\eta_{i,M}^{(1)} c_* }{\sqrt{1+\varepsilon_{i,M}^{(1)}}} \nonumber\\
		&+\frac 1N\sum_{i=1}^N  \frac{\pmb{\xi}_{i,M}^\top(\ba_{*})_{i}^\top+ \pmb{\xi}_{i,M}^\top\eta_{i,M}^{(1)}}{\sqrt{1+\varepsilon_{i,M}^{(1)}}}\,.\nonumber
	\end{align}
	Again, using Lemma \ref{Lem:AN_lem1} and Slutsky's theorem, we get the asymptotic normality of $\widehat c^{M,1}$
	\begin{align}\label{Eq:AN_cM}
		\sqrt{M}[\widehat c^{M,1}-c_*] \overset{d}{\to} \frac 1N\sum_{i=1}^N \pmb{\xi}_i^\top (\ba_{*})_{i}^\top\,. 
	\end{align}

We remove the dependence of $c_0$ in the limit distribution in \eqref{Eq:AN_aM1} by another iteration. That is, we minimize the loss function $\calE(\ba, \widehat c^{M,1})$ with respect to $\ba$ to obtain $(\widehat\ba^{M,2})_i$. With the same argument used for $(\widehat\ba^{M,1})_i$, with $c_0$ from \eqref{Eq:eta+eps} replaced by $\widehat c^{M,1}$ from \eqref{Eq:cM1}, we obtain an update 
	\begin{align*}
				\eta_{i,M}^{(2)}&:=(c_*^\top \widehat c^{M,1})^{-1}\pmb{\xi}_{i,M}\widehat c^{M,1}\,, \\ 
				\varepsilon_{i,M}^{(2)}&:=2 (\ba_{*})_{i} \eta_{i,M}^{(2)} +\|\eta_{i,M}^{(2)}\|^2 
				= 2(c_*^\top \widehat c^{M,1})^{-1}{(\ba_{*})_{i}\pmb{\xi}_{i,M}\widehat c^{M,1}}+(c_*^\top \widehat c^{M,1})^{-2}{\|\pmb{\xi}_{i,M}\widehat c^{M,1}\|^2}\,.
	\end{align*}
	Note that $\eta_{i,M}^{(2)}$ and $\varepsilon_{i,M}^{(2)}$ are well-defined because $c_*^\top \widehat c^{M,1} \neq 0$ almost surely. The asymptotic normality \eqref{Eq:AN_cM} implies $\widehat c^{M,1}$ converges to $c_*$ almost surely as $M$ tends to infinity. Hence, combining $\sqrt{M}\pmb{\xi}_{i,M} \overset{d}{\to} \pmb{\xi}_i$ with Lemma \ref{Lem:AN_lem1} and Slutsky's theorem we obtain
	\begin{align*}
		\sqrt{M}\eta_{i,M}^{(2)}&\overset{d}{\to} |c_*|^{-2}  \pmb{\xi}_i c_* \,,\qquad
		\sqrt{M}\varepsilon_{i,M}^{(2)} 
		\overset{d}{\to}2|c_*|^{-2} (\ba_{*})_{i} \pmb{\xi}_i c_* \,.
	\end{align*}
	Replacing $\eta_{i,M}^{(1)}$ and $\varepsilon_{i,M}^{(1)}$ by $\eta_{i,M}^{(2)}$ and $\varepsilon_{i,M}^{(2)}$ in \eqref{Eq:aM1_diff} respectively, we obtain the asymptotic normality
	\begin{align}
		\sqrt{M}[(\widehat\ba^{M,2})_i^\top-(\ba_{*})_{i}^\top] \overset{d}{\to} |c_*|^{-1} [  \pmb{\xi}_i c_*- (\ba_{*})_{i} \pmb{\xi}_i c_* (\ba_{*})_{i}^\top]\,. \label{Eq:AN_aM2}
	\end{align}
	Combining \eqref{Eq:AN_cM} and \eqref{Eq:AN_aM2}, we complete the proof of (ii).	
\end{proof}	

\begin{lemma}\label{Lem:AN_lem1} 
Let $\{\xi_M\}_{M=1}^\infty$ be a sequence of square integrable $\R^{(N-1)\p\times 1}$-valued random variables such that $\sqrt{M}\xi_M\xrightarrow{d} \xi_{\infty}$ as $M\to \infty$, where $\xi_\infty\overset{d}{\sim} \calN(0,\Sigma)$ with a nondegenerate $\Sigma$. Denote $\pmb{\xi}_M$ and $\mathbf{N}$ the random matrices corresponding to $\xi_M= \mathrm{Vec}(\pmb{\xi})$ and $\xi_\infty= \mathrm{Vec}(\mathbf{N})$, respectively. Also, let $\ba\in \R^{N\times N}$ and assume $c_M\to c$ almost surely as $M\to \infty$. 
Then, 
\begin{itemize}
	\item [(i)] $\sqrt{M}\pmb{\xi}_{M} c \overset{d}{\to}  \bN c$ and $\sqrt{M}\ba \pmb{\xi}_{M} c \overset{d}{\to} \ba \bN c$; 
	\item [(ii)] $\sqrt{M}\pmb{\xi}_{M} c_M \overset{d}{\to}  \bN c$ and $\sqrt{M}\ba \pmb{\xi}_{M} c_M \overset{d}{\to} \ba \bN c$; and
	\item [(iii)] $\sqrt{M}\pmb{\xi}_{M}^\top\pmb{\xi}_{M}c\to 0$ and $\sqrt{M}\|\pmb{\xi}_{M}c\|^2\to 0$ almost surely.
\end{itemize}
 
 \end{lemma}
\begin{proof} Part (i) follows directly from the convergence of ${\xi}_M$. Part (ii) and (iii) can be derived from the Borel-Cantelli lemma and Slutsky's theorem.
\end{proof}

\subsection{Trajectory prediction error} 
\begin{proof}[Proof of Proposition \ref{p:trajerrbounds}]
Since $\widehat \bX_t$ and $\bX_t$ have the same initial condition and driving force, we have 
$
	\widehat \bX_t - \bX_t = \int_0^t [ \widehat \ba \bB(	\widehat \bX_s) \widehat c - \ba \bB(	 \bX_s)  c] \, ds.   $

By Jensen's inequality in the form $| \frac{1}{t}\int_0^t f(s)ds|^2\leq t \int_0^t |f(s)|^2 ds$,
	\begin{align}\label{eq:Etraj}
	\E \| \widehat \bX_t - \bX_t\|_F^2
	& \leq t \int_0^t 	\E \|  \widehat \ba \bB(	\widehat \bX_s) \widehat c - \ba \bB(	\bX_s)  c\|_F^2 \, ds.    	
	\end{align}	
Next, we seek a bound for the integrand. With $\br_s^{i,j}:=X^j_s-X^i_s$, $\widehat\br_s^{i,j}:=\widehat X^j_s-\widehat X^i_s$,  $\IK(\br_s^{i,j}) := \sum_{k=1}^{\p} c_k \psi_k(\br_s^{i,j})$, we write  $\ba \bB(\bX_s)c =( \sum_{j\neq i} \ba_{ij} \IK(\br_s^{i,j}))_{i\in[N]} \in \R^{N\times d}$, and similarly for $\widehat \ba \bB(\widehat \bX_s) \widehat c$. Applying Jensen's inequality $\| \sum_{j\neq i} A_j\|_{\R^d}^2 \leq \frac{1}{N-1}\sum_{j\neq i} \|A_j\|_{\R^d}^2 $ and the triangle inequality, we obtain
	\begin{align}
	& \|  \widehat \ba \bB(	\widehat \bX_s) \widehat c - \ba \bB(
	\bX_s)  c\|_F^2
	= \sum_{i=1}^N  \bigg[\Big\| \sum_{j\neq i}\Big[\widehat{\ba}_{ij}\widehat{\IK}(\widehat\br_s^{i,j})-\ba_{ij}\IK(\br_s^{i,j})\Big] \Big\|_{\R^d}^2 \bigg]	\label{eq:traj_est}		 \\
    &\qquad \leq  \frac{1}{N-1} \sum_{i=1}^N \sum_{j\neq i} |\widehat{\ba}_{ij}-\ba_{ij}|^2 \big\|\IK(\br_s^{i,j})\big\|^2  
      + \frac{1}{N-1} \sum_{i=1}^N  \sum_{j\neq i} |\widehat{\ba}_{ij}|^2 \bigg[  \Big\|\widehat{\IK}(\widehat\br_s^{i,j})-\IK(\br_s^{i,j})\Big\|_{\R^d}^2 \bigg].   \notag
	\end{align}	
	We bound the above two terms in the last inequality by $\|\widehat \ba -\ba\|_F^2 $  and $\|\widehat c - c\|^2$ using the uniform boundedness of the basis functions:
for the first term we have
	\begin{align*}
		 \frac{1}{N-1} \sum_{i=1}^N \sum_{j\neq i} |\widehat{\ba}_{ij}-\ba_{ij}|^2 \big\| \IK(\br_s^{i,j})\big\|^2  
		&\leq \frac{{\p}C_0^2\|c\|_{2}^2}{N-1} \sum_{i=1}^N \sum_{j\neq i} |\widehat{\ba}_{ij}-\ba_{ij}|^2 
		\leq \frac{{\p}C_0^2\|c\|_{2}^2}{N-1} \|\ba-\widehat\ba\|_F^2\,,
	\end{align*}
	where the first inequality follows from $\big\| \IK(\br_s^{i,j})\big\|^2= \big\| \sum_{k=1}^{\p} c_k \psi_k(\br_s^{i,j}) \big\|^2\leq{\p}\|c\|_{2}^2 C_0^2$ for each $(i,j,s)$, since $\|\psi_k\|_\infty\leq C_0$ by assumption. 
	The bound on the second term follows from the assumptions on the basis functions and entry-wise boundedness of the weight matrix. We first observe that 
	\begin{align*}
 	\Big\|\widehat{\IK}(\widehat\br_s^{i,j})-\IK(\br_s^{i,j}) \Big\|^2  
	& \leq   \Big\|\widehat \IK(\br_s^{i,j})-\IK(\br_s^{i,j}) \Big\|^2  +   \Big\|\widehat{\IK}(\widehat\br_s^{i,j})-\widehat \IK(\br_s^{i,j}) \Big\|^2
		\leq{\p} \|c-\widehat c \|^2C_0^2  +{\p}\|\widehat c\|^2 C_0^2 \|\br_s^{i,j}-\widehat\br_s^{i,j} \|^2 \,, 
	  \end{align*}
	where the second inequality follows from
	\begin{align*}
	 \Big\|\widehat \IK(\br_s^{i,j})-\IK(\br_s^{i,j}) \Big\|^2 
	   &= \Big\|\sum_{k=1}^{\p} \big[c_k-\widehat c_k \big] \psi_k(\br_s^{i,j})\Big\|^2  
	  \leq{\p} \|c-\widehat c \|^2 C_0^2,    \\
     \Big\|\widehat{\IK}(\widehat\br_s^{i,j})-\widehat \IK(\br_s^{i,j}) \Big\|^2
       & = \Big\|\sum_{k=1}^{\p} \widehat c_k \big[ \psi_k(\br_s^{i,j})-\psi_k(\widehat\br_s^{i,j})\big] \Big\|^2 
        \leq{\p} \|\widehat c\|^2 C_0^2 \|\br_s^{i,j}-\widehat\br_s^{i,j} \|^2\,,
	\end{align*}
	 for each $(i,j,s)$, since $\|\psi_k\|_\infty\leq C_0$ and $\|\nabla \psi_k\|_\infty\leq C_0$.
	Hence, for the second term in \eqref{eq:traj_est}: 
	\begin{align*}
	 &\frac{1}{N-1} \sum_{i=1}^N  \sum_{j\neq i}  |\widehat{\ba}_{ij}|^2   \Big\|\widehat{\IK}(\widehat\br_s^{i,j})-\IK(\br_s^{i,j})\Big\|_{\R^d}^2  
	\leq \frac{\p C_0^2}{N-1}  \sum_{i=1}^N \sum_{j\neq i} |\widehat{\ba}_{ij}|^2  \Big[ \|c-\widehat c \|^2 + \|\widehat c\|^2  \|\br_s^{i,j}-\widehat\br_s^{i,j} \|^2 \Big] \\
	&\qquad \leq \frac{\p C_0^2}{N-1}  \sum_{i=1}^N \sum_{j\neq i} \Big[ \left(|\widehat{\ba}_{ij}|^2 \|c-\widehat c \|^2  \right)  +  \|\widehat c\|_2^2   \|\br_s^{i,j}-\widehat\br_s^{i,j} \|^2\Big]
	\leq \frac{\p C_0^2}{N-1}  \Big[  N \|c-\widehat c \|^2 + 4N \|\widehat c\|_2^2 \| \widehat\bX_s -\bX_s \|_F^2 \Big],
	\end{align*}
using  $\|\widehat \ba_{i\cdot}\|^2= 1$ for each $i$ and $\sum_i\sum_j \|\br_s^{i,j}-\widehat\br_s^{i,j} \|^2 \leq 4N\| \widehat\bX_s -\bX_s \|_F^2$.  

	Consequently, plugging the above two estimates into \eqref{eq:traj_est} 
	we obtain a bound
	\begin{align*}
	 \| \widehat \ba \bB(	\widehat \bX_s) \widehat c - \ba \bB(
	\bX_s)  c\|_F^2 &\leq    \frac{{\p}NC_0^2}{N-1} \left[ \|c\|_{2}^2  \|\ba-\widehat\ba\|_F^2 +  \|c-\widehat c \|^2  + 4\|\widehat c\|_2^2  \| \widehat\bX_s -\bX_s \|_F^2\right]. 
	\end{align*}
	Combining the above inequality with \eqref{eq:Etraj}, we conclude that 
	\begin{align*}
		\E[\|\widehat{\bX}_t-\bX_t \|^2]&\leq  \frac{{\p}NC_0^2}{N-1} \left[ T^2 (\|c\|_{2}^2 \|\ba-\widehat\ba\|_F^2+ \|\widehat c-c\|_2^2 ) +2\|\widehat c\|_2^2 T \int_0^t \E\Big[ \|\widehat{\bX}_s-\bX_s \|^2 \Big] ds \right] \\
		& \leq C_1  \left[ T^2 (C_2 \|\ba-\widehat\ba\|_F^2+ \|\widehat c-c\|_2^2 ) +2C_2 T \int_0^t \E\Big[ \|\widehat{\bX}_s-\bX_s \|^2 \Big] ds \right] 
	\end{align*}
	with $C_1= 2\p C_0^2$ and $C_2=\|\widehat c\|_2^2 +\| c\|_2^2  $. 
	Finally, \eqref{Ineq:Traj_err} follows from Gronwall's inequality.
\end{proof}

\subsection{Connection with the classical coercivity condition}\label{sec:CC-relation}

We discuss the relation between the joint and the interaction kernel coercivity conditions in Definitions \ref{def:jointCC}--\ref{def:kernelCC} and the classical coercivity condition for homogeneous system see e.g., {\rm \cite[Definition 1.2]{LLMTZ21} or \cite[Definition 3.1]{LZTM19pnas}}.  
To make the connection, we consider only a homogeneous multi-agent system in the form 
$ d {X}^i_t =\frac{1}{N-1}\sum_{j\neq i} \IK(X^j_t-X^i_t) dt + \sigma d {W}^i_t, \quad i\in[N]$.  
Such a system has a weight matrix with all entries being the same; note that the normalizing factor is $N -1$, since each agent interacts with all other agents.
Suppose that the initial distribution of $(X^1_0,\ldots, X^N_0)$ 
 is exchangeable (i.e., the joint distributions of $\{X_0^i\}_{i\in \mathcal{I}}$ and  $\{X_0^i\}_{i\in \mathcal{I_\sigma}}$ are identical, where $\mathcal{I}$ and $\mathcal{I_\sigma}$ are two sets of indices with the same size). 
Note that the distribution of $\bX_t= (X_t^1,\ldots,X_t^N)$ is exchangeable for each $t\geq 0$. This exchangeability plays a key role in simplifying the coercivity conditions below. It leads to the following properties. 
\begin{enumerate}[label=$(\mathrm{P\arabic*})$]
	\item\label{homo_rho} The exploration measure $\rho_L$ in \eqref{Def:exploration} is the average of the distributions of $\{X_{t_l}^1-X_{t_l}^2\}$:
         $\rho_L(A) =\frac{1}{L}\sum_{l=1}^L \prob{X_{t_l}^1-X_{t_l}^2\in A}, \, \forall A\in \R^d$,  
         and it has a continuous density supported on a bounded set, denoted by $\mathrm{supp}(\rho)$.  
	\item\label{homo-expections} Let $\br_{ij}(t_l)= X_{t_l}^i-X_{t_l}^j $ for any $i \neq j$. Then, for each $t_l$, 
          \begin{align*}
          	\E[\big| \IK(\br_{ij}(t_l)) \big|^2] & =\E[\big| \IK(\br_{12}(t_l))\big|^2], \, \forall i \neq j; \\
          	\E[ \innerp{\IK(\mathbf{r}_{ij}(t_l)),\IK(\mathbf{r}_{ik}(t_l))}_{\R^{Nd}}] & = \E[ \innerp{\IK(\mathbf{r}_{12}),\IK(\mathbf{r}_{13})}_{\R^{Nd}}], \, \forall i \neq j, i\neq k, j\neq k. 
          \end{align*}  
\end{enumerate}

We first extend the classical coercivity condition, which was defined for radial interaction kernels in the form $\IK(x)= \phi(|x|)\frac{x}{|x|}$, to the case of general non-radial interaction kernels. The extension is a straightforward reformulation of the definitions in  {\rm \cite[Definition 1.2]{LLMTZ21}, with minor changes accounting for the normalizing factor $1/(N-1)$ and lack of radial symmetry of the kernel. 

\begin{definition}[Classical coercivity condition for homogeneous systems] \label{def:homo-CC}
The homogeneous system with $a_{ij}=1$ sastifies the coercivity condition on a set $\calH\subseteq L^2_{\rho_L}$, with $\rho_L$ the exploration measure in \eqref{Def:exploration}, if  there exists $c_\calH>0$ such that for all $\IK\in \calH$
\begin{equation}\label{eq:homo-CC}
	\frac{1}{N(N-1)^2} \sum_{i=1}^N \frac{1}{L}\sum_{l=1}^L \E\left[ \bigg|\sum_{j\neq i} \IK(\mathbf{r}_{ij}(t_l) )\bigg|^2 \right] \geq  c_{\mathcal H} \|\IK\|_{L^2_{\rho_L}}^2\,.
\end{equation}	
\end{definition}

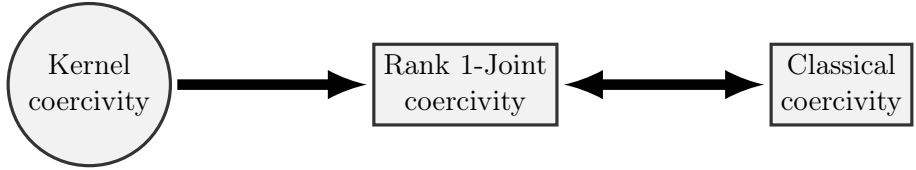
\begin{figure}[hbt]
\center
	\begin{tikzpicture}[roundnode/.style={circle, draw=black!80, fill=gray!10, very thick, minimum size=7mm},
		squarednode/.style={rectangle, draw=black!80, fill=gray!10, very thick, minimum size=5mm},]
	    \node[roundnode, align=center] (A) at (-5,0) {Kernel \\ coercivity};
	    \node[squarednode, align=center] (B) at (0,0) {Rank 1-Joint \\ coercivity};
	    \node[squarednode, align=center] (C) at (5,0) {Classical \\ coercivity};
	    
	    \draw[line width=4pt, shorten <=2pt, shorten >=2pt, -latex] (A) to [out=0,in=180] (B);
	    \draw[line width=4pt, shorten <=2pt, shorten >=2pt, latex-latex] (C) to [out=180,in=0] (B);
	\end{tikzpicture}  
	\caption{The relation between coercivity conditions for homogenous systems. }
	\label{Fig:Coer_relation}
\end{figure}

We show next that the three coercivity conditions (the joint and the interaction kernel coercivity conditions in Definitions \ref{def:jointCC}--\ref{def:kernelCC} and the above classical coercivity condition) are related as follows:
\begin{itemize}
\item The joint coercivity condition is equivalent to the classical coercivity.  
\item The kernel coercivity condition \eqref{eq:kernelCC} is stronger than the classical coercivity. It yields a suboptimal coercivity constant $\frac{c_{0,\calH}}{N-1}$ with $c_{0,\calH}\in (0,1)$ (see Proposition {\rm\ref{prop:Kernelcc-Jcc}}), which is smaller than $c_\calH= \frac 1 {N-1}$ for the classical coercivity. 
\end{itemize}
Without loss of generality, we set $L=1$ and drop the time index $t_l$ hereafter. Hence, we can write $\|\IK\|_{\rho_L}^2 = \E[|\IK(\mathbf{r}_{12})|^2]$. 
By Property \ref{homo-expections}, we can simplify Eq.\eqref{eq:homo-CC} in the above classical coercivity condition to 
$
 \frac{1}{(N-1)^2}\E\left[ \bigg|\sum_{j=2}^N \IK(\mathbf{r}_{1j})\bigg|^2 \right] \geq  c_{\mathcal H} \E[|\IK(\mathbf{r}_{12})|^2]
$.
This is exactly the joint coercivity condition after considering Property \ref{homo-expections}. Hence, the joint and the classical coercivity are equivalent for homogeneous systems with an exchangeable initial distribution.

On the other hand, the kernel coercivity \eqref{eq:kernelCC} is stronger than the classical coercivity. By Proposition {\rm\ref{prop:Kernelcc-Jcc}}, it yields a suboptimal coercivity constant $\frac{c_{0,\calH}}{N-1}$. This constant is smaller than the optimal constant $c_\calH= \frac 1 {N-1}$ in the classical coercivity condition in \cite{LLMTZ21}. 

Interestingly, while both the interaction kernel coercivity condition and the classical coercivity condition lead to the joint coercivity, they approach it from different directions. Specifically,  the classical coercivity condition seeks the infimum 
$ \inf_{\IK\in L^2_{\rho_L}, \|\IK\|_{L^2_{\rho_L}}=1}\E[ \innerp{\IK(\mathbf{r}_{12}),\IK(\mathbf{r}_{13})}] = 0$
to obtain $c_\calH= \frac{1}{N-1}$ as in \cite{LLMTZ21}. Under the assumption that $\br_{12}$ and $\br_{13}$ are independent conditional on $\calF^1$, which implies $\E[ \innerp{\IK(\mathbf{r}_{12}),\IK(\mathbf{r}_{13})}] = \E[ | \E[ \IK(\mathbf{r}_{12}) \mid \calF^1]|^2 ]$, the above infimum is equivalent to   
 $
\inf_{\IK\in L^2_{\rho_L}, \|\IK\|_{L^2_{\rho_L}}=1} \E[ | \E[ \IK(\mathbf{r}_{12}) \mid \calF^1]|^2 ] =0.
$ 

In contrast, the kernel coercivity, reducing to  $\E[\trCov( \IK(\mathbf{r}_{12}) \mid \calF^1)]  \geq c_{\mathcal H}^0 \E[|\IK(\mathbf{r}_{12})|^2]
$ 
after taking into account exchangeability, is equivalent to 
$
\inf_{\IK\in L^2_{\rho_L}, \|\IK\|_{L^2_{\rho_L}}=1}  \E[ | \E[ \IK(\mathbf{r}_{12}) \mid \calF^1]|^2 ] \leq (1- c_{\mathcal H}^0). 
$ 
Hence, the classical coercivity sets a lower bound for the term $\E[ | \E[ \IK(\mathbf{r}_{12}) \mid \calF^1]|^2 ] $, whereas the kernel coercivity sets an upper bound for this term so that the loss of dropping this terms (in \eqref{eq:kernelCC-drop}) is controlled. In general, it is easier to prove the upper bound than the lower bound.

\section{Details and additional numerical results} \label{sec:addi-num}

\subsection{Computational costs} \label{s:emp_comp_perforamance}
The detailed breakdown of the computational costs, leading to the overall costs in table \ref{tab:complexity}, is as follows. 
For both algorithms, the data processing involves $MLdN^2{\p}$ flops on evaluating $\{\psi_k(X_{t_l}^{j,m}-X_{t_l}^{i,m} ),1\leq i,j\leq{N}\}_{1\leq k\leq{N}}^{1\leq m\leq M}$, where these computations can be done in parallel in $M,L$ or $N$.  
\begin{figure}[thb]\vspace{-6mm}
\centering
	\includegraphics[width=0.4\textwidth]{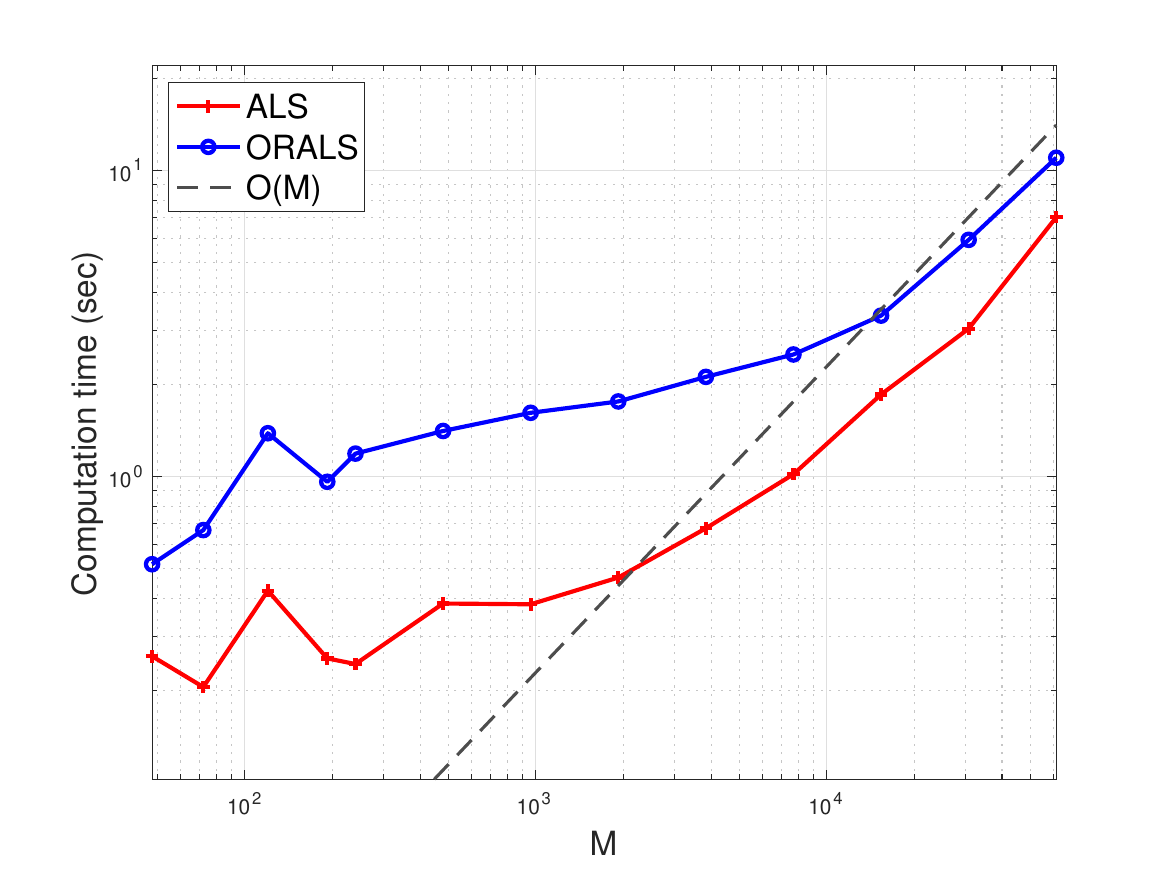}
	\includegraphics[width=0.4\textwidth]{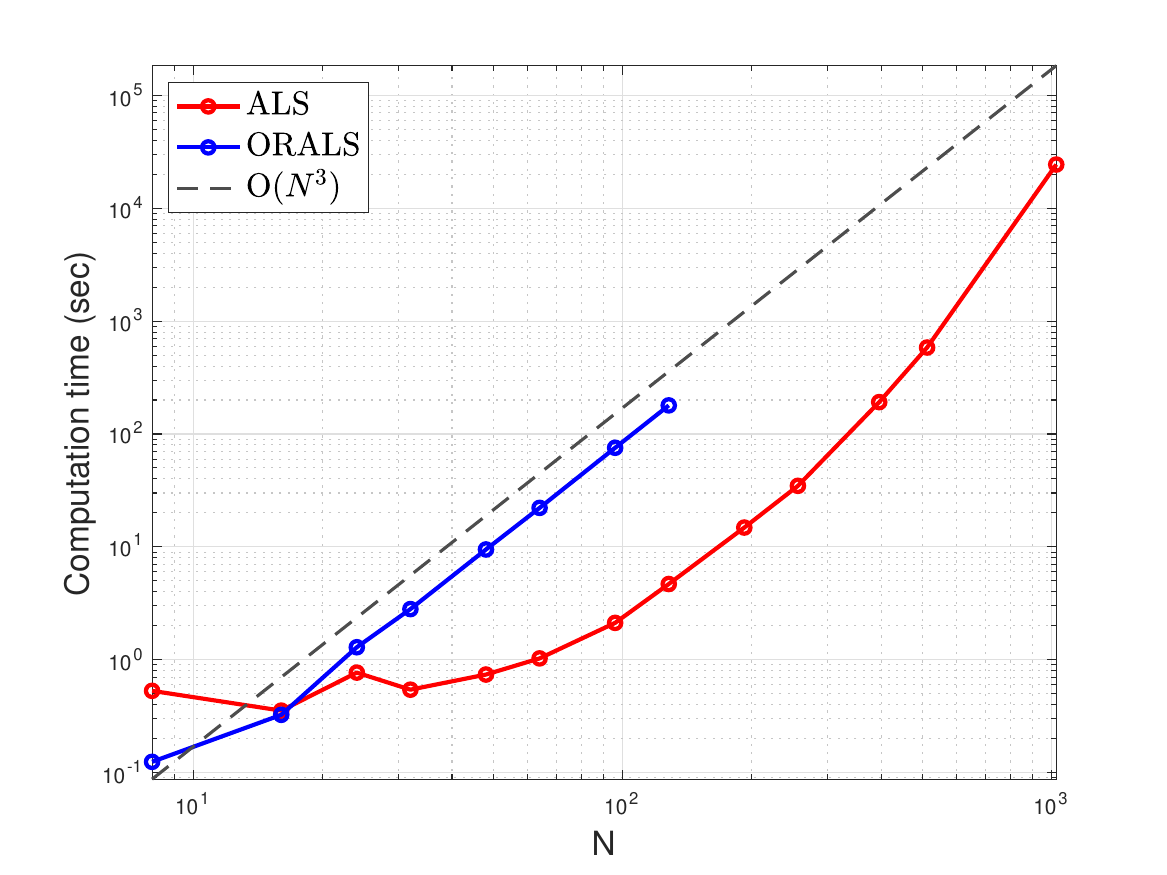}\vspace{-2mm}
	\caption{Computation time for the construction of the ALS and ORALS estimators, as a function of $M$ (left) and of $N$ (right). In both plots, the other parameters are set as: $L=2$, $d=1$, $n=8$; in the first plot $N=16$, and in the second plot $M=1024$; the interaction kernel is the inverse Fourier transform of a random vector with decaying coefficients, and no regularization is imposed. The scaling of ORALS as $N^3$ in the figure on the right, instead of the expected $N^4$, as the term $MLdN^3$ overcomes $N^4$ for the values of the parameters we have here; we could not perform runs with larger $N$ than what is represented in the figure due to excessive memory requirements. Tests are run on a machine with 2 processors with 12 cores each, and $448$GB of RAM.}
	\label{fig:compute_test}\vspace{-5mm}
\end{figure}

The ALS computation consists of two additional parts: solving the least square problems to estimate $\mathbf{a}$ and $\IK$, and iterating. 
In each iteration, when solving the least squares for the rows of the weight matrix via the $MLd\times N$ matrices, it takes $O(MLdNpN_{par})$ flops to assemble the regression matrices and $O((MLdN^2\wedge (MLd)^2N)N_{par})$ flops to solve the least squares problems; when solving the coefficient $c$ via the $MLdN^2\times \p$ matrix, it takes $O(MLdNN_{par}\p)$ flops to assemble the regression matrix and $O(MLdN^2\p^2\wedge(MLdN^2)^2\p)$ flops to solve the least squares problem.  Here $N _{par}$ means that the computation can be done trivially in parallel. Lastly, the number of iterations is often below, say, $20$, independent of $M,N,{\p}$, but we do not have any theoretical guarantees for this phenomenon. The total computational cost of ALS is $O(MLdN^2 (N_{par}+{\p}^2))$ flops, in the natural regime $M\ge N^2+\p$. 

The ORALS computation consists of three parts: data extraction, solving the least squares, and matrix factorization. The data extraction involves $MLdN^2\p$ flops, and the matrix factorization for the $Z_i$'s takes a negligible cost of $O((N^2+{\p^2})N_{par})$ flops.   The major cost takes place in solving the least squares. 
The long-matrix approach takes about $O(MLd(N\p)^2N_{par})$ flops to solve all the $Z_i$'s, in which assembling the $MLd\times Np$ regression matrix does not take extra time since it is simply reading the extracted array. 
The normal equation approach would require $O((MLN)_{par}dN^2\p^2+(N\p)^3N_{par})$ flops, which consists of $O((MLN)_{par}dN^2p^2)$ flops to assemble the normal matrices and $O((Np)^3)N_{par})$ flops to solve the equations. 
Therefore, the total computational cost for ORALS is of order $O(MLd(N\p)^2N_{par})$ for the long-matrix approach and $O((MLN)_{par}dN^2\p^2+(N\p)^3N_{par})$ for the normal matrix approach. 
When $ML>N^2+\p$, the normal equation approach is more efficient since the computation can be in parallel in $ML$.

We corroborate the computational complexity of ALS and ORALS discussed in section \ref{ss:compcomplexity} and reported in table \ref{tab:complexity} with the measurements in wall-clock runtime reported in figure \ref{fig:compute_test}.

\subsection{Regularization}\label{sec:regu_append}
Regularization  is helpful to produce stable solutions when the matrix in the least squares of ALS or ORALS is ill-conditioned, and the data is noisy. We have tested five methods to solve the ill-posed linear equations: direct backslash (denoted by ``NONE''), pseudo-inverse, minimal norm least squares (denoted by ``lsqmininorm''), the Tikhonov regularization with Euclidean norm (denoted by ``ID''), and the data-adaptive RKHS Tikhonov regularization (denoted by ``RKHS''). 

The data-adaptive RKHS Tikhonov regularization uses the norm of an RKHS adaptive to data and the basis functions of the kernel. In estimating the kernel coefficients in ALS, in addition to the regression matrix and vector, it uses the basis matrix $B$ with entries  
\begin{equation}\label{eq:B_psi}
B_\psi = \frac{1}{ (N-1)NLM}\sum_{l=0}^{L-1}\sum_{m=1}^{M} \sum_{j\neq i}\innerp{\psi_k(\br_{ij,t_l}^m),\psi_l(\br_{ij,t_l}^m)}_{\R^d}, \quad \br_{ij,t}= X^j_t-X^i_t, 
\end{equation}
where $\{\psi_k\}_{k=1}^{\p}$ are the basis functions in the parametric form and recall that $\sum_{j\neq i}:= \sum_{i=1}^N \sum_{j=1,j\neq i}^N$. In ORALS for the estimation of $\widehat z_{i,M}$ in \eqref{eq:Ai_operator}, we supply the DARTR with basis matrix $I_N\otimes B_\psi\in \R^{N\p\times N{\p}}$ with $B_\psi$ in \eqref{eq:B_psi}, where $\otimes$ denotes the Kronecker product of matrices.

\begin{figure}[thb]
	\centering
	\includegraphics[width=0.4\textwidth]{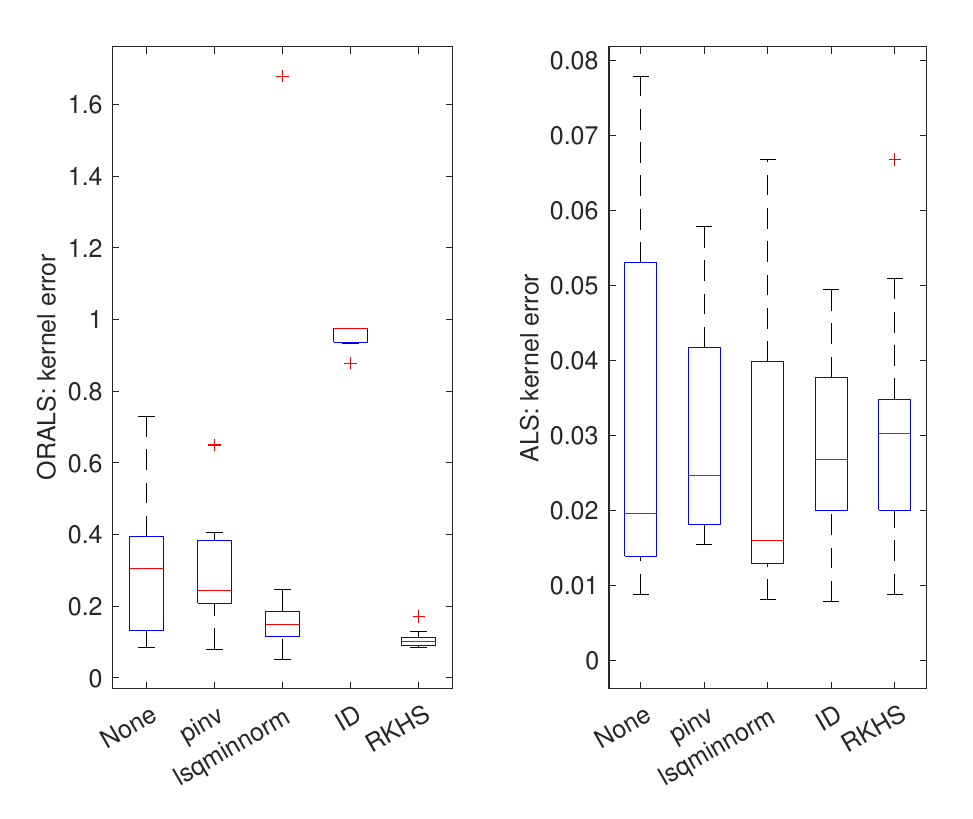}
	\includegraphics[width=0.4\textwidth]{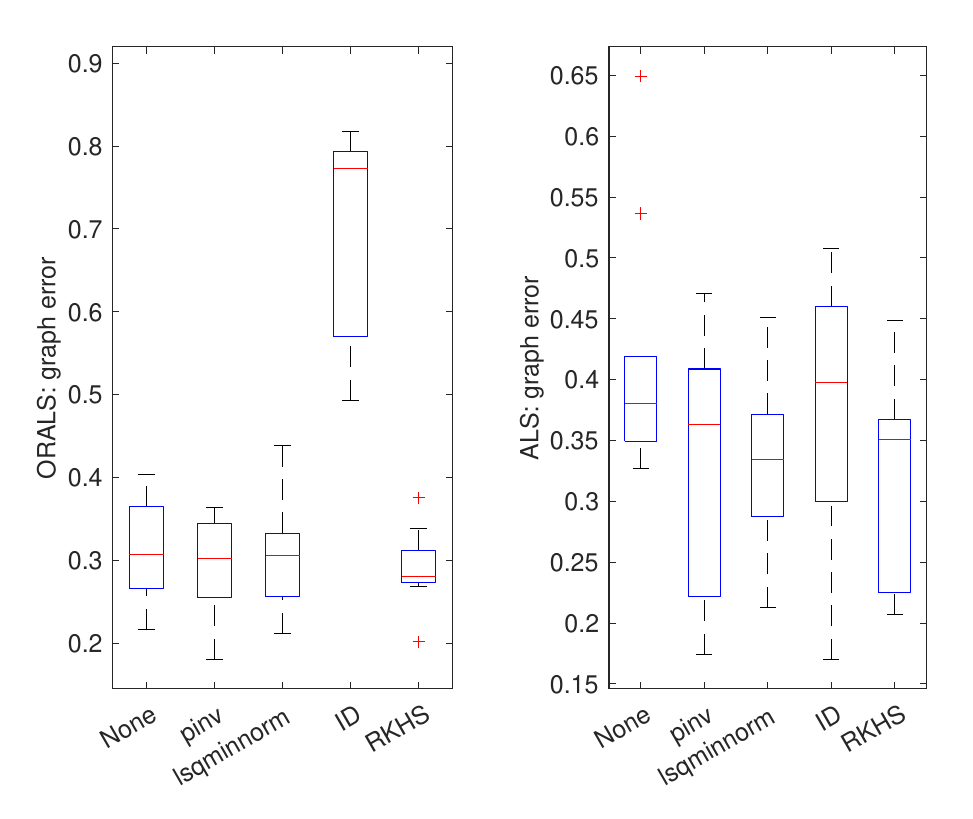}
	\caption{Errors of estimators in 10 simulations for different regularization methods. Here the regression matrices are deficient-ranked due to a small sample size $M=64$. The other parameters are $N=20$, $\p=3$, $L=5$  and $d=2$.}
	\label{fig:regu}\vspace{-3mm}
\end{figure} 
Figure \ref{fig:regu} shows the errors of regularized estimators in 10 simulations of the Lennard-Jones model. The model parameters are $N=20$, $\p=3$, $L=5$  and $d=2$. Here, the sample size $M=64$ is relatively small, so the regression matrices in ORALS tend to be deficient-ranked; in contrast, the regression matrices in ALS are well-conditioned. The results show that the minimal norm least squares and DARTR lead to more robust and accurate estimators than the other methods for the ORALS, but all methods perform similarly for ALS. Additional numerical tests show that as the sample size increases, the regression matrices for both ORALS and ALS become well-posed, and the direct backslash and the pseudo-inversion lead to accurate solutions robust to noise.

In short, regularization is helpful when the regression matrices are ill-conditioned and the data is noisy; otherwise, either the direct backslash operator or the pseudo-inversion is adequate. In the parametric estimation of the kernel, the regression matrices are often well-conditioned. However, in nonparametric estimation, the regression matrices are often ill-conditioned and even rank-deficient in the process of selecting an optimal dimension for the hypothesis space to achieve the optimal bias-variance tradeoff. 

Another type of regularization, different from those above that regularize the least squares in ALS or ORALS, 
is to enforce the low-rank property. Such regularizers include minimizing the nuclear norm \cite{RFP2010} or adding a term maintaining the norm-preserving property of the Hessian of the loss function \cite{GJZ2017}. They could be beneficial to the operation regression stage of the ORALS algorithm. We leave further exploration of these regularizers for future work.

\subsection{Dependence on noise level and stochastic force}\label{sec:decay_force_noise}
To examine robustness to stochastic force and observation noise, we test the estimation error(s) as a function of the stochastic force and the noise level.  
Figure \ref{fig:conv_vis} shows that for both ALS and ORALS estimators, the error decays linearly in the stochastic force level $\sigma$, with the reported mean and std computed over 100 simulations. In each simulation, we set the observation noise $\sigma_{obs} = 10^{-7}$, and the sample size is $M = 1000$. To see the effects of the stochastic force, we use long trajectories with time length $T = 100$.
Figure \ref{fig:conv_vis} also shows that for both ALS and ORALS estimators, the estimation error(s) decay linearly in the noise level $\sigma_{obs}$. In each of 100 simulations, we take $\sigma=0$, $M = 1000$, and $T = 1$. 
\begin{figure}[thb] \vspace{-4mm}
\centering
 \includegraphics[width=.9\textwidth]{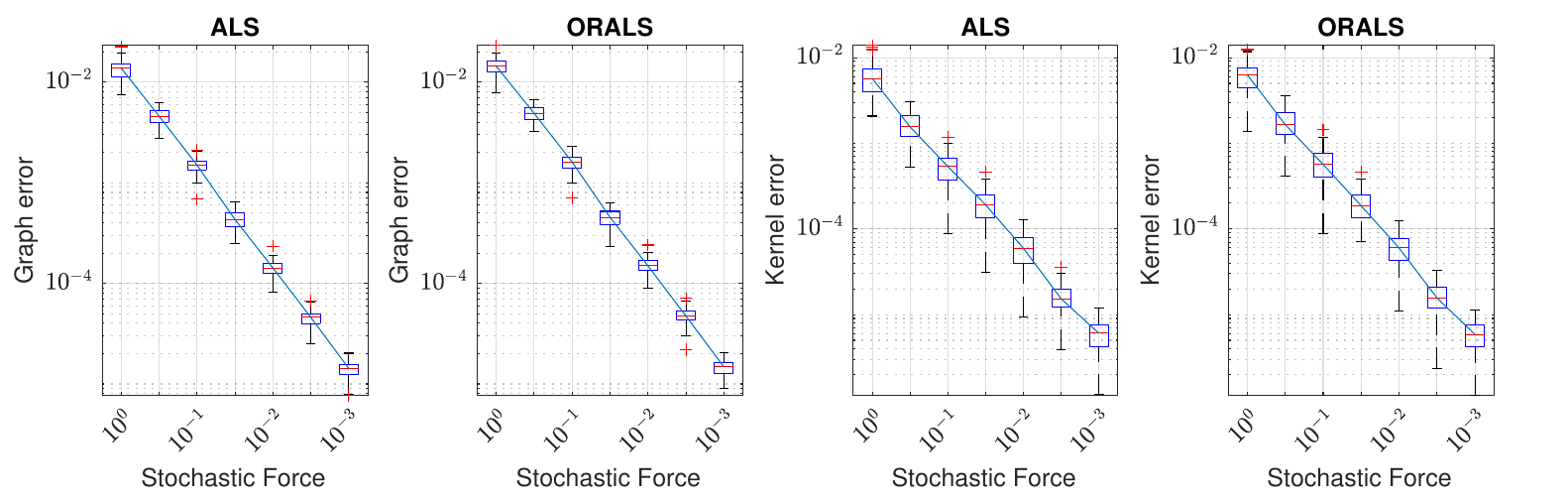}
 \includegraphics[width=0.9\textwidth]{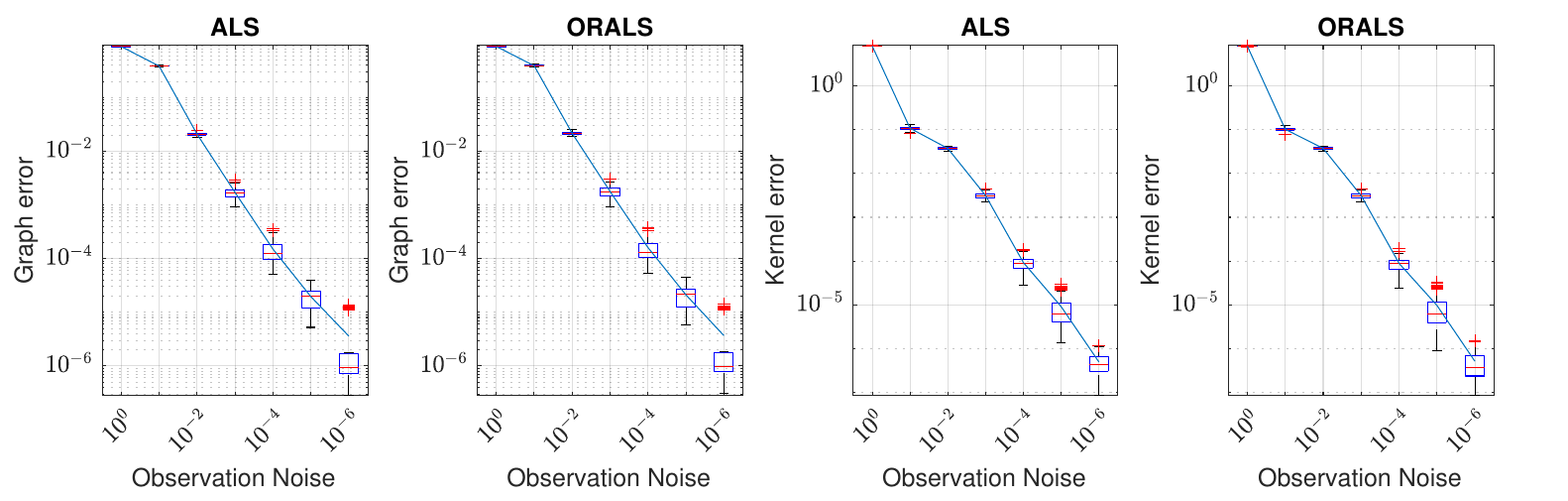}
\caption{Top: Decay of the estimation errors as the stochastic force decreases; bottom: decay of the estimation errors as the observation noise level decreases.}
\label{fig:conv_vis}\vspace{-6mm}
\end{figure}

\subsection{Additional tests on a directed graph on a circle}

We also provide an example with a very simple graph in our admissible set $\mathcal{M}$, i.e., a directed circle graph.  We present the graph, kernel estimation, and true trajectory in Figure \ref{fig:traj_kernel_graph_short}; the rest of the results are very similar to the previous settings and are hence omitted. 

\begin{figure}[thb] 
	\centering
	\includegraphics[width=\textwidth]{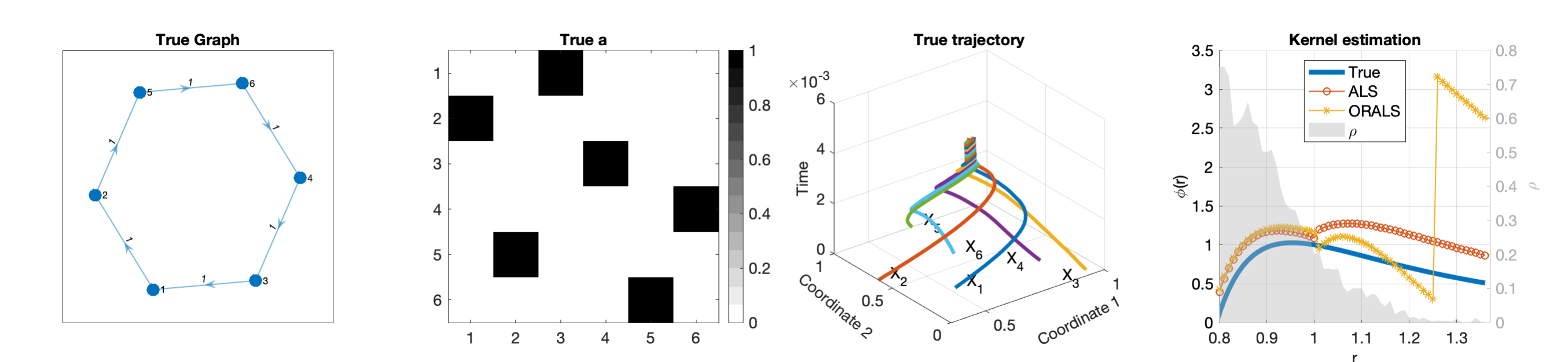}
	\caption{ An example of the simplest graph, the true trajectory, and the kernel estimators. Each particle only follows one other particle, forming a spiral dynamical behavior.  There is limited data of $r\in[1.2, 1.4]$ since particles quickly converge together, leading to small values of $\rho$ in the region. As a result, the kernel estimators have large errors in the region; yet, their overall $L^2_\rho$ errors remain small with $\varepsilon_K$ for ALS is $1.19\times 10^{-2}$ and for ORALS is $1.89\times 10^{-2}$. 
}
	\label{fig:traj_kernel_graph_short}
\end{figure}

\subsection{Additional details for identifying the leader-follower model}\label{Append:addition}

 We examine a Leader-follower system where leaders have significant impacts on others; see the left panel of Figure \ref{Fig:LF_cluster}. In the Impact-Influence coordinate, as shown in the middle panel of Figure \ref{Fig:LF_cluster}, one can observe that the leaders $A1$ and $A6$ stand out from the rest. As the sample size $M$ increases, both the estimated graph $\widehat{\ba}$ and the estimated interaction kernel $\widehat{\IK}$ in the top of \eqref{Fig:leader-follower} become more precise. It becomes evident that a more accurate estimator $\widehat{\ba}$ contributes to more precise identifications of leaders and their followers.  The Leader-follower network estimated with $M=100$ almost recovers the true network (the left one in Figure \ref{Fig:LF_cluster}). 
 Thus, the clustering result of $M=100$ shown in the last row of the right panel in Figure \ref{Fig:LF_cluster} aligns with the ground truth depicted in the first row. Nevertheless, it's noteworthy that identifying leaders and properly classifying followers remain feasible even when the estimator $\widehat{\ba}$ is not highly precise. 

\begin{figure}[thb] 
	\centering
	\includegraphics[width=0.275\textwidth]{figures/leader_follower_true.pdf}
	\includegraphics[width=0.338\textwidth]{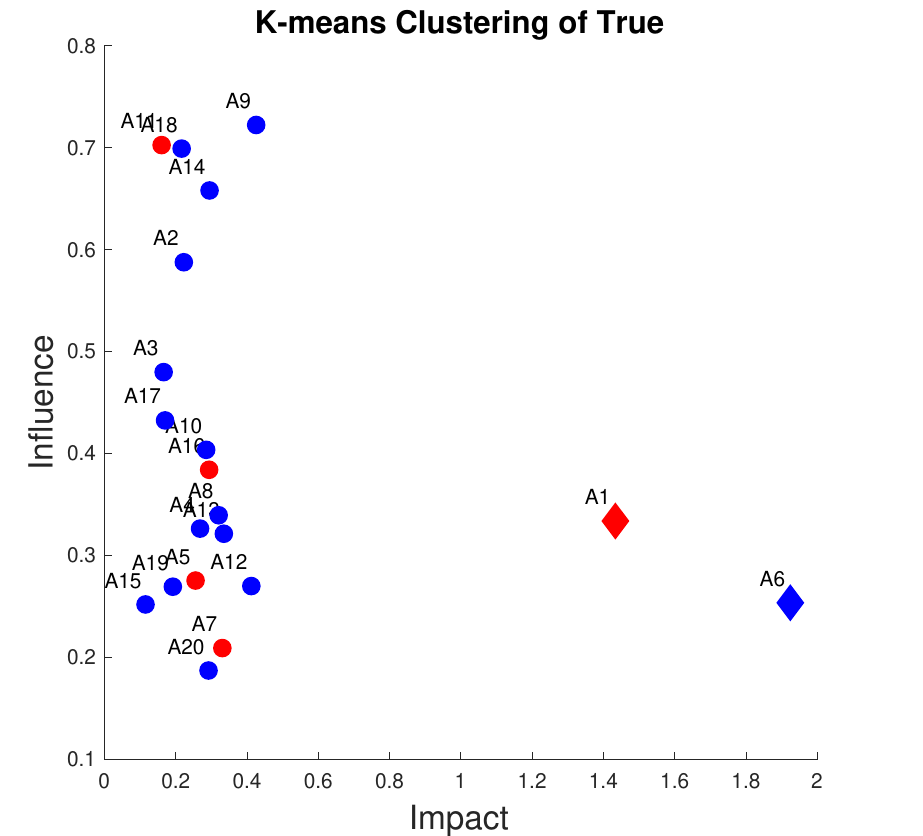}
	\includegraphics[width=0.36\textwidth]{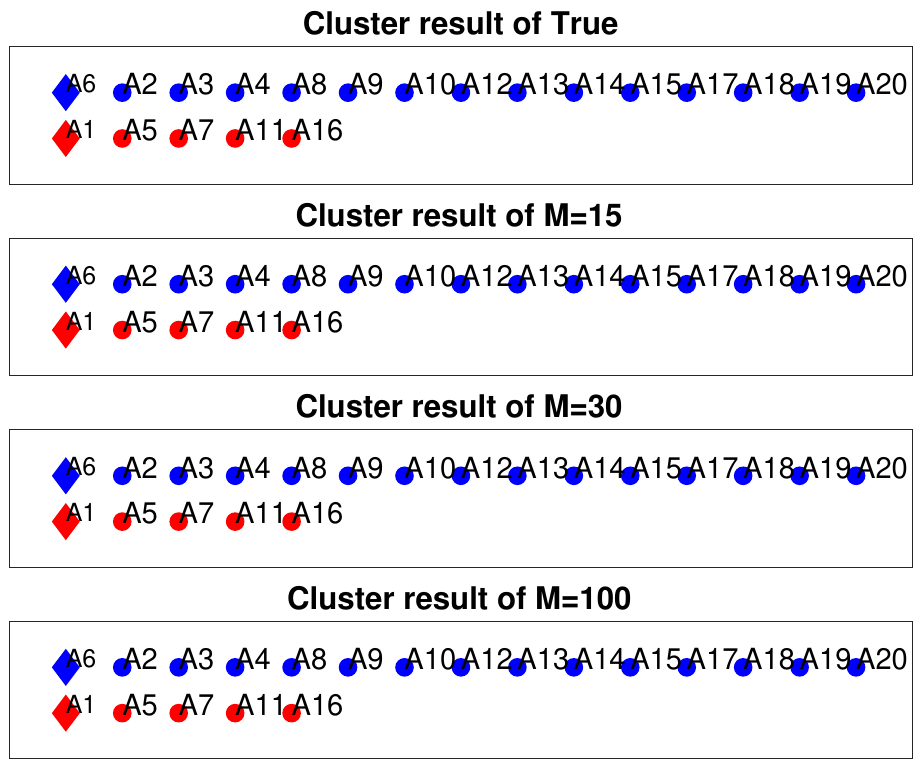}
	\caption{Left: the true Leader-follower network; Middle: the clustering of the true system, with two groups led by $A1$ (red group) and $A6$ (blue group); Right: the results of clustering based on the estimates with sample sizes  $M\in \{15,30,100\}$. The graph errors (in Frobenius norm) are $0.1254$, $9.8\times 10^{-3}$, $1.8\times 10^{-3}$; and the kernel errors are $0.0115$, $1.4\times 10^{-3}$, $3\times 10^{-4}$.}
	\label{Fig:LF_cluster}
\end{figure}

\section{Connection with matrix sensing and RIP}\label{Appen_RIP}
In this section, we connect our joint inference problem with matrix sensing (see \cite{GJZ2017,ZSL2017,RFP2010} for example) and study the restricted isometry property (RIP) of the joint inference. 

\paragraph{Matrix sensing and RIP}
The matrix sensing problem aims to find a low-rank matrix $Z^*\in \R^{n_1\times n_2}$ from data $b_m=\langle A_m,Z^*\rangle_{F}$, where $A_1,\cdots,A_M\in \R^{n_1\times n_2}$ are sensing matrices. 
To find $Z^*$ with rank $r\ll n_1 \wedge n_2$, one solves the following non-convex optimization problem
\begin{equation}\label{Mat_Sens}
	\min_{Z\in\R^{n_1\times n_2},\text{rank}(Z)=r} F(Z) = \frac{1}{M}\sum_{m=1}^M |\langle A_m,Z\rangle_{F}-b_m|^2=\frac{1}{M}\sum_{m=1}^M |\Tr(A_m^\top Z)-b_m|^2\,.
\end{equation}
The constrained optimization problem \eqref{Mat_Sens} is NP-hard. A common method of factorization is introduced by \cite{BM2003,BM2005} to treat \eqref{Mat_Sens}. Namely, we express $Z=UV^\top$ where $U\in\R^{n_1\times r}$ and $V\in\R^{n_2\times r}$. Then \eqref{Mat_Sens} can be transformed to the unconstraint problem
\begin{align}\label{Mat_Factor}
	\min_{U\in\R^{n_1\times r},V\in\R^{n_2\times r}} F(U,V) &= \frac{1}{M}\sum_{m=1}^M |\langle A_m,UV^\top\rangle_{F}-b_m|^2 \\
	&=\frac{1}{M}\sum_{m=1}^M |\Tr(A_m^\top UV^\top)-b_m|^2=\frac{1}{M}\sum_{m=1}^M |\Tr(U^\top A_m V)-b_m|^2\,. \nonumber
\end{align}
To simplify the notations, let us define a linear sensing operator $\sA :\R^{n_1\times n_2}\to \R^M$ by 
\begin{equation}\label{Def:SensOp}
	\sA(Z)= M^{-1/2} \left(\langle A_1,Z\rangle_{F},\cdots,\langle A_M,Z\rangle_{F}\right)\,.
\end{equation}

\begin{definition}[Restricted isometry property (RIP)]\label{Def:RIP} 
The linear map $\sA$  satisfy the $(r,\delta_{r})$-RIP condition with the RIP constant $\delta=\delta_{r}\in[0,1)$ if there is a (strictly) positive constant $C$  
\begin{equation}\label{Cond:RIP}
	(1-\delta) \norm{Z}_F^2 \leq \frac{1}{C} \norm{\sA (Z)}_2^2=\frac{1}{CM}\sum_{m=1}^M \langle A_m,Z\rangle_{F}^2 \leq (1+\delta) \norm{Z}_F^2
\end{equation}
holds for all $Z$ with rank at most $r$. We also simply say that $\sA$ satisfies the rank-$r$ RIP condition without specifying the RIP constant $\delta_{r}\in[0,1)$.
\end{definition}
\begin{remark}\label{Rmk:RIP_discuss}
The normalization factor $C$ in the condition \eqref{Cond:RIP} was introduced in \cite{RohdeTsybakov2011}. It enables the application of RIP to a larger class of sensing matrices that can be scaled to a near-isometry. In particular, in our setting, the sensing matrices are mostly far from an isometry, even up to a rescaling.  
\end{remark}
The restricted isometry property and the restricted isometry constant are powerful tools in the theory of matrix sensing \cite{RFP2010}, a generalization of compressed sensing \cite{CandesTao2005}. For example, they can characterize the identifiability of matrix sensing problems.
\begin{theorem}[Theorem 3.2 in \cite{RFP2010}]\label{Thm:RFP2010}
	Suppose that $\delta_{2r}<1$ for some integer $r\geq 1$, i.e., $\sA$ satisfy the rank-$2r$ RIP condition for $r\geq 1$. Then $Z^*$ is the only matrix of rank at most $r$ satisfying $\sA(Z)=b=[b_1,\cdots,b_M]^\top$.
\end{theorem}

This article establishes a connection between the rank-1 and rank-2 Restricted Isometry Property (RIP) conditions and their counterparts in joint coercivity conditions.

\paragraph{Joint inference of $\ba$ and $c$ as matrix sensing problems}
In our setting \eqref{eq:ips_K}, the estimator $(\widehat{\ba},  \widehat{c} )$ is the minimizer of the following loss function 
\begin{align}
\calE_{L, M}(\ba,c) &= \frac{1}{MT}\sum_{l=0, m = 1}^{L-1,M}\norm{\Delta\bX_{t_l}^m -\ba \bB (\bX_{t_l}^m) c\Delta t}_F^2 = \sum_{i=1}^N \calE^{(i)}_{L, M}(\ba_{i\cdot},c)\,, \nonumber \\
\text{where}\quad \calE^{(i)}_{L, M}(\ba_{i\cdot},c)&= \frac{1}{MT}\sum_{l=0, m = 1}^{L-1,M} \norm{(\Delta\bX_{t_l}^m)_i -\ba_{i\cdot} \bB (\bX_{t_l}^m)_i c\Delta t}_F^2\,.\label{Def:LossFun_ac}
\end{align}
If we minimize the loss function $\calE_{L, M}(\ba,c)$ row by row, i.e. by minimizing the loss functions $\calE^{(i)}_{L, M}(\ba_{i\cdot},c)$ for $i=1$ to $N$, each minimization is a rank-one matrix sensing problems \eqref{Mat_Factor} by substituting $U=\ba_{i\cdot}$, $V=c$, $b_m=(\Delta\bX_{t_l}^m)_i$ and $A_m=\bB (\bX_{t_l}^m)_i$ for each row $i$. To illustrate the idea, consider $d=1$, $L=1$, and $\Delta t=1$. Thus we set the rank-one decomposition $Z=UV^\top$ where $U=\ba_{i\cdot}\in \R^{N-1}$ and $V=c\in\R^{\p}$. Also, we define the sensing operator $\sA :\R^{(N-1)\times \p }\to \R^M$ with the sensing matrices 
\begin{align}\label{Def:SensMat_IPS}
	A_m = \bB (\bX_{t_0}^m)_i =\Big[ \psi_k\big(X_{t_0}^{j,m}-X_{t_0}^{i,m}\big) \Big]_{j \neq i \atop 1\leq k\leq \p}\,, \quad m=1,\cdots,M\,,	
\end{align}
where $\bX_{t_0}=(X_{t_0}^{1},\cdots,X_{t_0}^{N})$ is the initial condition and $\{\psi_k\}_{k=1}^p$ is the set of basis functions. Therefore, 
\begin{align*}
	\ba_{i\cdot} \bB (\bX_{t_0}^m)_i c = \langle A_m,Z\rangle_{F}= \Tr(A_m^\top UV^\top)\,. 
\end{align*}
In Section \ref{sec:multi_type}, the model \eqref{eq:ips_K} has $Q$ types of interaction kernels. We shall take $Q=2$ as an example to explain the connection with higher-rank matrix sensing problems. Namely, $\kappa(i)=1$ or $2$ indicating the type of kernel for agent $i$, and the coefficients are 
\begin{equation*}
	c_{ki}=c_{k}^{\kappa(i)}:=
	\begin{cases}
		c_{k}^{(1)}\,, & \kappa(i)=1\,;\\
		c_{k}^{(2)}\,, & \kappa(i)=2\,.
	\end{cases}
\end{equation*}
Without loss of generality, we still set  $d=1$, $L=1$, and consider $i$-th row. Thus, the interacting part in the system \eqref{eq:ipog_multitype} can be rewritten to be
\begin{align}\label{eq:ipog_multitype_i=1}
	\ba_{i\cdot}\bB(\bX_t^m)_i \bc_{\cdot i}  &= \sum_{j\neq i} \ba_{ij}\sum_{k=1}^{\p} \psi_k(X^{j,m}_t-X^{j,m}_t)c_{k}^{\kappa(i)} \nonumber \\
	&= \sum_{j\neq i} \ba_{ij}^{(1)}\sum_{k=1}^{\p} \psi_k(X^{j,m}_t-X^{i,m}_t)c_{k}^{(1)} + \sum_{j=2}^N \ba_{ij}^{(2)}\sum_{k=1}^{\p} \psi_k(X^{j,m}_t-X^{i,m}_t)c_{k}^{(2)}, 
\end{align}
where $\ba_{ij}^{(1)}=\ba_{ij}$ if $\kappa(i)=1$, $\ba_{ij}^{(1)}=0$ if $\kappa(i)=2$ and $\ba_{ij}^{(2)}$ is defined similarly. So, selecting a rank-two decomposition $Z=UV^\top$ with
$U=[\ba_{i\cdot}^{(1)},\ba_{i\cdot}^{(2)}]\in \R^{(N-1)\times 2}$ and $V=[c^{(1)},c^{(2)}]\in \R^{p\times 2}$,
\eqref{eq:ipog_multitype_i=1} can be expressed as
\begin{align*}
	\ba_{i\cdot}\bB(\bX_t^m)_i \bc_{\cdot i}=\langle A_m,Z\rangle_{F}= \Tr(A_m^\top UV^\top)\,. 
\end{align*}
Here, $A_m$ is the same sensing matrix defined in \eqref{Def:SensAm}. Also, for another multitype kernel model where the type of interacting kernel depends on agent $j$
\begin{align*}
	\ba_{i\cdot}\bB(\bX_t^m)_i \tilde{\bc}_{\cdot j}  &= \sum_{j\neq i} \ba_{ij}\sum_{k=1}^{\p} \psi_k(X^{j,m}_t-X^{j,m}_t)\tilde{c}_{k}^{\kappa(j)}\,,
\end{align*}
we have the same expression with $\ba_{ij}^{(1)}$ and  $\ba_{ij}^{(2)}$ adapted accordingly.

In the classical matrix sensing problem (refer to, for example, \cite{BM2003,RFP2010,LeeStoger2022}), the entries of the sensing matrix are i.i.d. standard Gaussian random variables. 
In our case, the entries of $A_m$ in \eqref{Def:SensMat_IPS} exhibit high correlations. This prevents us from employing the ``leave-one-out'' tool, as successfully applied in \cite{LeeStoger2022,Chandrasekher2022alternating}, to prove the convergence of the alternating least square algorithm.

\paragraph{RIP and joint coercivity conditions}
The lower bound in the RIP condition is closely related to the joint coercivity conditions in Definition \ref{def:jointCC}. In the following, we illustrate that the rank-1 and rank-2 RIP conditions lead to the rank-1 and rank-2 joint coercivity conditions, respectively, when  $\calH$ is finite-dimensional.
\begin{proposition}\label{Prop:RIPtoJCC}
	Let $\calH=\textrm{span} \{\psi_k\}_{k=1}^p$ with $\{\psi_k\}_{k=1}^p$ being orthonormal in $L^2_{\rho_L}$ for $p\geq 1$. Let  $\sA_i:\R^{(N-1)\times \p }\to \R^M$ be (row-wise) linear sensing operators in \eqref{Def:SensOp}  with sensing matrices in \eqref{Def:SensMat_IPS}. 
	Let $r\in \{1,2\}$. Suppose $\sA_i$ satisfies the rank-r RIP condition with a constant $\delta$ for all $i\in[N]$ uniform for all $M\to \infty$. Then, the rank-r joint coercivity condition holds.  
\end{proposition}
\begin{proof}
	Without loss of generality, we set $i=1$ and $L=1$ and abbreviate $\sA_1$ as $\sA$. We consider the rank-1 case first. For all rank-1 matrices $Z=uv^\top$, it is equivalent to consider any $u=\ba_{1\cdot}\in \mathcal{M}$ (defined in \eqref{eq:aMat_set}) and any $v=c\in \R^p$.  Then, substituting \eqref{Def:SensMat_IPS} into \eqref{Cond:RIP} and sending $M$ to infinity, we  get the lower bound by the Law of Large Numbers:
\begin{align*}
	\norm{\sA (Z)}_2^2 &= \E \bigg[\bigg|  \sum_{j=2}^N \ba_{1j} \IK(X_{t_0}^{j+1}-X_{t_0}^{1})\bigg|^2 \bigg]
	\geq C(1-\delta) \norm{Z}_F^2 = C(1-\delta) |\ba_{1\cdot}|^2  \|c\|^2
	= C(1-\delta) |\ba_{1\cdot}|^2  \|\IK\|_{\rho}^2
\end{align*}
for any $\IK=\sum_{k} c_k \psi_k\in \calH$. Thus,  the coercivity constant in \eqref{eq:jointCC} is $c_{\calH}=C(1-\delta)$ for a finite-dimensional hypothesis space, where $C$ is the normalization constant in the RIP condition when the kernel is represented on an orthonormal basis.

Next, we consider the rank-2 case. Recall that lower bound in rank-2 RIP condition implies that $\norm{\sA (Z)}_2^2 \geq C(1-\delta) \norm{Z}_F^2$ for all matrices with rank equal or less than two, i.e., $Z=u_1 v_1^\top+u_2 v_2^\top$ for all $u_1,u_2\in \R^{N-1}$ and $v_1,v_2\in\R^p$. We aim to show that 
\begin{align}\label{Ineq:RIP_discuss}
	\E \bigg[\bigg|  \sum_{j=2}^N [ \ba_{1j}^{(1)} \IK_1(X_{t_0}^{j+1}-X_{t_0}^{1})+\ba_{ij}^{(2)} \IK_2(X_{t_0}^{j+1}-X_{t_0}^{1})] \bigg|^2 \bigg] \geq c_\calH [|\ba_{i\cdot}^{(1)}|^2  \|\IK_1\|_{\rho_L}^2+ |\ba_{i\cdot}^{(2)}|^2  \|\IK_2\|_{\rho_L}^2]
\end{align}
with $c_{\calH}=C(1-\delta)$ for all $\IK_1,\IK_2\in\calH$ being orthogonal and for all weight matrices $\ba^{(1)},\ba^{(2)}\in \mathcal{M}$. For any $\IK_1=\sum_{k} c_{1,k} \psi_k\in \calH$ and $\IK_2=\sum_{k} c_{2,k} \psi_k\in \calH$ being orthogonal to each other, we have $c_1\perp c_2$ and $\norm{\ba_{1\cdot}^{(1)} c_1^\top +\ba_{1\cdot}^{(2)} c_2^\top}_F^2= |\ba_{1\cdot}^{(1)}|^2  |c_1|^2+|\ba_{1\cdot}^{(2)}|^2  |c_2|^2 $. Thus, with $u_1=\ba_{1\cdot}^{(1)}$, $u_2=\ba_{1\cdot}^{(2)}$ and $v_1=c_1$, $v_2=c_2$, the lower bound of rank-2 RIP amounts to
\begin{align*}
	\norm{\sA (Z)}_2^2 &= \E \bigg[\bigg|  \sum_{j=2}^N [ \ba_{1j}^{(1)} \IK_1(X_{t_0}^{j+1}-X_{t_0}^{1})+\ba_{ij}^{(2)} \IK_2(X_{t_0}^{j+1}-X_{t_0}^{1})] \bigg|^2 \bigg] \nonumber \\
	&\geq C(1-\delta) \norm{Z}_F^2= C(1-\delta)\norm{\ba_{1\cdot}^{(1)} c_1^\top +\ba_{1\cdot}^{(2)} c_2^\top}_F^2 \\
	&= C(1-\delta)[|\ba_{1\cdot}^{(1)}|^2  |c_1|^2+|\ba_{1\cdot}^{(2)}|^2  |c_2|^2] \nonumber 
	=C(1-\delta)[|\ba_{i\cdot}^{(1)}|^2  \|\IK_1\|_{\rho_L}^2+ |\ba_{i\cdot}^{(2)}|^2  \|\IK_2\|_{\rho_L}^2]\,. 
\end{align*}
So, we get \eqref{Ineq:RIP_discuss} and finish the proof.
\end{proof}

\paragraph{Large RIP constants and local minima in our setting}
The RIP constant $\delta$ plays a crucial role in characterizing the presence of spurious local minima and the convergence of search algorithms; see for example, \cite{BahmaniRomberg17a,GJZ2017,LeeStoger2022,Chandrasekher2022alternating}.  Notably, when the rank $r=1$ and in the symmetric setting $U=V$ in equation \eqref{Mat_Factor}, a precise RIP threshold of $\delta=\frac{1}{2}$ serves to establish both necessary and sufficient conditions for the exact recovery of $U=V$ in the matrix sensing problem \eqref{Mat_Factor}. For example, readers can find the interesting result in \cite{ZSL2017}.
\begin{theorem}[Theorem 3 in \cite{ZSL2017}]\label{Thm:ZSL2017}
	Let the sensing operator $\sA$ satisfy $(2,\delta)$-RIP condition and the loss function $F(U)=\|\sA(UU^\top -Z^*)\|^2$.
	\begin{enumerate}
		\item[(a)] If $\delta<1/2$, then $F$ has no spurious local minima.
		\item[(b)] If $\delta\geq 1/2$, then there exists a counterexample admitting a spurious local minima. 
	\end{enumerate}
\end{theorem}

However, the non-symmetric case introduces additional complexity, and achieving exact recovery with a sharp threshold for the RIP constant remains an open challenge. The noisy case is another open question \cite{ZSL2017}.  
Our joint inference problem is in a noisy, non-symmetric setting. Thus, the sharp results on $\delta$ in \cite{ZSL2017} for the symmetric noiseless setting do not apply. Nevertheless, the RIP constant $\delta$ provides insights into our problem, specifically regarding the existence of local minima and the convergence of the ALS algorithm. 

As an example, we consider an interacting particle system with $N=3$ particles in $\R^d$ with $d=1$ and $L=1$. We consider Gaussian i.i.d. initial conditions $\bX_{t_0}=(X^{1},X^{2},X^{3})\overset{i.i.d.}{\sim}\calN(0,I_3)$. To make it easy to present the results, we only consider two basis functions $\{\psi_1(x), \psi_2(x)\}$. Thus, giving $M$ samples,  the sensing matrices $\{A_m\}_{m=1}^M$ \eqref{Def:SensMat_IPS} are 
\begin{align}
	A_m = \Big[ \psi_k\big(X^{j+1,m}-X^{1,m}\big) \Big]_{1\leq j \leq 2 \atop 1\leq k\leq 2}
	= \begin{bmatrix}
		\psi_1\big(X^{2,m}-X^{1,m}\big) & \psi_2\big(X^{2,m}-X^{1,m}\big) \\
		\psi_1\big(X^{3,m}-X^{1,m}\big) & \psi_2\big(X^{3,m}-X^{1,m}\big)
	\end{bmatrix}\,,\label{Def:SensAm}
\end{align}
and the sensing operator $\sA$ is defined as in \eqref{Def:SensOp} correspondingly. 
Verifying Restricted Isometry Property \eqref{Cond:RIP} and finding the RIP constant $\delta$ for the operator $\sA$ are NP-hard problems in general.

We numerically estimate the RIP constant $\delta$ for rank $r=1$ as follows. First, compute the RIP ratios:  
\begin{align*}
	R_{\ell}=\frac{\norm{\sA (u_0^{\ell} (v_0^{\ell})^\top)}_2^2}{\norm{u_0^{\ell} (v_0^{\ell})^\top}_2^2}=\frac{1}{M}\sum_{m=1}^M \left|(u_0^{\ell})^\top A_m v_0^{\ell}\right|^2\,,\quad \ell=1,\cdots,2000\,, 
\end{align*}
where $\{u_0^{\ell},v_0^{\ell}\}_{\ell=1}^{2000}$ are unit vectors  randomly sampled in $\R^2$.  Next, normalize the RIP ratios to be in $[0,2]$. We choose $C=\frac{\max(\{R_{\ell}\})+\min(\{R_{\ell}\})}{2}$ in \eqref{Cond:RIP} so that $\widetilde{R}_{\ell}=\frac{R_{\ell}}{C}\in [0,2]$ and the RIP constant is given by
\begin{align*}
	\delta = \frac{\max(\{R_{\ell}\})-\min(\{R_{\ell}\})}{\max(\{R_{\ell}\})+\min(\{R_{\ell}\})} \in (0,1)\,.
\end{align*}
To highlight the effects of the basis functions on the RIP constant, we choose three sets of basis functions listed in Table \ref{Table:RIPTest}. 
Figure \ref{Fig:RIPTest} shows the distributions of the normalized RIP ratios for these three basis functions when $M=2000$. As a reference, we also present numerical tests of the RIP ratios for the classical Gaussian sensing operator, where the entries of $A_m$  are i.i.d.~standard Gaussian random variables. For the case of the Gaussian sensing operator, the normalized RIP ratios are clustered in the interval $[1-\delta,1+\delta]$ with the computed values for $\delta$ being $0.0461$,  $0.0399$, and $0.0534$; these values agree with the well-established result in \cite{RFP2010,LeeStoger2022} that $\delta \to 0$ when $M$ turns to infinity. In contrast, the normalized RIP ratios for the IPS spread widely in $[0,2]$ for all three sets of basis functions, and their RIP constants are $0.4151$, $0.8937$, and $0.9462$, which are relatively large. 
\begin{figure}[htb]
	\centering
	\includegraphics[width=0.3\textwidth]{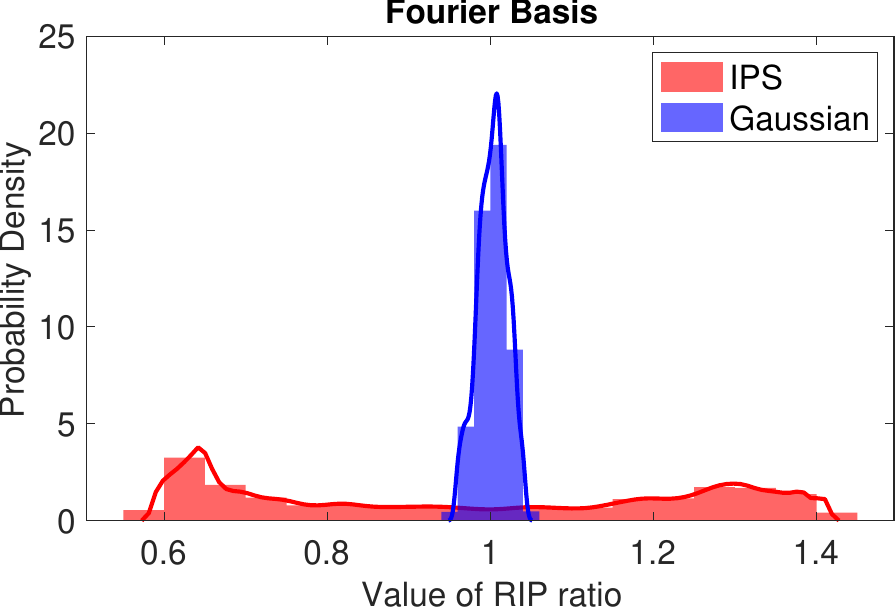}
	\includegraphics[width=0.3\textwidth]{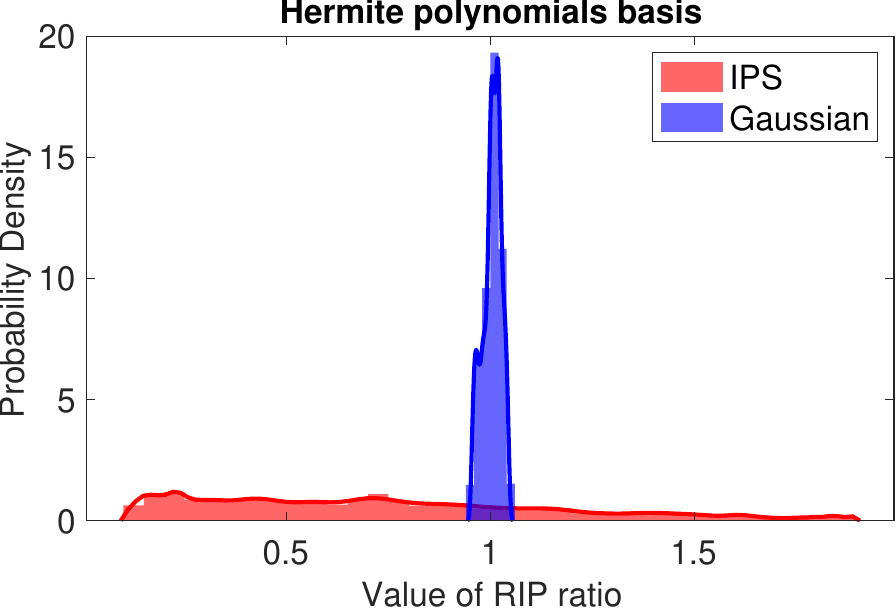}
	\includegraphics[width=0.3\textwidth]{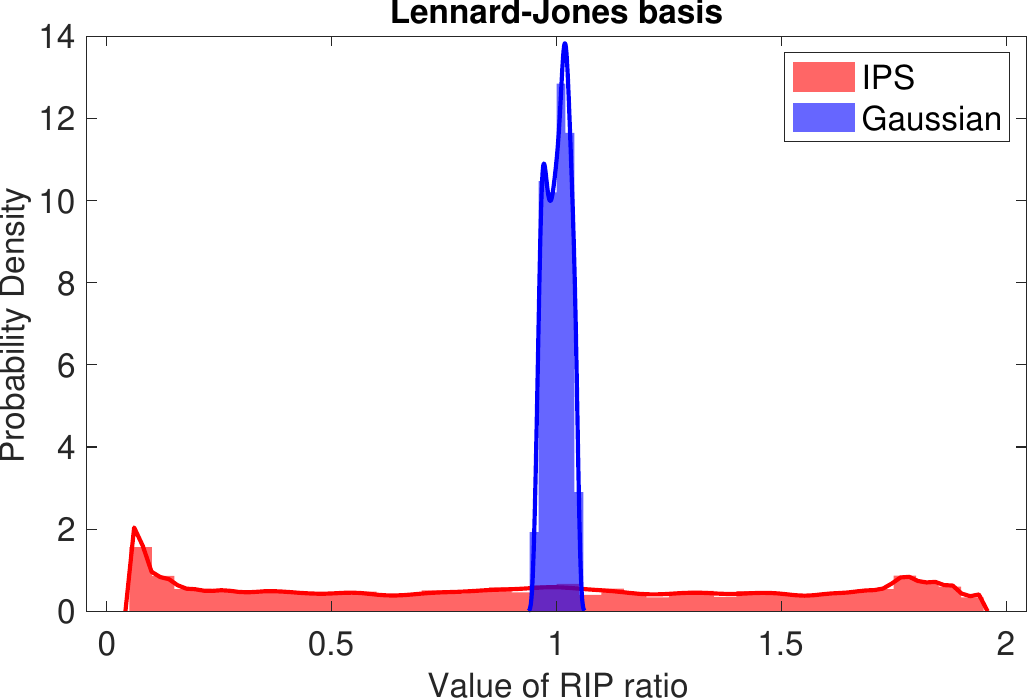}
	\caption{ Distributions of the RIP ratios of the interacting particle system (IPS) in red color. The basis functions and the estimated RIP constants are in Table \ref{Table:RIPTest}. The wide spread of the ratios in an interval $[1-\delta,1-\delta]$ indicates a large RIP constant $\delta$, particularly in the middle and right figures.  
	As a reference, we also present distributions of the RIP ratios for the Gaussian sensing operator in blue color, for which the RIP constants, from left to right, are  $0.0418$,  $0.0456$, and $0.0474$. }
	\label{Fig:RIPTest}
\end{figure}

\begin{table}[htbp]
\centering
\begin{tabular}{r|ccc}
\hline
\hline
                  & Left & Middle & Right \\ \hline
 $\psi_1(x)$ & $\sin(x)$ & $x^4-6x^2+3$ & $x^{-9}\mathbbm{1}_{[0.75, +\infty]}$ \\
                 $\psi_2(x)$ & $\cos(x)$ & $x^5-10x^3+15x$ & $x^{-3}\mathbbm{1}_{[0.25, +\infty]}$ \\ \hline \hline
             RIP constant     & 0.4151 & 0.8937 & 0.9462 \\ \hline \hline
\end{tabular}
\caption{The basis functions and the testing RIP constants in Figure \ref{Fig:RIPTest}. Left: Fourier basis; middle: Hermite polynomials basis; right: as in Section \ref{sec:numExper} for the Lennard-Jones interaction kernel. 
 }
\label{Table:RIPTest}
\end{table}

A large RIP value indicates that the matrix sensing problem may involve local minima, as highlighted by Theorem \ref{Thm:ZSL2017} and consistent with findings in non-symmetric scenarios in \cite{BahmaniRomberg17a,GJZ2017}. Therefore, local minima may exist in the joint inference, and we provide explicit examples. Let $U=\ba_{1\cdot}=(u_1,u_2)\in \R^{2}$ and $V=c=(v_1,v_2)\in\R^{2}$ be unit vectors. We then have 
\begin{align*}
	(u_1,u_2)=(\cos(\theta_1),\sin(\theta_1))\,,\quad (v_1,v_2)=(\cos(\theta_2),\sin(\theta_2))\,,\quad \theta_1,\theta_2 \in[0,2\pi)\,,
\end{align*} 
and the ground truth $(U^*,V^*)$ with $U^*=(u_1^*,u_2^*)=(\cos(\theta_1^*),\sin(\theta_1^*))$ and $V^*=(v_1^*,v_2^*)=(\cos(\theta_2^*),\sin(\theta_2^*))$.
The loss function denoted by $\calE_{M}(U,V)$ depends on $\theta_1$ and $\theta_2$:
\begin{align}\label{Def:LossFunRIP}
	\calE_{M}(\theta_1,\theta_2)=\calE_{M}(U,V)= \frac{1}{M} \sum_{m=1}^M \big|(U^*)^\top A_m V^*- U^\top A_m V  \big|^2
\end{align}
where the sensing matrices $\{A_m\}$ are defined in \eqref{Def:SensAm} and the basis functions are listed in Table \ref{Table:RIPTest}. It is clear that $(-U^*,-V^*)$ forms another global minimum pair, resulting in the loss function $\calE_{M}$ being zero. The corresponding angles are referred to as $(\widetilde{\theta}_1^*,\widetilde{\theta}_2^*)$.

\begin{figure}
	\centering
	\includegraphics[width=0.32\textwidth]{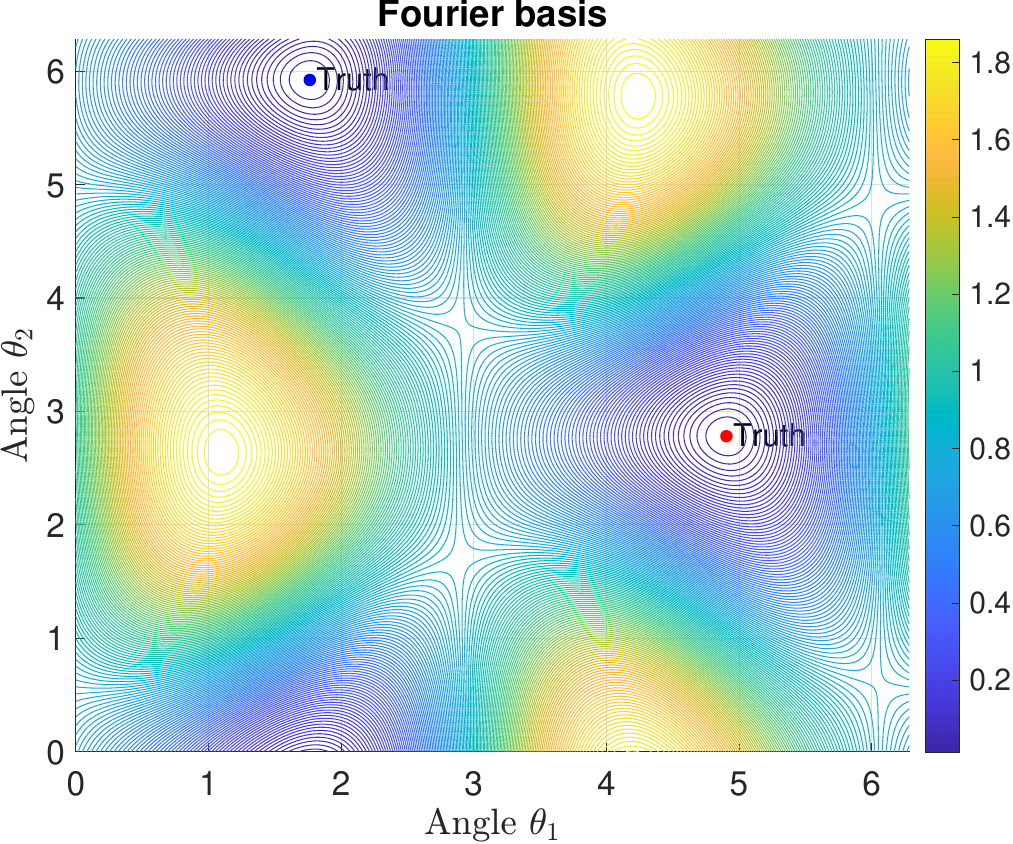}
	\includegraphics[width=0.32\textwidth]{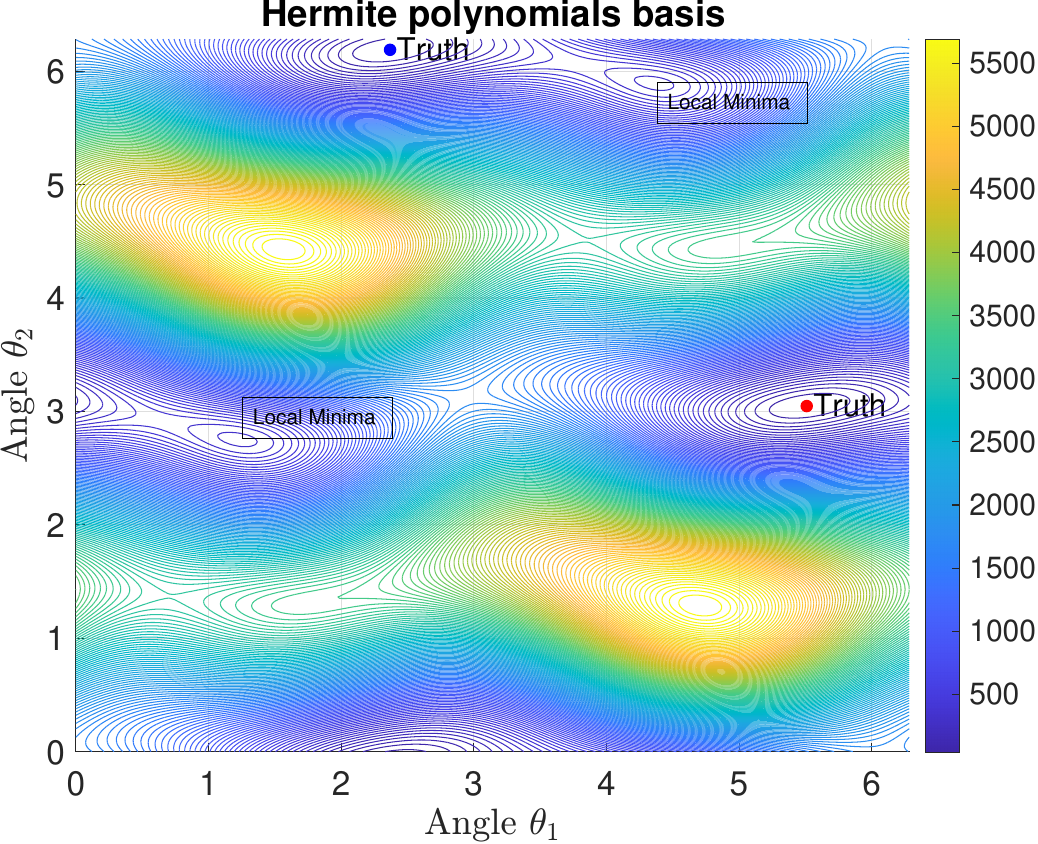}
	\includegraphics[width=0.32\textwidth]{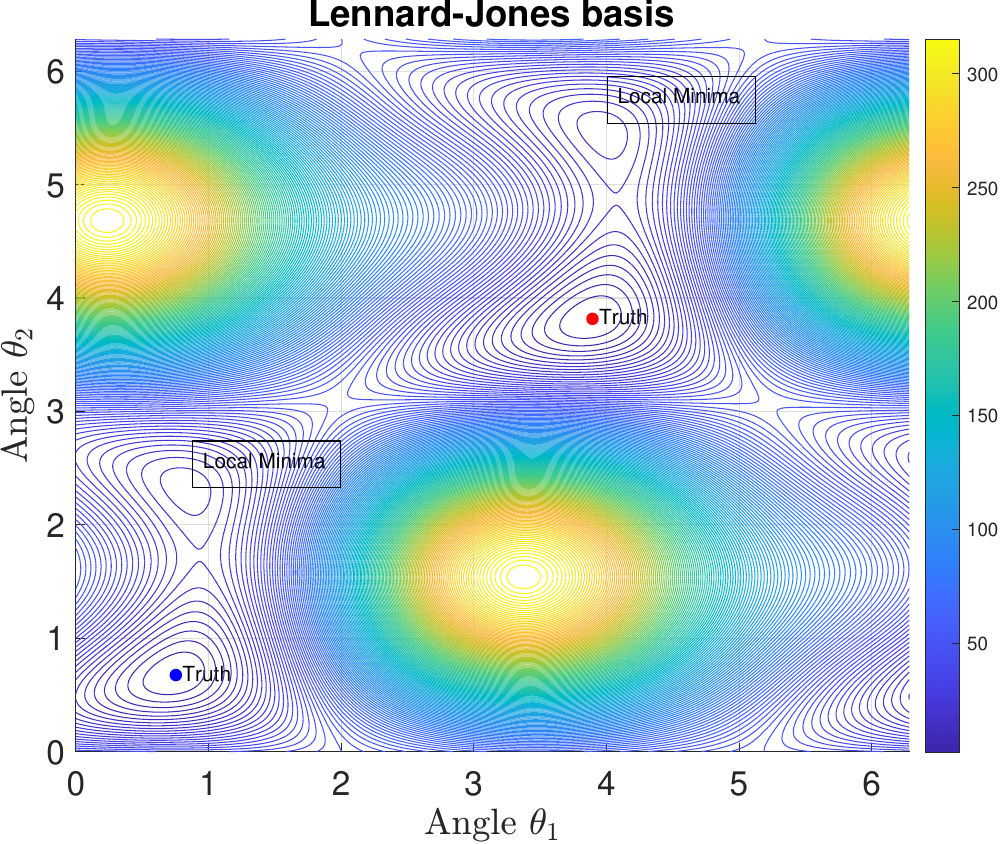}
	\caption{Contour plots of the loss functions for the three sets of basis functions. Local minima are present in the right two plots.}
	\label{Fig:ErrFuns}
\end{figure}

In Figure \ref{Fig:ErrFuns}, the red and blue dots locate the ground truths $(\theta_1^*,\theta_2^*)$ and $(\widetilde{\theta}_1^*,\widetilde{\theta}_2^*)$, respectively. Text boxes label the local minima. The basis functions are set as Hermite polynomials basis in the middle panel Figure \ref{Fig:ErrFuns} and are set as basis for the Lennard-Jones interaction kernel in the right panel of Figure \ref{Fig:ErrFuns}. The corresponding error functions $\calE_{M}(\theta_1, \theta_2)$ are plotted with $M=100$ samples and random choices of ground truth. Upon conducting a limited number of tests, the presence of local minima is not rarely observed, even when imposing normalization constraints on $U$ and $V$. This observation is expected, given that both scenarios exhibit high RIP constants (see Table \ref{Table:RIPTest}). 
However, we do not witness the existence of local minima with the selection of Fourier basis $\{\psi_1(x)=\sin(x),\psi_2(x)=\cos(x)\}$; see an example of the error function with $M=200$ samples in the left panel of Figure \ref{Fig:ErrFuns}. This may be a bit surprising, as the RIP value of the non-symmetric case may be expected to be half that of the symmetric case to ensure the absence of spurious local minima phenomena, as discussed in, for instance, \cite{GJZ2017}. The disappearance of local minima of the error function $\calE_{M}(\theta_1,\theta_2)=\calE_{M}(U,V)$ may be due to the constraints that $U$ and $V$ are unit vectors.  
Investigating the sharpness of the Restricted Isometry Property (RIP), exploring the non-existence of local minima, and understanding the convergence of the ALS algorithm for the joint inference in interacting particle systems on graphs are key subjects for future research.

\section{Algorithm: Three-fold ALS}\label{sec:3fold_ALS}

The three-fold ALS algorithm finds the minimizers of the loss function in \eqref{eq:est_joint_multitype_AUV} with an additional condition that $\bc = \bu \bv^\top$ has only $Q$ distinct columns.

Notice that the loss function \eqref{eq:est_joint_multitype_AUV} is quadratic in each of the unknowns $\ba, \bu, \bv$ if we fix the other two. Thus, we can apply ALS to alternatively solve for each of the unknowns while fixing the other two, and we call this algorithm  \emph{three-fold ALS}. In each iteration, this algorithm proceeds with the following three steps. To ensure that  $\bc$ has only $Q$ distinct columns, we add an optional step of $K$-means. 

\noindent{\em Step 1: Inference of the weight matrix $\ba$. } Given a coefficient matrix $\bu$ and a type matrix $\bv$, we estimate the weight matrix $\ba$ by least squares.  
For every $i \in [N]$, we obtain the minimizer of $\mE_{L, M}(\ba, \bu, \bv)$ in \eqref{eq:est_joint_multitype_AUV} by solving $\nabla_{\ba_{i \cdot}}\mE_{L, M}(\ba, \bu, \bv) = 0$, which is a linear equation in $\ba_{i \cdot}$:
\begin{equation}\label{eq:ALS_graphest_multitype}
\begin{aligned}
\widehat\ba_{i\cdot} \mathcal{A}^{\textrm{ALS}}_{\bu,\bv,M,i} := \widehat \ba_{i\cdot} ([\bB(\bX_{t_l}^m)_{i}]_{l,m} \bu \bv_{i\cdot}^\top) = [(\Delta \bX_{t_l})_i]_{l,m}/\Delta t\,\ \ ,
\end{aligned}
\end{equation}
using least squares with nonnegative constraints. The solution is then row-normalized to obtain an estimator $\widehat{\ba}_{i, \cdot}$ in the admissible set.

\noindent{\em Step 2: Inference of the coefficient matrix $\bu$. } Next, we estimate the coefficient matrix $\bu$ by minimizing the loss function $\mE_{L, M}(\ba, \bu, \bv)$ in \eqref{eq:est_joint_multitype_AUV} with the (estimated) weight matrix $\ba$ and a type matrix $\bv$. The minimizer is a solution to 
\begin{equation}\label{eq:ALS_multitype_uest_pre}
	\ba_{i\cdot}[\bB(\bX_{t_l}^m)_i]_{l,m} \widehat\bu \bv_{i \cdot}^\top = [\Delta \bX^i_{t_l}]_{l,m}/\Delta t, \quad i\in [N]. 
\end{equation}
Noting that for each $i\in [N]$, 
$\innerp{\mathcal{A}^{\textrm{ALS}}_{\ba, \bv, M, i}, \widehat\bu}_F:= \innerp{\ba_{i\cdot}[\bB(\bX_{t_l}^m)_i]_{l,m}^\top \otimes \bv_{i, \cdot}, \widehat\bu}_F = \ba_{i\cdot}[\bB(\bX_{t_l}^m)_i]_{l,m} \widehat\bu \bv_{i \cdot}^\top,$
we can write a linear equation for $\widehat\bu$ using Frobenius inner product: 
\begin{equation}\label{eq:ALS_multitype_uest}
	\mathcal{A}_{\ba, \bv, M}^{\textrm{ALS}} \widehat{\bu}:= \left(\innerp{\mathcal{A}^{\textrm{ALS}}_{\ba, \bv, M, i}, \widehat{\bu}}_F\right)_i = [\Delta \bX_{t_l} ]_{l,m}/\Delta t. 
\end{equation}

\noindent{\em Step 3: Inference of the type matrix $\bv$. } At last, we estimate the type matrix $\bv$ by minimizing the loss function $\mE_{L, M}(\ba, \bu, \bv)$ in \eqref{eq:est_joint_multitype_AUV} with the (estimated) weight matrix $\ba$ and coefficient matrix $\bu$. Firstly we solve the linear equation, 
\begin{equation}\label{eq:ALS_multitype_vest}
	\mathcal{A}^{\textrm{ALS}}_{\ba,\bu,M,i} \widehat{\bv}_{i\cdot}^\top:= ( \ba_{i\cdot} [\bB(\bX_{t_l}^m)_{i}]_{l,m} \bu )\widehat{\bv}_{i\cdot}^\top = [(\Delta \bX_{t_l})_i]_{l,m}/\Delta t\,,  \quad i\in [N]
\end{equation}
with the result denoted as $\widehat{\bv}'$. Then,  we apply a final normalization step to ensure the orthogonality at the end. Namely, we find an orthogonal matrix $\widehat{\bv}$ such that
$
	\widehat{\bv} = \argmin{\bv^\top \bv = I_Q} \norm{\bv - \widehat{\bv}'}_F
$.
The above problem is known as the orthogonal Procrustes problem \cite{gower2004procrustes}, and the solution is given by normalizing the singular values of $\widehat{\bv}'$, namely, 
\begin{equation}\label{eq:3ALS_v_orthogonal_normalization}
	\widehat{\bv'} = U\Sigma V^\top \implies \widehat \bv = U V^\top.
\end{equation}

\noindent{\em Step 4 (optional): apply $K$-means to the estimated  $\widehat \bv$. }  To enforce the coefficient matrix $\bc$ to have $Q$ distinct columns, we cluster the rows of $\widehat \bv$ by $K$-means.  

Algorithm \ref{alg:ALS_multitype} summarizes the above iterative procedure.

\begin{algorithm}
\caption{Three-fold ALS}\label{alg:ALS_multitype}
\small{
\begin{algorithmic}[]
\Procedure{Three-fold ALS}{$\{\bX^m_{t_0:t_L}\}_{m=1}^M , \{\psi_k\}_{k=1}^{\p} , \epsilon ,{\p}_{maxiter}$}
  \State Construct the arrays $\{\bB(\bX_{t_l}^m)\}_{l,m}$ and $\{\Delta \bX_{t_l}^m\}$ in \eqref{eq:est_joint} for each trajectory.  
  \State Randomly pick initial conditions $\widehat{\bu}_0$ and $\widehat{\bv}_0$.
  \For{$\tau = 1,\dots,{\p}_{maxiter}$}  
  \State Estimate the weight matrix $\widehat{\mathbf{a}}_{\tau}$ by solving \eqref{eq:ALS_graphest_multitype} with $\bu=\widehat{\bu}_{\tau-1}$ and $\bv=\widehat{\bv}_{\tau-1}$, with nonnegative \indent \quad least squares followed by a row normalization.
  \State Estimate the coefficient matrix $\widehat\bu_\tau$ by solving \eqref{eq:ALS_multitype_uest} with $\mathbf{a}=\widehat{\mathbf{a}}_\tau$ and $\mathbf{v}=\widehat{\mathbf{v}}_{\tau-1}$ by least squares.
  \State Estimate the type matrix $\widehat \bv_\tau$ by solving \eqref{eq:ALS_multitype_vest} with $\mathbf{a}=\widehat{\mathbf{a}}_\tau$ and $\mathbf{u}=\widehat{\mathbf{u}}_\tau$ by least squares followed \indent \quad by normalization in singular values as in \eqref{eq:3ALS_v_orthogonal_normalization} and an optional step clustering the rows of $\widehat \bv$. 
  \State Exit loop if $||\widehat{\mathbf{a}}_{\tau}-\widehat{\mathbf{a}}_{\tau-1}||\le\epsilon ||\widehat{\mathbf{a}}_{\tau-1}||$, $||\widehat{\bu}_\tau-\widehat{\bu}_{\tau-1}||\le\epsilon||\widehat{\bu}_{\tau-1}\|$ and $||\widehat{\bv}_\tau-\widehat{\bv}_{\tau-1}||\le\epsilon||\widehat{\bv}_{\tau-1}\|$.
\EndFor
\Return  $\widehat{\mathbf{a}}_\tau,
\widehat{\bu}_\tau, \widehat{\bv}_\tau$.
\EndProcedure
\end{algorithmic}
}
\end{algorithm}

\section*{Acknowledgement}
{This work was supported by National Science Foundation Grants DMS-2238486; Air Force Oﬃce of Scientific Research Grant
AFOSR-FA9550-20-1-0288, FA9550-21-1-0317, and FA9550-23-1-0445; and the Johns Hopkins Catalyst Award; and the Sun Yat-sen
University research startup. }




\bibliographystyle{alpha}
\bibliography{ref_IPS_graph.bib,ref_FeiLU2023_8,ref_grn,ref_regularization23_09,MM_Publications}


\end{document}